\pdfoutput=1 %
\documentclass{article} %
\usepackage{iclr2027_conference,times}

\iclrfinalcopy   %

\usepackage{amsmath,amsfonts,bm}

\def\eqref#1{equation~\ref{#1}}

\def\1{\bm{1}}

\DeclareMathAlphabet{\mathsfit}{\encodingdefault}{\sfdefault}{m}{sl}
\SetMathAlphabet{\mathsfit}{bold}{\encodingdefault}{\sfdefault}{bx}{n}

\newcommand{\E}{\mathbb{E}}

\usepackage{hyperref}
\usepackage{url}

\usepackage{etoc}  %

\usepackage[margin=1.1in]{geometry}
\usepackage{amsmath,amssymb,amsthm}
\usepackage{mathtools}
\usepackage{booktabs}
\usepackage{enumitem}
\usepackage{microtype}
\usepackage{graphicx}
\definecolor{darkblue}{RGB}{0,40,80}
\usepackage{tikz}
\usetikzlibrary{arrows.meta,positioning,fit,backgrounds,calc}

\usepackage{cleveref}  %
\usepackage{booktabs}
\usepackage{multirow}
\usepackage{array}

\usepackage{longtable}   %
\usepackage{booktabs}    %
\usepackage{array}       %
\newcolumntype{P}[1]{>{\raggedright\arraybackslash}p{#1}}

\usepackage[varqu,scaled=0.9]{zi4}   %

\usepackage{placeins} %

\usepackage{xspace}

\newcommand{\xaxis}{{$x$-axis}\xspace}
\newcommand{\yaxis}{{$y$-axis}\xspace}

\newtheorem{definition}{Definition}[section]
\newtheorem{proposition}{Proposition}[section]

\newtheorem{corollary}{Corollary}[section]

\newcommand{\defeq}{\coloneqq}

\newcommand{\M}{\mathcal{M}}
\newcommand{\pseed}{p^{0}}
\newcommand{\Lu}{L_{u}}
\newcommand{\ku}{k_{u}}
\newcommand{\kf}{k_{f}}

\newcommand{\Kco}{K^{|}}
\newcommand{\Ktc}{Kt^{|}}
\newcommand{\Kcc}{Kc^{|}}
\newcommand{\Khat}{\hat{K}}

\newcommand{\pbare}{p^{\triangledown}}          %
\newcommand{\pK}{p^{\blacktriangledown}}        %
\newcommand{\pKc}{p^{\overline{\blacktriangledown}}}  %
\newcommand{\C}{\mathcal{C}}
\newcommand{\V}{\mathcal{V}}
\renewcommand{\E}{\mathcal{E}}
\newcommand{\T}{\mathcal{T}}
\newcommand{\J}{\mathcal{J}}
\newcommand{\cost}{c}
\newcommand{\shown}{x_{1:n}}
\newcommand{\nxt}{x_{n+1:n+k}}

\title{Artificiality's First Exam:\\An LLM-Referenced C-test for Supraintelligence
}

\author{José Hernández-Orallo\thanks{Code and data available at \url{https://github.com/jorallo/ctest-one}.}\\
Leverhulme Centre for the Future of Intelligence\\
University of Cambridge\\
\texttt{jh2135@cam.ac.uk} \\
Valencian Research Institute for Artificial Intelligence\\
Universitat Politècnica de València\\
\texttt{jorallo@upv.es} \\
}

\begin{document}
\etocdepthtag.toc{mtmain}        %

\maketitle
\lhead{Preprint. Under review.}

\begin{abstract}
Under the accelerated pace of AI development, any benchmark whose difficulty
is calibrated with humans or current AI systems  becomes circular, or saturates %
soon. %
How can we then evaluate machines that may become more capable than the current machines and the people who
write the tests? 
One answer to the question hinges on {\em inverse problems}, inductive inference questions for which difficulty  %
depends on the complexity of the underlying mechanism, {\em given} the observations and the previous knowledge. Levin's $Kt$, a computable variant of Kolmogorov complexity, puts a theoretical upper bound to that difficulty, balancing the simplest representation encoding for the mechanism and the logarithm of the time steps to run that mechanism. Previous tests derived from Levin's $Kt$, 
such as the original C-test, are based on minimalistic Universal Turing Machines (UTM), capable of very limited representations and no previous knowledge. 
{\em Reasoning} large language models (LLM) are UTMs that capture a compendium of human culture as knowledge, using prompts as natural language programs. This paper presents the first test based on Levin's $Kt$ with LLMs as representational machines.  
Fresh test items can be produced at high difficulties,  estimated in a ratio scale measured in bits, before any human or frontier model can solve them. %
We formalise the
measure, devise a lightweight generation and certification pipeline, explore the effect of the power of the reference machine and the gap between examiner and examinee. %
We produce a first battery of certified items administered to twelve LLMs. The resulting agent characteristic curves decrease
with difficulty, down to the uncharted high-difficulty region, the {\em supra-intelligence window} of items we can certify, yet %
beyond %
generations of frontier intelligence. %
\end{abstract}

\section{Introduction}\label{sec:intro}

We can all conceive and verify questions we cannot solve ourselves. This possibility is supported by theoretical computer science \citep{cook1971complexity,levin1973}. For example, inverse problems \citep{mccarthy2006challenges}, such as finding the pattern behind a sequence, are hard to solve but easy to verify. %
Indeed, verifiability is behind the recent progress of `reasoning' models in areas such as mathematics and software engineering. %
But even if we can generate very difficult questions for the next generation of AI systems, it is an open question how to do so systematically, gradating items in terms of an intrinsic measure of difficulty, when no existing human or AI system is still able to solve them.

As AI progresses, benchmarks devised by human experts or based on current challenges saturate and must be replaced 
\citep{zellers2019hellaswag,schlangen2019language};
replacements leak into
training corpora \citep{balloccu2024leak,jiang2024investigating}; and progressively `harder' items can no longer be produced and verified  
\citep{center2026benchmark,phan2025humanity}. 
If we want to keep measuring AI as capabilities grow, we need tests that can be 
\emph{generated on demand}, beyond the reach of the testees, with item difficulty that is \emph{defined by construction}
rather than by consensus or norm, and  with \emph{certifiable} correctness that does not require the items to be solved directly.

The goal of this paper is to devise a feasible and systematic procedure to generate problems whose difficulty can be gradated over an intrinsic, non-moving scale, derived by a weak model $\M$ to evaluate strong models. These problems can belong to any broad class of inverse problems, whose difficulty is their conditional Levin's $Kt$ complexity \citep{levin1973,levin2013universal}, which expresses the search effort of Levin's universal search, an optimal solution (up to a constant) to any inverse problem. The intuitive idea behind the test is simpler: an item is easy if there is a program describing the solution {\em given the observation} that has a low cost and runs fast on $\M$.

The scope of this paper is delineated by %
self-imposed constraints towards this goal: (1) The use of small representation machines, which can fit in small GPUs while still being able to certify items that are unsolvable by frontier models, realising bounded and sparse supra-intelligence evaluation windows; (2) Similarly, the use of a medium-capability composer, able to govern the certification search procedures; (3) An experimental design that is sufficient for the evaluative goals, nothing more nothing less, for sustainability; (4) While the new C-test, ONE, is useful in itself to rank models, it is simply the first of its kind in a series of broader and more sophisticated tests where difficulty can be estimated by a bound to the search cost of the solution given the problem. Our contributions are:

\begin{itemize}
\item We demonstrate a general method for an examiner that is weaker than the examinee to generate challenging problems that are verifiable (valid, unquestionable and certified) and whose difficulty can be anticipated even before any frontier intelligence can solve them, gradating on the upper end of difficulty. 
\item %
We upend the subjective choice of the universal Turing machine (UTM) in the calculation of Levin's $Kt$, %
by using LLMs as UTMs, fixing the representations to human culture. %
$Kt$ %
becomes conditional: the logprobs of the prompt {\em given the  observation} plus the logarithm of the number of all inference steps. %
\item Models are assigned an ability from zero to infinity, bits of a ratio scale as per measurement theory. A linear increase in the number of $Kt$ bits corresponds to an exponential growth of search compute, a well-founded scale to really quantify the pace of AI progress, compared to other populational or contaminated scales. 
\item The item generator is based on the number of retrievals and transformations used in the inversion problem, connecting this generalisation of the original C-test with the notions of fluid and crystallised intelligence. $Kt$, however, accounts for simplifications, is very predictive, and  %
is now grounded in human priors.
\item Verification can be performed with small reasoning LLMs such as Qwen14B or Phi14B, even for very hard items. The composer is a medium-sized LLM doing an informed search, avoiding the intractable enumeration dove-tail approach, previously used for the approximation of $Kt$. 
\item A 14B-parameter machine suffices to show strong Spearman correlation between $Kt$ and item performance of the twelve examinees ($-0.79$), as well as ability with other psychometric indicators, such as Epoch's ECI ($+0.76$). The supraintelligence window is about 190 bits for the strongest examinee, GPT5Sol.
\end{itemize}

These are very promising signs for the scalability of supra-intelligence evaluation. In the following sections we will explain the properties and results of C-test ONE, our first exam of an endless series. %

\section{Background}\label{sec:background}

General intelligence evaluation has usually been associated with the identification of patterns that can explain the regularities in data, while at the same time being able to use existing knowledge to solve new problems. These components roughly correspond to Cattell's concepts of fluid and crystallised intelligence in humans \citep{cattell1963theory}, but more generally they can be identified as the degree of transformations and retrievals, respectively, in an {\em inverse problem}. 
Inverse problems are common in education and science \citep{groetsch1999inverse}. In computer science they are often described as finding the input $x$ that produces $y$ for a given function $f$, with $f(x) = y$. One general yet very precise class of these problems consist in determining the pattern behind a sequence, %
common in letter and number series problems in IQ tests \citep{thurstone1937ability}. Their most general form consists in inferring the program that generates an observed sequence given a Universal Turing Machine, for which compression is shown to be the best strategy asymptotically \citep{solomonoff1960preliminary,solomonoff1964formal1,solomonoff1964formal2}. These connections led to the introduction of enhanced tests with compression, such as \citep{Dowe-Hajek}, the C-test \citep{hernandez1998,hernandez2000beyond} and the `universal intelligence test' \citep{legg2007universal}. For the relation between algorithmic information theory, compression and (the measurement of) intelligence more generally, see appendix \ref{sec:extendedrelated} and the extended coverage in \citet{hernandez2017book}. We just highlight a few key connections next.

\textbf{Kolmogorov complexity and Levin search.} For a
universal machine $U$, the Kolmogorov complexity of a sequence $x$ is defined as $K_{U}(x) \defeq \min\{\ell(p): U(p)=x\}$ with $\ell(p)$ being the length of the program (in bits). $K_U$ is invariant
across universal machines up to an additive constant \citep{livitanyi}, but this constant is as large as coding a machine-to-machine interpreter. Another issue of $K$ is that it is incomputable, only being approximated from above. 
 Levin's $Kt$ complexity is defined as $Kt(x) \defeq
\min_{p}\{\ell(p) + \log \tau(p)\}$, penalising the execution time $\tau(p)$ of running the program. 
The most salient property of $Kt$ is not its computability, but the fact that it quantifies the difficulty of an inversion problem following a {\em universal search}, dovetailing programs by increasing budget of computational steps. 
 Universal search will solve the problem in a time $2^{\ell(p)}\cdot\tau(p)$ at most, which is the best general inversion algorithm despite the huge multiplicative factor. Taking logarithm of this expression leads to $Kt$.

\textbf{C-test.} The original C-test \citep{hernandez1998,hernandez2000beyond} %
generates letter
sequences produced by programs in a minimalistic Turing-complete %
language (a RISC computer), %
with item difficulty given by the shortest program found by enumeration. The C-test introduces the notion of
\emph{unquestionability}, i.e., no alternative continuation is producible
by a program of comparable length. %
For instance, the letter series \emph{a, b, c, a, b, c, a, \dots} is admitted
at low difficulty: every short program consistent with the prefix continues with
\emph{b}. The letter series
\emph{a, b, b, c, c} is \emph{rejected}: ``{\tt nth letter n times}'' predicts \emph{c}, but ``{\tt one $a$, then each letter twice}'' %
predicts \emph{d}, at similar program
length, so no continuation is unquestionable from these five letters only.

\textbf{LLMs as universal Turing machines (UTMs).} A single transformer forward
pass is not Turing-complete, but the autoregressive loop with an unbounded
generated context is computationally universal \citep{perez2021turing,schuurmans2023memory}; the
model's chain of thought plays the role of a working tape. Deterministic
decoding %
can make an LLM %
work as a reference
machine. Also, the view of language models as compressors
\citep{deletang2024language} suggests that the use of the logprobs for prompts is the right homologue to  bits in a UTM.%

\textbf{Saturation and the ARC challenges.} New benchmarks get saturated soon, because AI becomes more powerful, new models `benchmax' or are contaminated, in a challenge-solve-and-replace dynamics \citep{zellers2019hellaswag,schlangen2019language}. This is epitomised by Humanity's Last Exam \citep{center2026benchmark,phan2025humanity}, a compendium of very difficult problems in science that required the most expert humans for conception. Other
tests aim at what `is missing in AI', such as the famous ARC (Abstraction and Reasoning Challenge)
benchmark series \citep{foundation2026arc}. %
Originally conceived from 
algorithmic information theory \citep{chollet2019}, following 
\citet{hernandez2017book}, the formalism, however, plays no operative role in
the test. Items have been mostly hand-designed, difficulty is calibrated against human performance, and several versions of the challenge have saturated as usual, forcing a complete redesign. Multiple ARC variants abound \citep{maasch2025causalarc}, now converging towards the old idea of evaluating AI with video games \citep{liebana2022general,ying2507assessing,ying2026ai}. 

\textbf{Well-founded SuperARC.} To our
knowledge, the C-test remained the only test whose items were
\emph{generated and calibrated by} algorithmic information theory, together with its interactive variants \citep{legg2013approximation}. Recently, SuperARC \citep{superarc2026} has
revisited much of the original C-test
programme: sequence items of increasing algorithmic complexity %
\citep{zenil2018bdm}, and an accompanying
thesis that compression is comprehension, recapitulating the motivation and name of the C-test %
(the C was for comprehension). %
However, SuperARC is composed of a fixed inventory of (mostly binary) sequences,  %
no
unquestionability requirement is imposed, allowing for rival continuations. Shared cultural knowledge is treated as contamination, to be
removed by the retreat to binary strings. %
Unsurprisingly, in this narrow context their central empirical
finding is that when LLMs are asked for the pattern behind a sequence, they
mostly return the equivalent of ``print the sequence''. %

\textbf{Verification, Scales and Scalability of AI Evaluation.}
Creating tests with calibrated difficulty for future intelligence is a major challenge in AI evaluation. %
The use of humans for judging the evaluation results does not scale up, with LLM-as-judges and other kinds of automated verification being commonplace. However, there is a convoluted relationship between the verifier's capabilities and item difficulty
\citep{zhou2026variation}. Even if verification is clear in some domains (e.g., in theorem proving), the scales of difficulty are not known until a solution is found. Also,
\citet{achille2026universal} have connected the capability of reasoning models %
with Levin search, but only for verifiable tasks. 
Alternatively, populational measures of difficulty, such as IRT \citep{martinez2019item,ho2025rosetta}, cannot anticipate the difficulty of an item that has not been solved yet. But we can start with the solution of an inverse problem to create an item, and explore its properties to estimate difficulty, intrinsically.

\textbf{Prompt compression.} The greatest limitation of approaches based on algorithmic information theory is program size  minimisation. When this is applied to LLMs, the problem reappears. Namely, finding the shortest prompt producing a target
output is intractable in general; gradient-guided discrete search
\citep{shin2020autoprompt,zou2023gcg,wen2023hard} and amortised
language-model inversion \citep{morris2023language} provide upper bounds. Also, these techniques are not applicable for reasoning models, as the final part of the output is what matters. Our pipeline avoids enumeration by using an LLM-informed search of more compact prompt candidates.%

\section{LLM Costing and Metrics}\label{sec:formal}

We will use several LLMs for different roles. Reference machines, used for costing, are frozen, versioned LLMs, %
using %
a deterministic decoding policy %
and tokens proceed autoregressively until a stopping condition, usually a designated delimiter
$\langle\mathsf{out}\rangle$. This  constitutes the model's \emph{chain of
thought}, allowing nontrivial computations (e.g., calendar
arithmetic) before emitting the final answer. For simplicity, given a machine (LLM) $\M$, a prompt $p$ and a context $t$, we denote by $\M(p|t)$ the final output of $\M$, and $\tau_\M(p|t)$ the number of tokens that have been used during the whole inference. Different models have different vocabularies $\Sigma_\M$ and tokenisers. For many open-weight models we have access to  
$P_{\M}\!\left(t_{i} \mid t_{1}\dots t_{i-1}\right)$, the probability $\M$ assigns to token $t_{i}$ when preceded by $t_{1:i-1}$. This is the prior the model represents, usually human language in a particular culture.

Token count is a misleading program length: with $|\Sigma_\M| \approx 10^{5}$, every token would carry $\log_{2}|\Sigma_\M| \approx 17$ bits, but tokens are not equally informative. Instead, we use a different coding for all inputs and outputs: 

\begin{definition}[Costing]
\label{def:cost}
The \emph{cost} of any string $s \in \Sigma^*_\M$ conditioned to a context $t \in \Sigma^*_\M$, where $s = s_{1}\dots s_{n}$ and $t= t_{1}\dots t_{m}$, is its
Shannon--Fano code length under the model prior,
\[
c_\M(s|t) \;\defeq{}\; -\sum_{i=1}^{n} \log_{2}
P_{\M}\!\left(s_{i} \mid t_{1}\dots t_{m} s_{1}\dots s_{i-1}\right).
\]
\end{definition}

With no context, by Kraft's inequality $\sum_{s} 2^{-c_\M(s)} \le 1$, $\cost_\M$ behaves as
a prefix-free program-length measure. %

\begin{definition}[LLM-relative Kolmogorov complexity]
\label{def:kllm}
For any strings $s$ and $t$,
\[
K_\M^\V(s|t) \;\defeq\; \min_p\left\{\, c_\M(p|t) \;:\; \V(p|t) = s \,\right\},
\]
where $\V(p|t)$ is the final output of $\V$. We assume $s$ and $t$ belong to {\em both} $\Sigma^*_M$ and $\Sigma^*_V$, even if the tokenisations may be different, and equivalence is performed at the level of character encoding (standard strings). %
\end{definition}

This is a more general definition than the standard Kolmogorov complexity in that apart from checking that the program reproduces the input with a reference machine $\V$, we have an additional  reference machine $\M$ to cost the program, conditional to $t$, instead of mere length. %
When the superindex is not specified, we will assume both are the same machine. Even in this case, the definition differs from standard Kolmogorov complexity, because of the use of a coding that depends on $t$. %
As we will see, if $\V$ naturally follows with $s$ when given $t$, as when an LLM reads a prompt  ``{\tt 1,3,5,7,9}" and computes ``{\tt 11, 13, ...}", then $p = \epsilon$, the empty string, and $K_\M^\V(s|t) = 0$. 

\begin{definition}[LLM-relative Levin complexity]
\label{def:kt}
For any strings $s$ and $t$,
\[
Kt_\M^\V(s|t) \;\defeq\; \min_p\left\{\, c_\M(p|t) + \log_{2} \tau_\V(p|t) \;:\; \V(p|t) = s \,\right\},
\]
where $\V(p|t)$ is the final output of $\V$, but $\tau_\V(p|t)$ is computed with all the inference tokens, including CoT. 
\end{definition}

This is the {\em conditional} Levin complexity of $s$ given $t$, and shortened by $\Ktc$. The unit is bits, as the logarithm of the combinations of the original search perplexity $2^{c_\M(p|t)} \tau_\V(p|t)$. The use of $c_\M$ rather than the length in tokens is justified because this search can be guided by the model probabilities (see e.g., \citet{achille2026universal}). %

\section{Sequential Problems and Unquestionability}\label{sec:composition}

The above definitions are generic, and can work for arbitrary text sequences given as problem formulation, but also for inverse problems in the spirit of the original C-test, and similar tests, as in \citet{thurstone1937ability}'s letter series. 

\begin{definition}[Sequential inversion problem]
\label{def:invprob}
A sequential inversion problem is a comma-separated text string $x$ that is composed of a prefix $x_{1:n}$ and a postfix $x_{n+1:n+k}$. Each of the $n+k$ elements is known as a term, and can consist of one or more tokens. The goal is to extrapolate the postfix having only seen the prefix.
\end{definition}

We will use the same subindex notation as before, now referring to terms, not tokens. For instance, $x_{1:n}$ = ``{\tt Italy, Slovenia, Croatia, Bosnia and Herzegovina, Montenegro}'' %
with prefix cut point at $n=5$, includes terms that have more than one token. Namely, $x_4$ is {\tt Bosnia and Herzegovina}. 

Now, is there a mechanism that unequivocally produces the postfix given the prefix? In this case, the pattern could be $p$ = ``{\tt Mediterranean countries clockwise from Italy}''. First, we can figure out that $p$, when executed on $\V$, will produce a string whose first $k$ terms will match the postfix, i.e., $\V(p|x_{1:n}) = x_{n+1:n+k}$. This is the verification. Also, we expect $p$ to have lower cost than the listing program $q$ that simply enumerates the following countries, but also lower than making a minor modification to program $p$ such as $p'=$``{\tt Mediterranean countries clockwise from Italy, except Croatia}''. Otherwise,  $p'$ could be a reasonable pattern for $x$.

While many philosophers of induction have argued that there is no reason to choose one mechanism over another \citep{sterkenburg2016solomonoff,dingle2018input}, 
the original C-test demanded  
\emph{unquestionability}: among all programs of {\em near-minimal cost} that
reproduce the shown prefix, all must agree on the next term. In our setting, formally.

\begin{definition}[Unquestionability at $L$]
\label{def:unq}
Let $x = x_{1} x_{2} \dots$ be a comma-separated sequence, $x_{1:L}$
its first $L$ terms, and $x_{L+1:L+k}$ the intended next terms. Let $p^{\star}$ be a program generating $x$ on $\V$ that minimises the unconditional $K_\M^V(x)$ = $c_\M(p^{\star})$. The item
$x$ is \emph{unquestionable} at $L$ on $\V$ and $\M$ if every program $q$ with
  $c_\M(q) \;\le\; c_\M(p^{\star}) + \delta_\M$
 and
  $\V(q)_{1:L} = x_{1:L}$  
also satisfies $\V(q)_{L+1:L+k} = x_{L+1:L+k}$.
\end{definition}

Typically, we look for the smallest $L$ such that unquestionability holds, defining the cutpoint $n$ between prefix and postfix. Complex patterns require a higher $L$, for the terms to justify the program.  
The tolerance $\delta_\M$ looks like a parameter but it is derived as the cost of the least expensive exception clause on $\M$, minimising on `except ..", or ``and fail randomly", etc. Accordingly, ``$p^{\star}$, but $x'_{L+1}$ at term $L{+}1$'' is not a valid rival program. %
Several behaviours are possible as $L$ grows:
a comprehensive program %
(``the odd numbers'') has cost
\emph{constant} in $L$; a program that patches a rule with exceptions pinned to
positions (``the odd numbers except number 7 and 19'') has a cost that grows by  $\log L$ if more exceptions are needed along the way; a {\em listing} program $p^\lambda$ that just enumerates the terms
(``print $x_{1}\dots x_{L}, x_{L+1}$'') has cost growing around $\beta  L$,
where $\beta$ is the mean per-term cost. %
In general, showing more terms (i.e., increasing $L$) eventually puts all rivals out of the $\delta_\M$-window. 

As a listing will work for any finite continuation,  
\Cref{def:unq} implies that the program must compress $x$:

\begin{proposition}[Compressibility]
\label{cor:compress}
If an item $x$ is unquestionable at $L$, then
$c(p^{\star}) \le c(p^{\lambda}) - \delta$. In words, the program
with smallest coding must be simpler than $p^\lambda$ (the listing that  enumerates the continuation)
by a margin $\delta$. 
\end{proposition}

\begin{corollary}\label{cor:compressibility-cond}
Assume  $c(p^{\lambda})=\beta L + c_0$, with $\beta$ being the per-term cost of $p^{\lambda}$ and $c_0$ a printing scaffold. %
If the best program $p^{\star}$ for an item $x$ has 
$c(p^{\star}) > \beta L - \delta + c_0$ 
then $x$ is questionable at $L$. [Proof in Appendix \ref{app:proof}.]
\end{corollary}

For minimal Turing machines, such as those in the original C-test, loop or recursion are the standard sources for these sequential {\em transformations}, such as ``{\tt $x_{i+1} = x_{i}+2$ loop forever}''. 
However, for an LLM, a good sequence can be obtained by retrieval only, such as $p$ = ``{\tt the kings and queens of England}". Note that this is not a listing, but a named {\em retrieval} that recovers a memorised listing. The cost of this $p$ on a modern LLM is small (a few common words), and the probability that any decent LLM can recover the list is high. This is perfectly okay: continuing a sequence such as ``{\tt William I, William II, Henry I}'' is easy, and one of the motivations for using LLMs in the first place. Many mechanisms are composed of zero or more transformations and zero or more retrievals. Of course there will be differences between reference machines, and some of them will have better retrieval than transformation capacity. Reasoning models will usually add to the latter. %

Retrieval has a cost. For instance, indexing the On-Line Encyclopedia of Integer Sequences (OEIS) requires almost 20 bits to name one. Similarly to name a  published book. Transformations also pay when they are entropic, such as ``$3.253n^{0.734}$''. Actually, it is in the combination of retrieval and transformation where many patterns lie, such as ``the weekdays backwards, interleaved with the months''. 
As we said in the introduction, retrieval and transformation connect with the psychometric literature through crystallised and fluid intelligence, respectively. These two dimensions $r,t$ are simply used in the prompts for the composer, when populating the test battery. 
However, it is unclear how the levels of retrieval and transformation will influence observed performance. 

\section{C-test ONE:  Pipeline}\label{sec:pipeline}

Borrowing from the original C-test, we call the new test C-test ONE. Beyond the use of LLMs as UTMs for the representation of programs and their coding, the test has many other innovations, such as the use of domains and the retrieval-transformation matrix in the generator, the intervention of a composer for unquestionability and, most especially, the use of the conditional version of Levin's complexity, $\Ktc$ (\Cref{def:kt}), when certifying the items. 

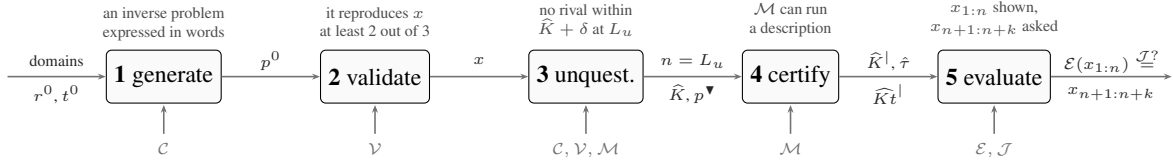
\begin{figure}[h]
\centering
\begin{tikzpicture}[
    font=\footnotesize,
    stage/.style={draw, rounded corners=2pt, align=center, inner sep=3pt,
                  minimum height=7mm, fill=gray!4},
    who/.style={font=\tiny\itshape, text=black!55},
    lbl/.style={font=\tiny, align=center, text=black!70},
    up/.style={font=\tiny, align=center, text=black!85, inner sep=0pt,
               yshift=2.1mm},
    dn/.style={font=\tiny, align=center, text=black!85, inner sep=0pt,
               yshift=-2.3mm},
    ar/.style={-{Stealth[length=3pt]}, black!55, semithick},
    node distance=13mm,
  ]

  \node[stage] (c) {\textbf{1} generate};
  \node[stage, right=of c] (v) {\textbf{2} validate};
  \node[stage, right=of v] (u) {\textbf{3} unquest.};
  \node[stage, right=of u] (a) {\textbf{4} certify};
  \node[stage, right=of a] (e) {\textbf{5} evaluate};

  \coordinate (in) at ($(c.west)+(-13mm,0)$);
  \draw[ar] (in) -- (c);
  \node[up] at ($(in)!0.5!(c.west)$) {domains};
  \node[dn] at ($(in)!0.5!(c.west)$) {$r^{0}$, $t^{0}$};

  \foreach \x/\y in {c/v, v/u, u/a, a/e} \draw[ar] (\x) -- (\y);

  \node[up] at ($(c.east)!0.5!(v.west)$) {$p^{0}$};
  \node[up] at ($(v.east)!0.5!(u.west)$) {$x$};
  \node[up] at ($(u.east)!0.5!(a.west)$) {$n=L_u$};
  \node[dn] at ($(u.east)!0.5!(a.west)$) {$\widehat{K}$, $\pK$};
  \node[up] at ($(a.east)!0.5!(e.west)$) {$\widehat{K}^{|}$, $\hat\tau$};
  \node[dn] at ($(a.east)!0.5!(e.west)$) {$\widehat{Kt}^{|}$};

  \coordinate (out) at ($(e.east)+(16mm,0)$);
  \draw[ar] (e) -- (out);
  \node[up] at ($(e.east)!0.5!(out)$) {$\mathcal{E}(\shown)\overset{\J?}{=}$};
  \node[dn] at ($(e.east)!0.5!(out)$) {$\nxt$};

  \node[who, below=4mm of c] (wc) {$\mathcal{C}$};
  \node[who, below=4mm of v] (wv) {$\mathcal{V}$};
  \node[who, below=4mm of u] (wu) {$\mathcal{C},\mathcal{V},\mathcal{M}$};
  \node[who, below=4mm of a] (wa) {$\mathcal{M}$};
  \node[who, below=4mm of e] (we) {$\mathcal{E},\mathcal{J}$};
  \foreach \w/\b in {wc/c, wv/v, wu/u, wa/a, we/e}
      \draw[ar, shorten >=0.5pt] (\w) -- (\b);

  \node[lbl, above=0mm of c] {an inverse problem\\expressed in words};
  \node[lbl, above=0.5mm of v] {it reproduces $x$\\at least 2 out of 3};
  \node[lbl, above=0.5mm of u] {no rival within\\$\widehat{K}+\delta$ at $L_u$};
  \node[lbl, above=0.5mm of a] {$\mathcal{M}$ can run \\a description};
  \node[lbl, above=0.5mm of e] {$\shown$ shown,\\$\nxt$ asked};
\end{tikzpicture}
\vspace{-0.5cm}
\caption{C-test ONE pipeline: items are generated followed by several filters, applied sequentially. The items at the end are those that are {\em validated}, {\em unquestionable} and {\em certified}. Stages 1 and 2 are independent of the reference machine $\M$, but stages 3 and 4 apply filters that depend on $\M$. Notation is explained in the text.}
\label{fig:pipeline}
\end{figure}

\Cref{fig:pipeline} shows the C-test ONE pipeline, with five major stages. (1) Generation: the composer works with a diverse pool of domains to integrate a number of retrieval ($r^0$) and transformation ($t^0$) operations. The composer $\C$ generates fresh sequence mechanisms, called seeds, $\pseed$, as text in natural language (e.g., ``{\tt The first letter of the chemical symbol of each atomic element}''). (2) Validation: for each seed a verifier $\V$ runs them as a standalone prompt, and checks  the level of consistency $[\V(p^0) = x]^{2@3}$ twice for a sufficient total size $N$. For instance, do we repeatedly get $x$ = ``{\tt H, H, L, B, B, C, N, O, F, N, N, ...}''? 
(3) Unquestionability for $\M$: Items are split between the prefix (the observable part to the examinee, $\shown$) and the postfix (what the examinee must guess, $\nxt$). The point of split $L_u$ is increased gradually until unquestionability holds.  
The seed is used for the composer to propose more compact unconditional representations $\{p : \V(p) = x\}$, referred to as `equivalents', such as ``{\tt enumerate first letter of elements}", and looks for unconditional mechanisms that reproduce the prefix but diverge on the postfix $\{q : \V(q)_{1:n} = \shown, \V(q)_{n+1:n+k} \neq \nxt\}$, known as `rivals', such as ``{\tt Loop H, H, L, B, B, C forever}'' diverging from $n=7$. The equivalent with lowest cost, $\pK$, is used as a proxy of complexity $\Khat^\V_\M(x) = c_M(\pK)$.  
If there is no rival $q$ of a similar cost $c_\M(q)$ to $c_\M(\pK)$ then the item is unquestionable at $n$.  (4) %
Conditional certification for $\M$. Conditional hints, such as ``chemical elements" (together with the seed, equivalents, etc.) are explored and checked on $\M$, i.e., $\{p : \M(p|\shown) = \nxt\}$, and metrics such as $\Ktc$ are derived by minimising the conditional cost and then adding the logarithm of the number of inference tokens. Finally, (5) Evaluation applies the battery to any examinee $\E$. Difficulty and performance can be compared to draw characteristic curves.

Asking the composer $\C$, an LLM, to generate a diversity of items would not work by repeating the same prompt, with or without temperature. To ensure some diversity, the generator works with a set of domains, {\em themes} that are created and curated by the composer, such as ``chess openings'' and ``Beaufort scale'. The database grows as more domains are required, and the composer always keeps track of the ones already created to avoid repetition. While we could inject domains manually, for the current battery everything has been generated automatically. 

To ensure diversity in retrieval and transformation levels in the item, the generator aims at populating a matrix, ranging between $r=0..3$ and $t=0..7$. The chosen $r^0$ and $t^0$ are sent to the composer to create a seed with that number of components. However, because the composer can miscalculate or the seeds $p^0$ may have retrievals and transformations that cancel, the actual values of $r$ and $t$ may be different. Because of this, the generator looks actively for gaps and interpolates between demanded values and estimated values, correcting this with a linear function. As we can see in the $(r,t)$ matrices in the appendix, the procedure manages to fill almost every cell.

\Cref{tab:filters} gives a formal expression of the filters applied to each candidate seed. 
For more details about the pipeline and the prompts used for the composer, we refer to Appendices \ref{app:pipeline} and \ref{app:prompts}.

\begin{table}[t]
\centering \small \caption{The three filters that are applied after generation. The final row shows how evaluation works: an examinee
$\mathcal{E}$ is shown $x_{1:\Lu}$ and asked for $\kf = 5$ more terms. All terms must be equal. %
}
\label{tab:filters}
\begin{tabular}{@{}llll@{}}
\toprule
 & condition & costing & attempts\\
\midrule
validity at $N$ & $[\V(\pseed)]^{2@3}_{1:N} = [\V(\pseed)]^{2@3}_{1:N} = x_{1:N}$ & --- & 2 out of 3 attempts equal, twice\\[2pt]
unquestionability at $\Lu$ &
  $\cost_\M(q) > \Khat^{\V}_{\M}(x) + \delta$ for every rival at $\Lu$ &
  $\cost_\M(\cdot)$ & rivals checked once\\[2pt]
certified at $\Lu$ &
  $\exists\, p : \M(p \mid x_{1:\Lu}) = x_{\Lu+1:\Lu+\ku}$ &
  $\cost_\M(\cdot \mid x_{1:\Lu})$  & one run (greedy decoding) \\
\midrule
full correctness at $\kf$ &
  $\,\mathcal{E}(x_{1:\Lu})\,_{1:\kf}
   = x_{\Lu+1:\,\Lu+\kf}$ &
  --- & average of three attempts reported \\
\bottomrule
\end{tabular}
\end{table}

\section{Experimental design}\label{sec:experimental}

Experiments must be sufficient to answer several research questions. RQ1: Can the pipeline create a battery of inverse sequences that are valid, unquestionable and certifiable using small LLMs? RQ2: Are they diverse in retrieval and transformation components and difficulty? %
RQ3: Does $\Ktc$ work as a calibrated measure of difficulty for both weak and strong models?  
RQ4: Is the ability
obtained from the characteristic curves using $\Ktc$ aligned with populational ability metrics? RQ5: What is the
supraintelligence window that we can obtain with this procedure, as a sign of its scalability? and RQ6: How does
the power of the reference machine influence all these questions?

As a self-imposed constraint for sustainability and reproducibility, the experimental design should not require more than the necessary resources to answer the research questions. Accordingly, we estimated that for RQ1, RQ2, RQ4 and RQ5 we had to set a target of a minimum of 150 generated items with at least 30  being certified for each and every reference machine. For RQ1, RQ3, RQ5 and RQ6, we chose three reference machines that can fit in a T4 GPU. For RQ3, RQ4 and RQ5 we used 12 LLMs of different families, including commercial and open-weight. Also, for the logic in the pipeline we chose several Google Gemini variants because of their balance of speed, quality and price. All the choices are shown in \Cref{tab:machines}. Three Anthropic models as examinees had high refusal percentages (see \Cref{tab:refusal-qwen14B} in the appendix), and we used versions with fallback, represented with *.

\newcolumntype{P}[1]{>{\raggedright\arraybackslash}p{#1}}
\begin{table}[t]
\centering\scriptsize\setlength{\tabcolsep}{4pt}
\caption{LLMs used. Models with a star used fallback, to reduce refusal percentages. Temperature 0 unless stated.}
\label{tab:machines}
\begin{tabular}{@{}l l P{3.4cm} P{10.0cm}@{}}
\toprule
 & Role & Requirements & LLMs used \\
\midrule

\multirow{3}{*}{$\M$} & \multirow{3}{*}{reference}
  & \multirow{3}{=}{forward-pass log-probs, and enough reasoning for sequences}
  & \textbf{Qwen3-14B} (4-bit NF4), a reasoning model \\
 & & & \textbf{Qwen3-8B} (4-bit NF4), the same family and tokeniser as Qwen3-14B \\
 & & & \textbf{Phi-4-14B-reasoning} (4-bit NF4), another lineage and tokeniser \\
\addlinespace[4pt]

$\V$ & verifier
  & running sequences
  & \textbf{Gemini 3.6 Flash}, a fast, cheap model \\
\addlinespace[4pt]

$\C$ & composer
  & complex prompts; creative
  & \textbf{Gemini 3.6 Flash}, a fast, cheap model\\
\addlinespace[4pt]

$\T$ & typer
  & determining $\hat{r}$ and $\hat{t}$
  & \textbf{Gemini 3.6 Flash}, at temperature $0.4$ \\
\addlinespace[4pt]

$\J$ & judge
  & short, fast replies
  & \textbf{Gemini 3.6 Flash} scores a reply; \textbf{Gemini 3.5 Flash-Lite}
    for certification budget \\
\addlinespace[4pt]

$\E$ & examinee
  & diversity in kind and power
  & {\bf GPT-5.6-sol, GPT-5-nano, Claude Opus~5*, Claude Sonnet~5*, Claude Haiku~4.5, Claude Fable~5.1*, Gemini 3.8 Flash, Gemini 3.6 Flash, Gemini
    3.5 Flash-Lite, Gemini 3.1 Flash-Lite,  
    OLMo-2-7B} and {\bf Qwen2.5-1.5B} \\

\bottomrule
\end{tabular}
\end{table}

The implementation of \Cref{def:kllm} and \Cref{def:kt} is instantiated for sequences as follows:
\[
K_\M^\V(x_{1:n+k}) \;\defeq\; \min_p\left\{\, c_\M(p) \;:\; \V(p) = x_{1:n+k} \,\right\},
\]
\[
Kt_\M(x_{n+1:n+k}|x_{1:n}) \;\defeq\; \min_p\left\{\, c_\M(p|x_{1:n}) + \log_{2} \tau_\M(p,x_{1:n}) \;:\; \M(\langle p,x_{1:n}\rangle) = x_{n+1:n+k} \,\right\},
\]
We have a specific prompt that is standard for all verifications (see \Cref{fig:examprompt} in the appendix), instead of a simple concatenation $\langle p,\shown\rangle$. Finally, we used many other baselines for comparisons, such as 
\[
Kc_\M(x_{n+1:n+k}|x_{1:n}) \;\defeq\; \min_{p} \big\{\, \cost_\M(p \mid \shown)
   \;:\; \M(\langle p,\shown\rangle) = \nxt \,\big\} \;+\; \cost_\M(\shown)
\label{eq:h}
\]
which adds a term measuring how unusual the prefix is, instead of the time to run the program. %
We calculated $r$ and $t$ (original and estimated) and the length in characters $\ell_{char}$ of the programs $p^0$ (the seed), $\pbare$ (a bare version of the seed), $\pK$ (the unconditional program %
with lowest cost), and $\pKc$ (the conditional program with lowest cost).

\section{Results}\label{sec:results}

We will include a selection of plots and tables, mostly for Qwen3-14B, since it is the reference machine that certified the highest number of items in the battery. More plots for Qwen3-14B can be found in Appendix \ref{app:qwen3-14B}, and all the results for Qwen3-8B and Phi-4-14B-reasoning can be found in appendices \ref{app:qwen3-8B} and \ref{app:phi-4} respectively. 

The full battery had 151 items, of which 107 were valid. From those, Qwen3-14B made 71 unquestionable, of which 51 were certified. Qwen3-8B increased unquestionable items to 76 but certified only 32, due to its incapability of successfully running any hint or program for complex items given the budget. Finally, Phi-4-reasoning is placed in between, with 80 unquestionable and 38 certified. {\bf RQ1 is answered affirmatively.}

For Qwen3-14B specifically, \Cref{fig:heatQwen14B} shows the percentage of items being solved by each examinee for each cell in the grid. \Cref{tab:chain-qwen3-14b} in the appendix gives the full breakdown more explicitly. 
High $t$ with low $r$ seems to be a specific area where models tend to fail. {\bf RQ2 is answered affirmatively.}  

\begin{figure}[h]
\centering
\includegraphics[width=1.0\linewidth]{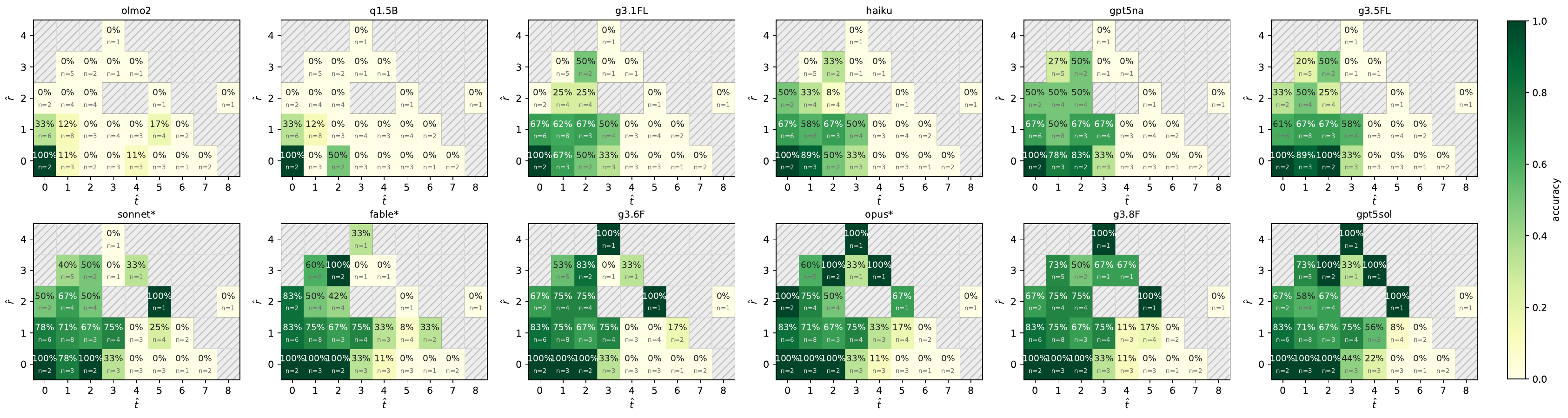}
\vspace{-0.7cm}
\caption{The 12 LLM examinees on the ($\hat{r},\hat{t}$) grid and accuracy for the unquestionable items for Qwen3-14B.}
\label{fig:heatQwen14B}
\end{figure}

\Cref{tab:cert-metricsQwen14b} shows Spearman correlations for a range of metrics per model for the 51 certified items under Qwen3-14B. %
$\Ktc$ works very well (within sampling error with $\Kcc$ and $\Kco$, see Appendix \ref{app:stats}) with strong correlations per model (except the weak examinees) and globally. %
{\bf RQ3 is answered affirmatively}. 

\begin{table}[t]
\centering
\scriptsize\setlength{\tabcolsep}{2.0pt}
\caption{Spearman correlations between several metrics and full correct response for the 51 valid, unquestionable and certified items under Qwen3-14B. Each item had three repetitions per examinee. All metrics in \Cref{tab:cert-metricsQwen14b-full}.}
\label{tab:cert-metricsQwen14b}
\begin{tabular}{lrrrrrrrrrrrrr}
\toprule
measure & olmo2 & q1.5B & g3.1FL & haiku & gpt5na & g3.5FL & sonnet* & fable* & g3.6F & opus* & g3.8F & gpt5sol & \textsf{all} \\
\midrule
$c_\M(x_{1:n})$ & $-0.44$ & $-0.45$ & $-0.59$ & $-0.66$ & $-0.63$ & $-0.61$ & $-0.50$ & $-0.62$ & $-0.58$ & $-0.57$ & $-0.56$ & $-0.55$ & $-0.68$ \\
$\tau_{\mathcal{M}}$ & $-0.34$ & $-0.42$ & $-0.60$ & $-0.68$ & $-0.63$ & $-0.57$ & $-0.46$ & $-0.56$ & $-0.57$ & $-0.52$ & $-0.59$ & $-0.59$ & $-0.65$ \\
$\smash[t]{\hat{K}^{\V}_{\M}}(x_{1:n+k})$ & $-0.43$ & $-0.45$ & $-0.60$ & $-0.66$ & $-0.63$ & $-0.65$ & $-0.54$ & $-0.61$ & $-0.62$ & $-0.57$ & $-0.60$ & $-0.56$ & $-0.69$ \\
$\smash[t]{\hat{K}_{\mathcal{M}}}(x_{n+1:n+k} \mid x_{1:n})$ & $-0.31$ & $-0.45$ & $-0.73$ & $-0.73$ & $-0.73$ & $-0.76$ & $-0.63$ & $-0.73$ & $-0.68$ & $-0.66$ & $-0.69$ & $-0.66$ & $-0.78$ \\
$\smash[t]{\widehat{Kc}_{\mathcal{M}}}(x_{n+1:n+k} \mid x_{1:n})$ & $-0.43$ & $-0.49$ & $-0.70$ & $-0.73$ & $-0.70$ & $-0.71$ & $-0.60$ & $-0.70$ & $-0.66$ & $-0.65$ & $-0.65$ & $-0.63$ & $-0.77$ \\
$\smash[t]{\widehat{Kt}_{\mathcal{M}}}(x_{n+1:n+k} \mid x_{1:n})$ & $-0.37$ & $-0.50$ & $-0.75$ & $-0.75$ & $-0.74$ & $-0.74$ & $-0.63$ & $-0.72$ & $-0.67$ & $-0.66$ & $-0.68$ & $-0.65$ & $-0.79$ \\
\bottomrule
\end{tabular}
\end{table}

\Cref{fig:abilitiesQwen14B} (left) shows the logistic fit of the points for the 51 certified examples for Qwen3-14B ($2N$ prior points added at (0,1) as anchor). Empirical curves in the appendix (\Cref{fig:curvesQwen14B}). The plot on the right compares the ability of each model with Item Response Theory estimates of the abilities of the same models according to Epoch's Capability Index (ECI) \citep{ho2025rosetta}. Correlation is high, despite the %
different scales.  %
{\bf RQ4 is answered affirmatively}.

\begin{figure}[h]
\centering
\includegraphics[width=0.43\linewidth]{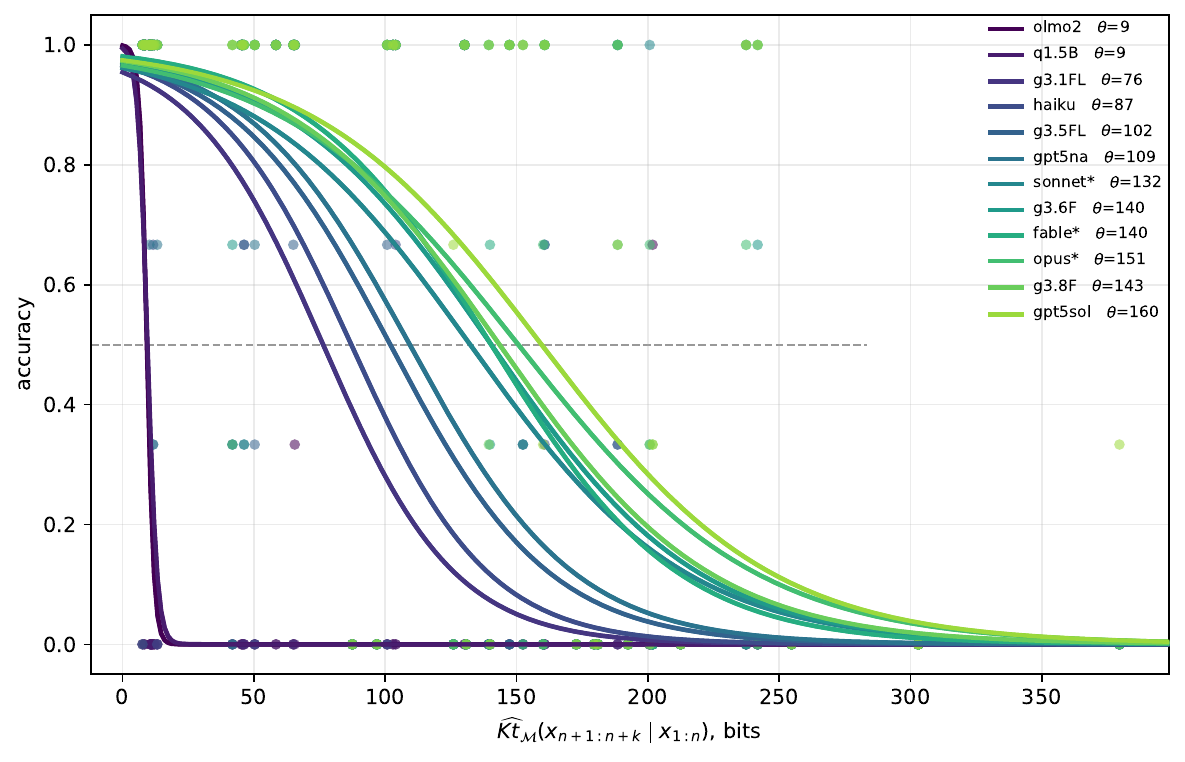}
\includegraphics[width=0.43\linewidth]{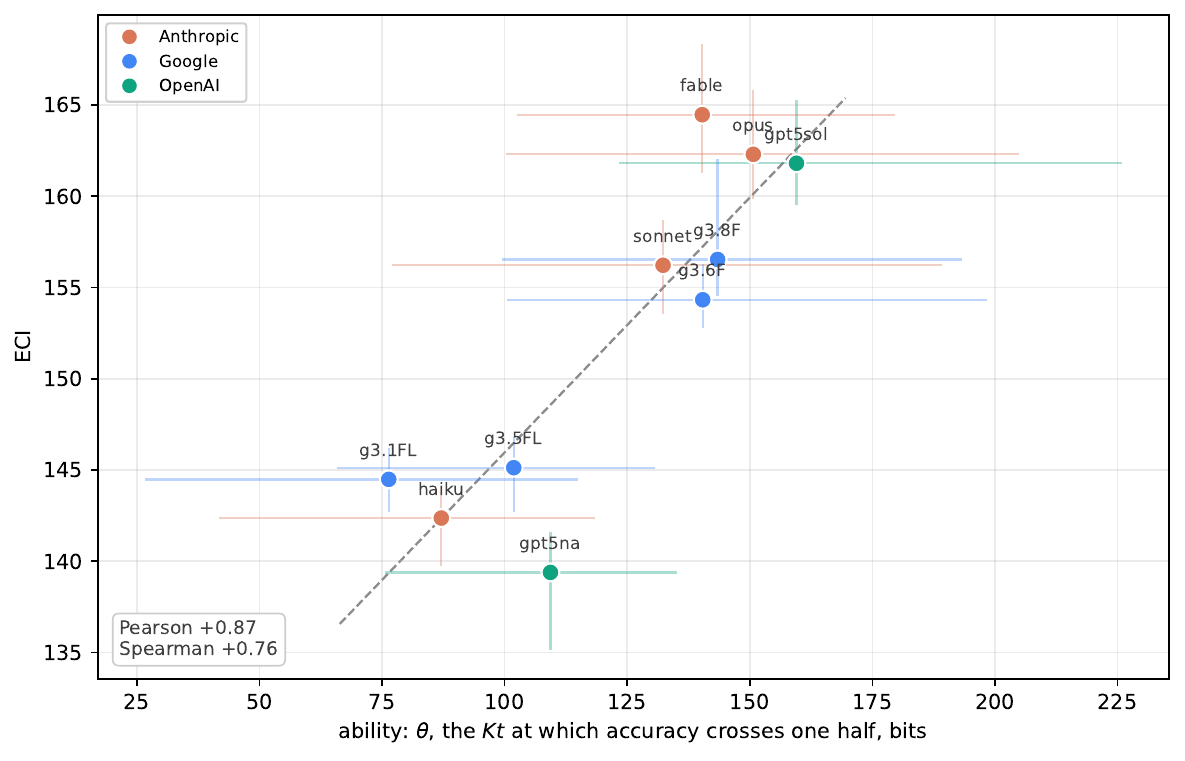}
\vspace{-0.3cm}
\caption{Left: Agent characteristic curves for the 12 LLM, with $\Ktc$ provided by Qwen3-14B (51 certified examples), with a logistic fit, where the ability of each model $\theta$, measured in bits, is estimated by the point where performance is 0.5. Right: Comparison between each model's {\em ability} and the psychometric, populational score given by Epoch's Capability Index (ECI) \citep{ho2025rosetta}. Only 10 of our 12 are listed in ECI.}
\label{fig:abilitiesQwen14B}
\end{figure}

\Cref{tab:windows} shows the infra and supra-intelligence windows  for the three reference models. The infra window runs from the item with lowest $\Ktc$ to the point where the fitted logistic curve crosses 2/3 accuracy. Similarly, the supra window runs from where it falls through 1/3 to the item with highest  $\Ktc$.  %
While the interval between the windows is where information is high, these windows show the margin for discrimination in both ends of the evaluation spectrum. Qwen3-8B does not have enough resolution in the higher end of the scale. Qwen3-14B has limited sensitivity in the other end, but its upper end shows the scalability of the construction. Even for a model such as gpt5sol, a much weaker model such as Qwen3-14B can generate a very ample, yet sparse window between 190-380 where the fitted curve falls below 1/3 (Appendix \ref{app:stats}). %
{\bf RQ5 is answered by the 
supraintelligence window being bounded and sparse}, yet potentially accommodating models that are much stronger than the reference machine.

\begin{table}[t]\centering\scriptsize
\caption{Infra- and supra-intelligence windows of all examinees per reference machine. The supra ranges for Qwen3-14B are wide, indicating it %
can %
potentially 
challenge frontier models. %
Ill-defined windows in grey.
}
\label{tab:windows}
\begin{tabular}{@{}l r@{\,--\,}l c r@{\,--\,}l r@{\,--\,}l c r@{\,--\,}l r@{\,--\,}l c r@{\,--\,}l@{}}\toprule
 & \multicolumn{5}{c}{Qwen3-8B} & \multicolumn{5}{c}{Qwen3-14B} & \multicolumn{5}{c}{Phi-4} \\
\cmidrule(lr){2-6}\cmidrule(lr){7-11}\cmidrule(lr){12-16}
 & \multicolumn{2}{c}{infra} & span & \multicolumn{2}{c}{supra} & \multicolumn{2}{c}{infra} & span & \multicolumn{2}{c}{supra} & \multicolumn{2}{c}{infra} & span & \multicolumn{2}{c}{supra} \\
\addlinespace[2pt]
\midrule
gpt5sol & 7 & \textcolor{gray}{194} & \textcolor{gray}{88} & \textcolor{gray}{283} & 190 & 8 & 129 & 61 & 190 & 380 & 7 & 125 & 53 & 178 & 187 \\
fable* & 7 & 167 & \textcolor{gray}{72} & \textcolor{gray}{238} & 190 & 8 & 116 & 49 & 165 & 380 & 7 & 130 & 45 & 174 & 187 \\
g3.8F & 7 & 166 & \textcolor{gray}{79} & \textcolor{gray}{244} & 190 & 8 & 116 & 55 & 171 & 380 & 7 & 122 & 40 & 162 & 187 \\
opus* & 7 & 176 & \textcolor{gray}{90} & \textcolor{gray}{266} & 190 & 8 & 119 & 63 & 182 & 380 & 7 & 128 & 50 & 177 & 187 \\
g3.6F & 7 & 156 & \textcolor{gray}{79} & \textcolor{gray}{235} & 190 & 8 & 113 & 55 & 168 & 380 & 7 & 119 & 35 & 155 & 187 \\
sonnet* & 7 & 146 & \textcolor{gray}{89} & \textcolor{gray}{235} & 190 & 8 & 104 & 57 & 161 & 380 & 7 & 109 & 49 & 158 & 187 \\
gpt5na & 7 & 127 & \textcolor{gray}{93} & \textcolor{gray}{219} & 190 & 8 & 88 & 43 & 131 & 380 & 7 & 82 & 49 & 131 & 187 \\
g3.5FL & 7 & 110 & 76 & 186 & 190 & 8 & 81 & 42 & 123 & 380 & 7 & 72 & 43 & 115 & 187 \\
haiku & 7 & 88 & 65 & 153 & 190 & 8 & 69 & 36 & 105 & 380 & 7 & 59 & 38 & 97 & 187 \\
g3.1FL & 7 & 89 & 71 & 160 & 190 & 8 & 59 & 35 & 94 & 380 & 7 & 71 & 54 & 125 & 187 \\
q1.5B & 7 & 21 & 13 & 33 & 190 & 8 & 8 & 3 & 11 & 380 & 7 & 7 & 1 & 9 & 187 \\
olmo2 & 7 & 9 & 4 & 13 & 190 & 8 & 8 & 2 & 10 & 380 & 7 & 8 & 1 & 9 & 187 \\
\bottomrule
\end{tabular}
\end{table}

\Cref{fig:sample} shows how the $\Ktc$ metric changes between two reference machines of the same family. Not only are many items not certified by Qwen3-8B, as compared to Qwen3-14B, but their difficulty is higher. Qwen3-8B is a worse compressor, but mostly because many equivalent programs and hints that reduce $\Ktc$ in Qwen3-14B do not work for Qwen3-8B. Accordingly, {\bf RQ6 is partially answered positively} by this plot, but the appendix contains the corresponding plots and tables for all the reference machines, to observe more precisely how they affect the generation and results. The other two pairwise combinations can be found in \Cref{fig:sample-qwen3-8b} and \Cref{fig:sample-phi-4-14B-reasoning} in the appendix. %

\begin{figure}[h]
\centering
\IfFileExists{figures/fig_shift_kt_paper_full_qwen3-14b.pdf}{\includegraphics[width=.77\linewidth]{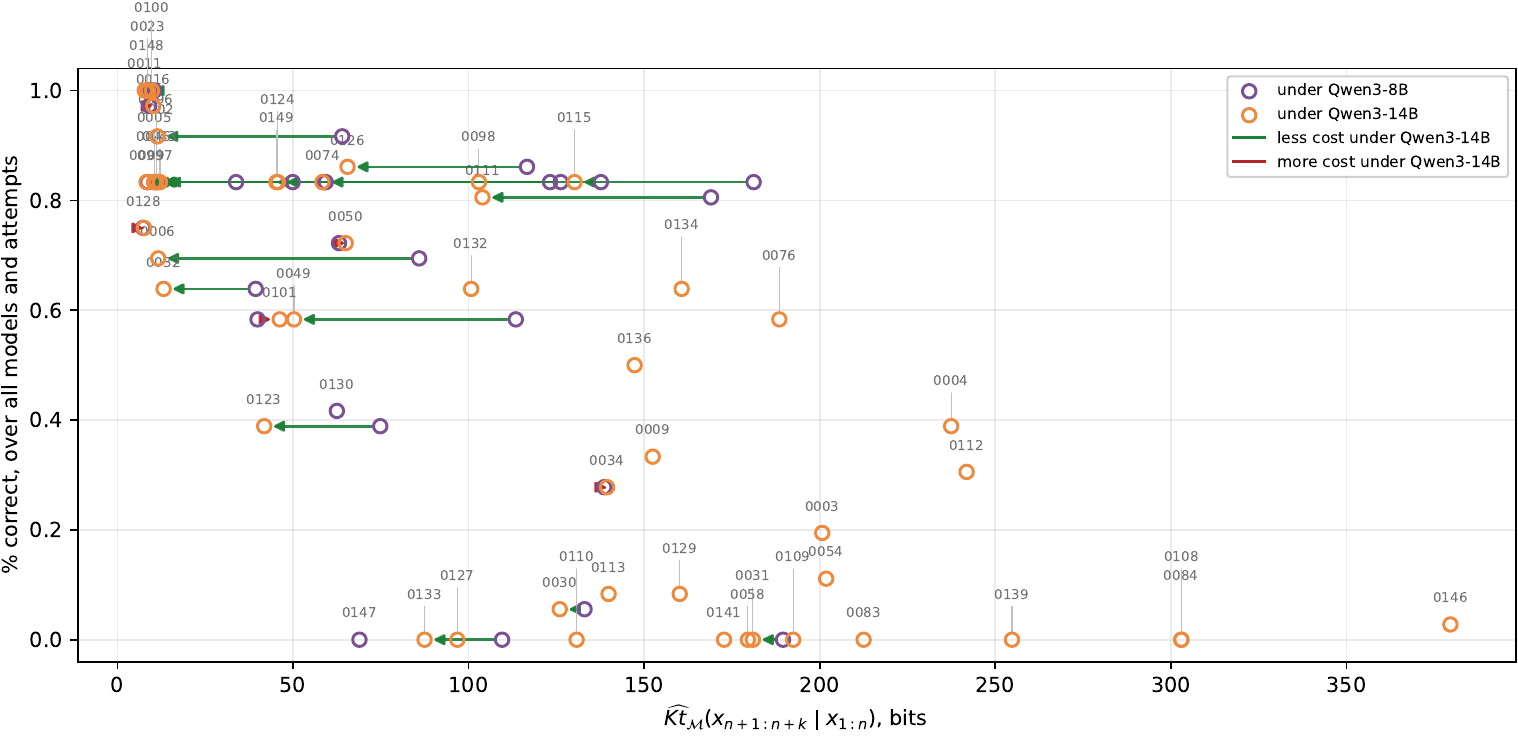}}\\
\vspace{-0.3cm}
\caption{Comparison of $\Ktc$ between Qwen-3-8B and Qwen-3-14B vs average results of 12 other LLMs on the 71 unquestionable items for Qwen-3-14B. Success decreases with $\Ktc$. The weaker model (Qwen-3-8B, purple) brings more refinement for low $\Ktc$ (left), but struggles to cost many items of high $\Ktc$ (right). An arrow pointing left means Qwen-3-14B describes the item at less cost. %
Values near zero are empty programs, %
with $Kt= 0 + \log_2\tau$. %
}
\label{fig:sample}
\end{figure}

\section{Limitations and Future Work}\label{sec:limitations}

The choice of very small LLMs as reference machines and a commodity model as composer %
shows the pipeline works with low resources. Conducting low-budget evaluations that make frontier models fail predictably and gradually gives a positive signal towards the scalability of evaluations. 
So far, the use of Kolmogorov complexity has usually been objected to because of its incomputability and subjectivity. We use Levin's $Kt$, which is computable, as it considers the number of tokens in the reasoning steps in a refurbished universal search expression that turns out to be fundamental \citep{achille2026universal}. Subjectivity to $\M$ is partially addressed in terms of ratio scales properties in Appendix \ref{app:ratio}. %
Of course, given the infinite options for $\M$, an explicit choice as an LLM makes the complexity measure both meaningful for humans and also for the current paradigm in AI.

One of the limitations comes precisely from the bias on human language and culture: the generator produces text in English, dominated by Western culture. This is a convenience, %
but the battery and costing can be parametrised to other languages and cultures by changing $\C$ and $\M$. 
A second limitation is that C-test ONE is based only on sequences. %
The purpose of this paper was to show the approach works. %
The next step could be non-sequential problems, such as agential scenarios %
 (e.g., mini-games, \citep{ying2507assessing,ying2026ai}) where the solution of the game (not its description as in \citep{legg2007universal}) is quantified according to $Kt(p|x)$. But note that applying the processes of minimisation and unquestionability a posteriori, may expose the weakness of existing batteries such as \citep{foundation2026arc}. Instead, generating items adapting our procedure can create more robust tests. Nonetheless, the possibility of calculating intrinsic properties of existing benchmarks, such as their difficulty, and their fluid and crystallised components, should be at the core of a staircase of batteries we should use for frontier AI.

As future work, the generator can be improved to speed up the estimates of $\Ktc$ \citep{orseau2024exponential,li2026levin,tuero2026structure}, and more reference machines, including multimodal, could always be explored, and combined. Beyond these optimisations, the potential of this paper lies in rebuilding the foundations of  %
AI evaluation broadly. 

The lesson learnt is that with modest budget and models we can generate a systematic class of verifiable questions beyond current frontier capabilities, with their difficulty measured in a theoretically-grounded way, populating a supra-intelligence window that extends our potential for evaluating the frontier. C-test ONE is the first in a series achieving this. Let others climb the next steps in the ladder.

\subsection*{AI use statement}%
We used so-called generative AI in several stages of this project, other than the test itself (with LLMs being used as  verifiers, composers or reference machines). Several LLMs were used for idea generation and brain-storming by the authors, and helped with locating and processing related work, but all papers and books that are cited were read, or at least scanned for the relevant passages, by the authors. References were checked manually with standard sources (e.g., google scholar), which may still contain errors. LLMs were used  extensively for vibe coding, processing the results, and generating tables, plots and other result analyses. The logs and results were the main way in which the code was verified to work correctly. LLMs were also used to draft early versions of some sections, especially those associated with the experiments, but all text was finally rewritten by the authors. Some illustrations were also partly generated by AI. We take full responsibility for all the content and associated code and datasets, especially the parts that were drafted by AI.

\subsection*{Ethics statement}

This work does not have any special ethical issues other than the general concerns of a crucial field such as AI evaluation. Problems in underestimation or overestimation of capabilities may have an impact on the appropriate deployment of AI, including areas such as fairness and safety. The focus of this paper is put on the scalability of evaluations, and some level of delegation of this evaluation, which may entail some risks. Another element of concern is the implications of this research for the environment. We have devised the experimental setting in such a way that we performed the experiments that were sufficient to answer our research questions and support any claim that we made. We also report a detailed and comprehensive use of compute and its environmental impact in Appendix \ref{sec:resources}. On purpose, we did not include very large reference machines or a very wide battery of API models to keep our impact as low as possible. We think the variations, ablations, repetitions and numbers of models and items that we are using are reasonable. Apart from this, we did not conduct any experiment with humans or animals, or included data from these sources directly (although we used API calls, and we cannot be sure of how proprietary models were trained). 

\subsection*{Reproducibility statement}

Appendices \ref{app:pipeline} to \ref{app:params} include a detailed description of the pipeline, the prompts used, the description of all notebooks and data files in the project, as well as all the parameters we used in the experiments. We also include the full battery in appendix \ref{sec:battery}. 
All code and data are available at \url{https://github.com/jorallo/ctest-one}.
Also, despite our self-imposed limitation on the number of GPU/CPU hours and LLM tokens, our detailed report of the use of compute and tokens in appendix \ref{sec:resources} helps other researchers anticipate their costs of reproduction.%

\paragraph*{Acknowledgments}
I'm thankful to Jan Arne Telle for insightful comments on a very early version of this project.

This paper would have not been possible without the help of the other co-author of the original C-test paper, Neus Minaya-Collado, who this time helped with extra `compute' for this project. The time was ripe after thirty years.

\bibliography{ctest}
\bibliographystyle{iclr2027_conference}

\newpage

\appendix

$\:$

$\:$

\etocdepthtag.toc{mtappendix}    %

\etocsettagdepth{mtmain}{none}             %
\etocsettagdepth{mtappendix}{section}   %
\renewcommand{\contentsname}{Appendix contents}
\tableofcontents

\newpage

\section{Full Results for Qwen3-14B}\label{app:qwen3-14B}

Here we include some other plots and tables for Qwen3-14B specifically that did not fit in the main paper. Let us start with the detailed results from the 12 models we used as examinees. More precisely, they were 15, as we tried two versions of Sonnet, Opus and Fable, because on a first attempt they had a high degree of refusal, as shown in \Cref{tab:refusal-qwen14B}. Our guess for this high level of refusal is that some items look cryptic to the safety filters, as if they were jailbreaks. We had three options: remove most of the Anthropic family, remove the items affected or include them with a clear indication that they are routing items to other models. We explored the last two\footnote{Refusals are not spread evenly. The plain Anthropic models refused at least
once on 22 of the 51 items certified under Qwen3-14B, with a median
$\widehat{Kt}^{|}$ of 176 bits against 45 for the rest, including 8 of the 10
items above 190 bits. %
Treating the starred models' non-answers as
missing, instead of as failures, leaves the correlation for $\widehat{Kt}^{|}$ unchanged, at $-0.79$.}, but include the latter in the paper: we kept the version with fallback for all the experiments in the paper, shown with a star.  \Cref{tab:refusal-qwen14B} includes all the results, finalising all the runs for the non-fallback versions,  for completeness\footnote{The plain Claude Sonnet 5 answered its first 54 items with a budget of 8000
tokens, and the rest with 16000. Its fallback version ran all items with 16000, so their counts of cut replies
are not directly comparable, but the increase made no difference to how often it was cut.}. It also shows that some of the outputs that were still cut by budget or refusal did not disappear (especially for Sonnet*). The discrepancy of budget comes from the first runs trying to determine what an appropriate budget could be. %
Where we could test it, increasing the budget did not have an effect. 
In particular, for Sonnet 5, doubling the budget from 8000 to 16000 did not reduce the
   share of cut replies, 23\% against 25\%, so a larger budget is
   unlikely to change its results. GPT-5 nano was run only at 8000, but we followed our criterion of assigning a budget roughly proportional to the potential and needs of the model. In the era of reasoning models, an examinee is a pair model and budget. In this case, increasing the budget for any of the models would have not had a noticeable effect. The models and budgets are sufficient for the purpose of our research questions (such as the comparison with ECI).

\begin{table}[h]
\centering
\small
\caption{Reason of non-answers for the 71 valid, unquestionable items under Qwen3-14B. We tried two versions of sonnet, opus and fable to reduce refusal, and we kept the one with fallback, marked with a star, in which a refused
request is answered by another model that Anthropic recommends. Comparing the two versions suggests that the fallback served roughly two
fifths of the replies of Fable 5.1*, a tenth of those of Opus 5*, and almost
none of Sonnet 5*'s. These models can change their delegation strategies, so these results may change in future runs. All scores, plots and calculations are performed with these 12 models, and we consider any of these reasons (cut, refusal, timeout, malformed or other) as a failure.  A reply is cut when it reaches the token budget before the model finishes,
   and `cut' only counts those that were cut {\em and} gave no answer. A refusal or a truncation is not evidence that the model could not solve the item, given more tokens or finding the cause for the rejection. The last column is the score on the rows that were answered.}\label{tab:refusal-qwen14B}
\begin{tabular}{lrrrrrrrrr}
\toprule
examinee & budget (tokens) & rows & \textsf{cut} & \textsf{refusal} & \textsf{timeout} & \textsf{malformed} & \textsf{other} & any & score on the rest \\
\midrule
olmo2 & 2048 & 213 & -- & -- & -- & -- & -- & 0 & 10.1\% \\
q1.5B & 2048 & 213 & -- & -- & -- & -- & -- & 0 & 10.5\% \\
g3.1FL & 8000 & 213 & -- & -- & -- & -- & -- & 0 & 36.2\% \\
haiku & 8000 & 213 & -- & -- & -- & -- & -- & 0 & 36.8\% \\
gpt5na & 8000 & 213 & 73 & -- & -- & -- & -- & 73 & 62.3\% \\
g3.5FL & 8000 & 213 & -- & -- & -- & -- & -- & 0 & 44.6\% \\
fable & 16000 &213 & 1 & 100 & -- & -- & -- & 101 & 86.1\% \\
sonnet* & 16000 & 213 & 50 & 6 & -- & -- & -- & 56 & 69.0\% \\
sonnet & 8000/16000 & 213 & 50 & 4 & -- & -- & 1 & 55 & 69.6\% \\
fable* & 16000 & 213 & 4 & 9 & -- & -- & -- & 13 & 60.1\% \\
g3.6F & 8000 &213 & -- & -- & -- & -- & -- & 0 & 58.6\% \\
opus & 16000 & 213 & 16 & 28 & -- & -- & -- & 44 & 74.1\% \\
opus* & 16000 &213 & 12 & 6 & -- & -- & -- & 18 & 67.3\% \\
g3.8F & 8000 & 213 & -- & -- & -- & 1 & -- & 1 & 62.8\% \\
gpt5sol & 16000 & 213 & -- & -- & -- & -- & -- & 0 & 65.7\% \\
\bottomrule
\end{tabular}
\end{table}

In \Cref{tab:scoring-qwen14B} we show the scores of the 12 models for all unquestionable items following different criteria. In general, in this paper, we use {\sf full:mean}, i.e., the average of the three attempts, each of them considered as the continuation of the sequence must agree fully.

\begin{table}[h]
\centering
\small\setlength{\tabcolsep}{4pt}
\caption{Six different ways of scoring for the 71 valid, unquestionable items under Qwen3-14B.  \textsf{\small prefix\%} gives the percentage of terms before its first mistake. \textsf{\small full} scores the same attempt all or nothing: every term right, or none of the credit. Within each, as we have three repetitions when evaluating, \textsf{\small worst} and \textsf{\small best} take the least and the most of the three attempts and \textsf{\small mean} averages them. \textsf{\small agree} says whether the three said the same thing, right or wrong. An attempt with no reply is left out of the agreement, as three refusals are not really agreeing on anything.}\label{tab:scoring-qwen14B}
\begin{tabular}{lrrrrrrrr}
\toprule
 & \multicolumn{4}{c}{\textsf{prefix\%}} & \multicolumn{4}{c}{\textsf{full}} \\
\cmidrule(lr){2-5}\cmidrule(lr){6-9}
examinee & worst & mean & best & agree & worst & mean & best & agree \\
\midrule
olmo2 & 8.2\% & 10.1\% & 12.7\% & 100.0\% & 7.0\% & 8.9\% & 11.3\% & 100.0\% \\
q1.5B & 10.4\% & 10.5\% & 10.7\% & 90.4\% & 8.5\% & 8.5\% & 8.5\% & 100.0\% \\
g3.1FL & 36.1\% & 36.2\% & 36.3\% & 100.0\% & 31.0\% & 31.0\% & 31.0\% & 100.0\% \\
haiku & 29.9\% & 36.8\% & 42.8\% & 33.0\% & 28.2\% & 31.9\% & 35.2\% & 29.6\% \\
gpt5na & 34.6\% & 40.9\% & 47.9\% & 58.3\% & 31.0\% & 38.0\% & 45.1\% & 52.2\% \\
g3.5FL & 34.9\% & 44.6\% & 55.2\% & 39.4\% & 28.2\% & 37.6\% & 47.9\% & 35.2\% \\
sonnet* & 45.9\% & 50.9\% & 56.6\% & 73.7\% & 42.3\% & 47.4\% & 53.5\% & 70.4\% \\
fable* & 51.3\% & 56.4\% & 61.7\% & 62.4\% & 47.9\% & 52.1\% & 56.3\% & 60.6\% \\
g3.6F & 53.8\% & 58.6\% & 63.4\% & 58.3\% & 50.7\% & 54.0\% & 57.7\% & 57.7\% \\
opus* & 55.2\% & 61.6\% & 68.2\% & 74.2\% & 52.1\% & 57.3\% & 62.0\% & 69.7\% \\
g3.8F & 56.3\% & 62.5\% & 69.6\% & 65.1\% & 52.1\% & 57.3\% & 63.4\% & 62.0\% \\
gpt5sol & 56.9\% & 65.7\% & 74.4\% & 71.0\% & 52.1\% & 59.2\% & 67.6\% & 69.0\% \\
\bottomrule
\end{tabular}
\end{table}

\Cref{tab:chain-qwen3-14b} shows the distribution grid of its 71 unquestionable items by measured ($\hat{r},\hat{t}$). We see that most items are concentrated in low and medium transformations and low retrieval. 

\begin{table}[h]
\centering
\small
\caption{Number of \emph{valid} / \emph{unquestionable} / \emph{certified} items under Qwen3-14B, by the cell they fell in.}
\label{tab:chain-qwen3-14b}
\begin{tabular}{@{}lrrrrrrrrrr@{}}
\toprule
$\hat{r}$ & $\hat{t}=0$ & $\hat{t}=1$ & $\hat{t}=2$ & $\hat{t}=3$ & $\hat{t}=4$ & $\hat{t}=5$ & $\hat{t}=6$ & $\hat{t}=7$ & $\hat{t}=8$ & all \\
\midrule
0 & 2/2/2 & 4/3/3 & 10/2/2 & 9/3/3 & 6/3/1 & 3/3/3 & 2/1/1 & 3/2/0 & 1/0/0 & 40/19/15 \\
1 & 6/6/5 & 9/8/6 & 5/3/3 & 7/4/4 & 6/3/1 & 5/4/3 & 3/2/1 & -- & -- & 41/30/23 \\
2 & 2/2/1 & 4/4/2 & 4/4/3 & -- & -- & 1/1/1 & -- & -- & 2/1/0 & 13/12/7 \\
3 & -- & 5/5/4 & 2/2/2 & 2/1/0 & 1/1/0 & -- & -- & -- & 1/0/0 & 11/9/6 \\
4 & -- & -- & 1/0/0 & 1/1/0 & -- & -- & -- & -- & -- & 2/1/0 \\
\midrule
all & 10/10/8 & 22/20/15 & 22/11/10 & 19/9/7 & 13/7/2 & 9/8/7 & 5/3/2 & 3/2/0 & 4/1/0 & 107/71/51 \\
\bottomrule
\end{tabular}
\end{table}

\Cref{tab:cert-metricsQwen14b-full} shows Spearman correlations for a wide range of metrics per model and together for the 51 certified items for Qwen3-14B. It is an extension of \Cref{tab:cert-metricsQwen14b}, but here we see that the correlations of $r^0$ and $t^0$ are very similar to the measured ones. We also see that that length of the programs in number of characters ($\ell_{char}$) is a good proxy, only clearly improved by the last three rows.

The length of the seed and other programs in the process is informative
only on items the pipeline has validated and made unquestionable, and the correlation grows along the process that filters using $V$ verification and $\M$ coding. Certification also shows that a weak machine can run a description of
each item, which seed length cannot. $\widehat{Kt}^{|}$ is also a measure
and not a proxy. It is in bits under a stated reference machine, conditional
on the prefix, and a minimum over many descriptions, so it does not depend
on how verbose one description is, and it can be compared across batteries
and composers. Among the harder half of the items, above 101 bits, it also
correlates more strongly with the accuracy of the six strongest examinees,
at $-0.48$ against $-0.31$, which is the supra region we are most interested in.

\begin{table}[t]
\centering
\scriptsize\setlength{\tabcolsep}{2.0pt}
\caption{Full version of \Cref{tab:cert-metricsQwen14b}. Spearman correlations between several metrics and full correct response for the 51 valid, unquestionable and certified items under Qwen3-14B. Each item had three repetitions per examinee.}
\label{tab:cert-metricsQwen14b-full}
\begin{tabular}{lrrrrrrrrrrrrr}
\toprule
measure & olmo2 & q1.5B & g3.1FL & haiku & gpt5na & g3.5FL & sonnet* & fable* & g3.6F & opus* & g3.8F & gpt5sol & \textsf{all} \\
\midrule
$r^0$ & $-0.28$ & $-0.26$ & $-0.28$ & $-0.24$ & $-0.08$ & $-0.19$ & $-0.04$ & $-0.05$ & $+0.01$ & $-0.04$ & $+0.01$ & $+0.05$ & $-0.16$ \\
$t^0$ & $-0.29$ & $-0.33$ & $-0.48$ & $-0.56$ & $-0.59$ & $-0.55$ & $-0.49$ & $-0.60$ & $-0.56$ & $-0.56$ & $-0.56$ & $-0.57$ & $-0.60$ \\
$r^0 + t^0$ & $-0.36$ & $-0.39$ & $-0.59$ & $-0.63$ & $-0.59$ & $-0.62$ & $-0.51$ & $-0.59$ & $-0.55$ & $-0.56$ & $-0.55$ & $-0.55$ & $-0.65$ \\
$\ell_{\mathrm{char}}(p^0)$ & $-0.46$ & $-0.50$ & $-0.58$ & $-0.63$ & $-0.66$ & $-0.63$ & $-0.54$ & $-0.57$ & $-0.58$ & $-0.56$ & $-0.60$ & $-0.54$ & $-0.68$ \\
$\ell_{\mathrm{char}}(p^{\triangledown})$ & $-0.52$ & $-0.52$ & $-0.59$ & $-0.59$ & $-0.54$ & $-0.59$ & $-0.49$ & $-0.59$ & $-0.54$ & $-0.55$ & $-0.53$ & $-0.53$ & $-0.67$ \\
$\ell_{\mathrm{char}}(p^{\blacktriangledown})$ & $-0.54$ & $-0.53$ & $-0.66$ & $-0.67$ & $-0.58$ & $-0.61$ & $-0.48$ & $-0.56$ & $-0.52$ & $-0.53$ & $-0.53$ & $-0.51$ & $-0.70$ \\
$\smash[t]{\hat{r}}_{\mathcal{M}}$ & $-0.21$ & $-0.23$ & $-0.19$ & $-0.16$ & $-0.01$ & $-0.12$ & $+0.08$ & $+0.09$ & $+0.13$ & $+0.12$ & $+0.15$ & $+0.18$ & $-0.06$ \\
$\smash[t]{\hat{t}}_{\mathcal{M}}$ & $-0.36$ & $-0.40$ & $-0.48$ & $-0.58$ & $-0.57$ & $-0.51$ & $-0.47$ & $-0.54$ & $-0.52$ & $-0.54$ & $-0.53$ & $-0.52$ & $-0.59$ \\
$\smash[t]{\hat{r}_{\mathcal{M}}+\hat{t}_{\mathcal{M}}}$ & $-0.38$ & $-0.45$ & $-0.58$ & $-0.66$ & $-0.58$ & $-0.59$ & $-0.46$ & $-0.51$ & $-0.50$ & $-0.49$ & $-0.47$ & $-0.45$ & $-0.62$ \\
$L_u$ & $-0.40$ & $-0.35$ & $-0.42$ & $-0.50$ & $-0.51$ & $-0.53$ & $-0.40$ & $-0.53$ & $-0.52$ & $-0.48$ & $-0.50$ & $-0.54$ & $-0.56$ \\
$c_\M(x_{1:n})$ & $-0.44$ & $-0.45$ & $-0.59$ & $-0.66$ & $-0.63$ & $-0.61$ & $-0.50$ & $-0.62$ & $-0.58$ & $-0.57$ & $-0.56$ & $-0.55$ & $-0.68$ \\
$\tau_{\mathcal{M}}$ & $-0.34$ & $-0.42$ & $-0.60$ & $-0.68$ & $-0.63$ & $-0.57$ & $-0.46$ & $-0.56$ & $-0.57$ & $-0.52$ & $-0.59$ & $-0.59$ & $-0.65$ \\
$\smash[t]{\hat{K}^{\V}_{\M}}(x_{1:n+k})$ & $-0.43$ & $-0.45$ & $-0.60$ & $-0.66$ & $-0.63$ & $-0.65$ & $-0.54$ & $-0.61$ & $-0.62$ & $-0.57$ & $-0.60$ & $-0.56$ & $-0.69$ \\
$\smash[t]{\hat{K}_{\mathcal{M}}}(x_{n+1:n+k} \mid x_{1:n})$ & $-0.31$ & $-0.45$ & $-0.73$ & $-0.73$ & $-0.73$ & $-0.76$ & $-0.63$ & $-0.73$ & $-0.68$ & $-0.66$ & $-0.69$ & $-0.66$ & $-0.78$ \\
$\smash[t]{\widehat{Kc}_{\mathcal{M}}}(x_{n+1:n+k} \mid x_{1:n})$ & $-0.43$ & $-0.49$ & $-0.70$ & $-0.73$ & $-0.70$ & $-0.71$ & $-0.60$ & $-0.70$ & $-0.66$ & $-0.65$ & $-0.65$ & $-0.63$ & $-0.77$ \\
$\smash[t]{\widehat{Kt}_{\mathcal{M}}}(x_{n+1:n+k} \mid x_{1:n})$ & $-0.37$ & $-0.50$ & $-0.75$ & $-0.75$ & $-0.74$ & $-0.74$ & $-0.63$ & $-0.72$ & $-0.67$ & $-0.66$ & $-0.68$ & $-0.65$ & $-0.79$ \\
\bottomrule
\end{tabular}
\end{table}

\Cref{fig:curvesQwen14B} (left) shows the empirical curves of the points for the 51 certified examples for Qwen3-14B. 
It shows a clean monotone decline of performance versus Levin's $\Ktc$. %
The two local models fall after the second bin, On the right we see the histogram of difficulties ($\Ktc$). The first 100 bits can be explained by the fact that very simple items that have a full rule around 50-100 bits to express the sequence in the unconditional regime collapse when the shortest conditional program is searched for. Many items end up very close to zero, as they are continued by the model without any other information. For these items, $\log_2 \tau$ adds very few bits in any case.

\begin{figure}[h]
\centering
\includegraphics[width=0.49\linewidth]{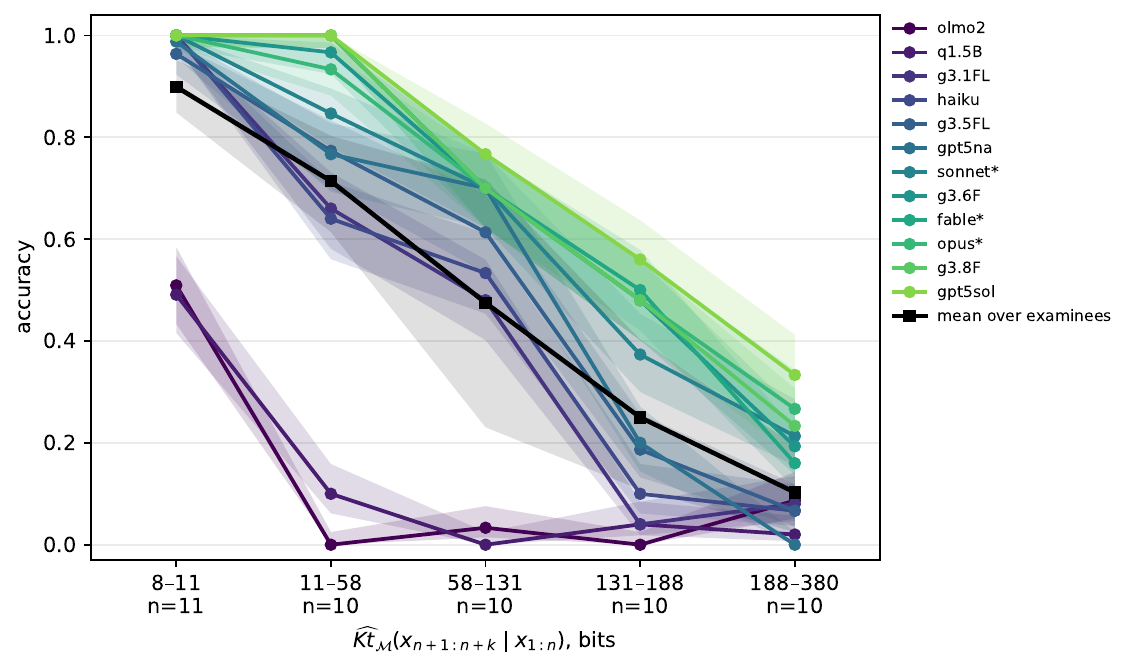}
\includegraphics[width=0.45\linewidth]{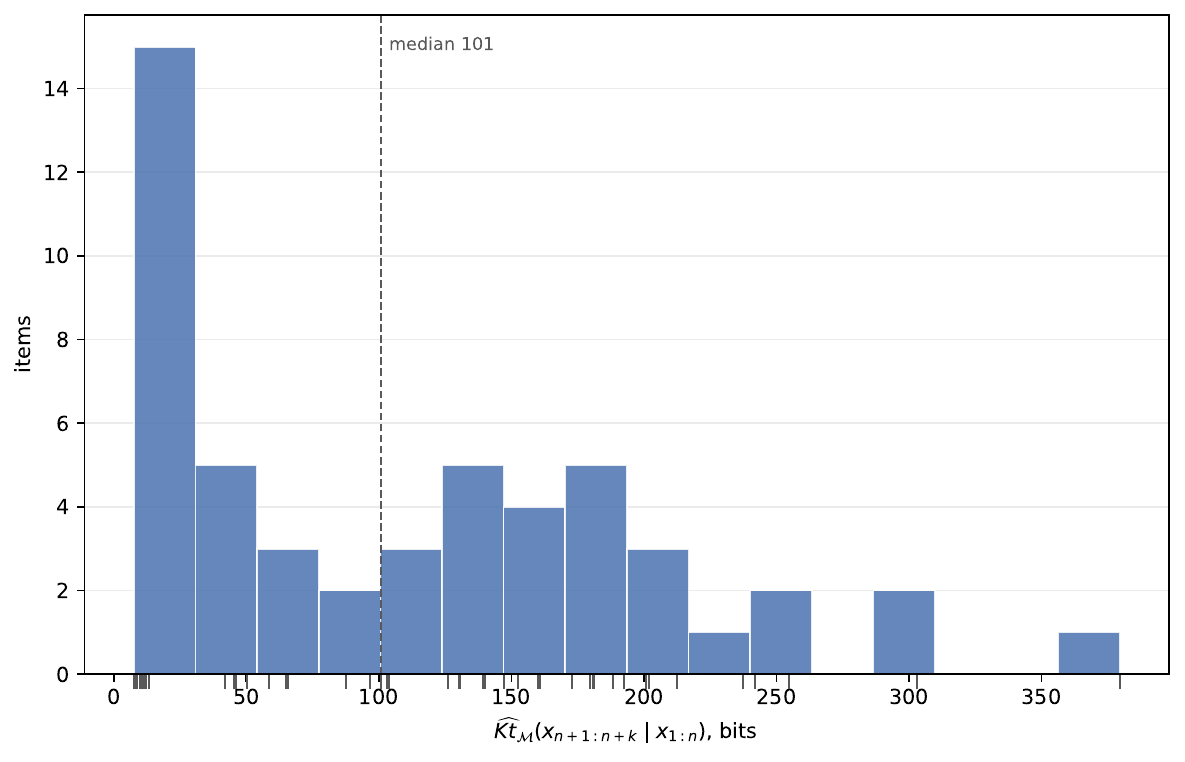}
\caption{Left: Empirical agent characteristic curves (same-size bins) for the 12 LLM, with $\Ktc$ provided by Qwen3-14B (51 certified examples). Right: a histogram of difficulty ($\Ktc$) of the battery.}
\label{fig:curvesQwen14B}
\end{figure}

\section{Results for Qwen3-8B}\label{app:qwen3-8B}

Here we include all the plots and tables for Qwen3-8B specifically. This model is weaker and it has a smaller number of certified items. \Cref{tab:chain-qwen3-8b} shows the distribution grid of its 80 unquestionable items by measured ($\hat{r},\hat{t}$). We see that, again, most items are concentrated in low and medium transformations and low retrieval. 
For 29 of the 107 valid items, Qwen3-8B was costed with an earlier value of
the clause constant, $\delta = 13.3$ bits. Recomputed from the stored margins
with $\delta = 22.7$, two of them become questionable, neither of them
certified, and one certified item, \texttt{0002}, needs a prefix of 13
terms instead of 5. The results in this appendix use the values as costed.

\begin{table}[t]
\centering
\small
\caption{Number of  \emph{valid} / \emph{unquestionable} / \emph{certified} items under Qwen3-8B, by the cell they fell in.}
\label{tab:chain-qwen3-8b}
\begin{tabular}{@{}lrrrrrrrrrr@{}}
\toprule
$\hat{r}$ & $\hat{t}=0$ & $\hat{t}=1$ & $\hat{t}=2$ & $\hat{t}=3$ & $\hat{t}=4$ & $\hat{t}=5$ & $\hat{t}=6$ & $\hat{t}=7$ & $\hat{t}=8$ & all \\
\midrule
0 & 2/2/2 & 5/3/3 & 9/5/3 & 7/2/1 & 7/5/3 & 3/3/0 & 3/2/0 & 4/1/0 & -- & 40/23/12 \\
1 & 6/6/4 & 9/7/5 & 6/4/2 & 8/6/4 & 4/2/0 & 6/4/0 & 1/1/0 & -- & 1/1/0 & 41/31/15 \\
2 & 3/3/2 & 3/3/1 & 3/3/0 & 1/1/0 & -- & 1/1/0 & -- & -- & 2/1/0 & 13/12/3 \\
3 & -- & 4/4/1 & 3/3/1 & 1/1/0 & 2/1/0 & -- & -- & 1/0/0 & -- & 11/9/2 \\
4 & -- & -- & 1/0/0 & 1/1/0 & -- & -- & -- & -- & -- & 2/1/0 \\
\midrule
all & 11/11/8 & 21/17/10 & 22/15/6 & 18/11/5 & 13/8/3 & 10/8/0 & 4/3/0 & 5/1/0 & 3/2/0 & 107/76/32 \\
\bottomrule
\end{tabular}
\end{table}

For Qwen3-8B specifically, \Cref{fig:heatQwen8B} shows the percentage of items being solved by each examinee for each cell in the grid. 
Again, high $t$ with low $r$ seems to be a specific area where models tend to fail.

\begin{figure}[h]
\centering
\includegraphics[width=1.0\linewidth]{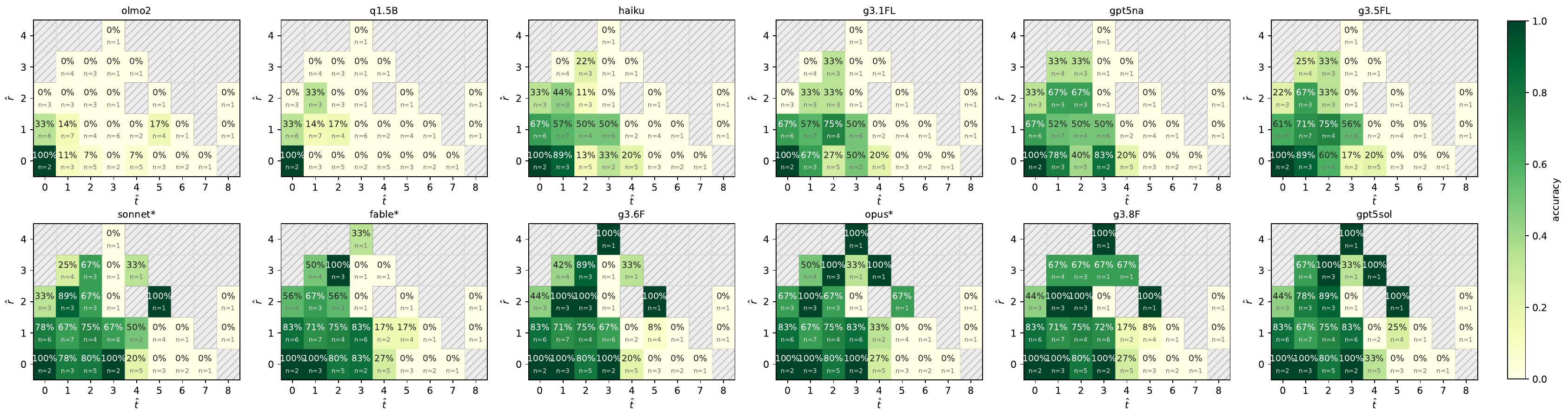}
\vspace{-0.5cm}
\caption{The 12 LLM examinees on the ($\hat{r},\hat{t}$) grid and accuracy for the unquestionable examples for Qwen3-8B.}
\label{fig:heatQwen8B}
\end{figure}

\Cref{fig:abilitiesQwen8B} shows the logistic fit  (left) of the points for the 32 certified examples for Qwen3-8B ($2N$ prior points added at (0,1) as anchor). The right plot compares the ability of each model with Item Response Theory estimates of the abilities of the same models according to Epoch's Capability Index (ECI) \citep{ho2025rosetta}.

\begin{figure}[h]
\centering
\includegraphics[width=0.48\linewidth]{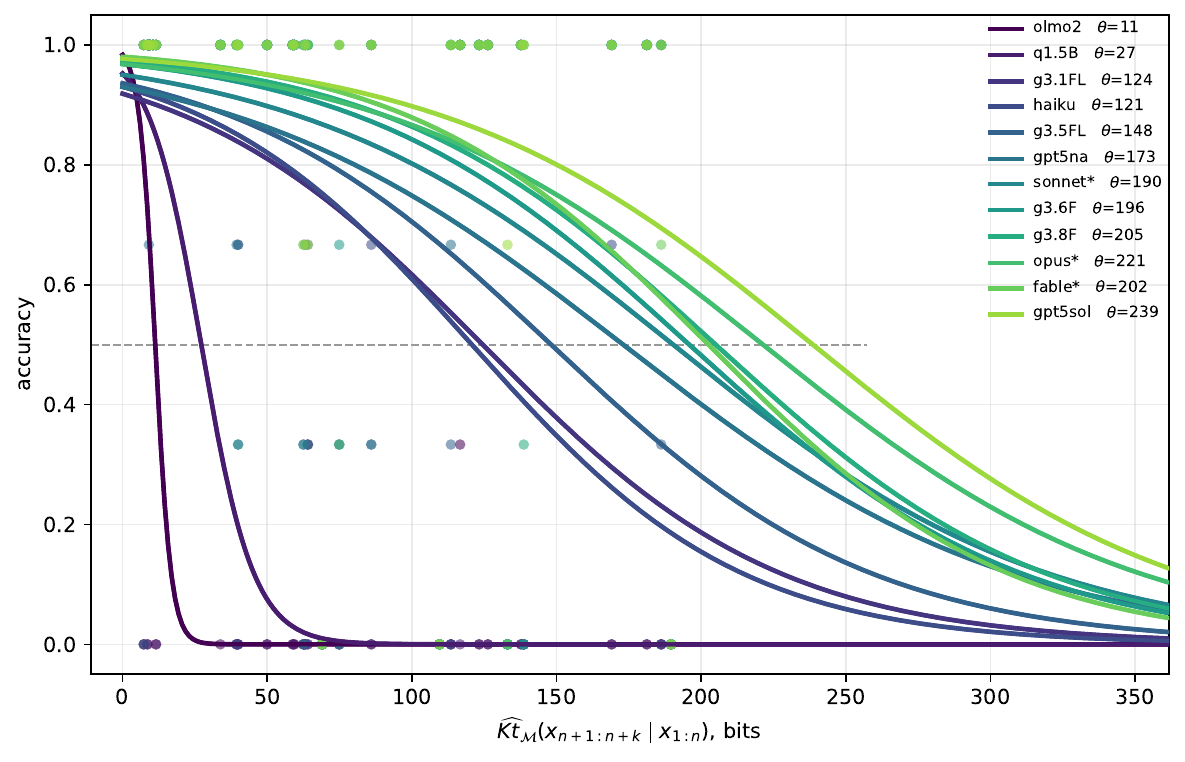}
\includegraphics[width=0.48\linewidth]{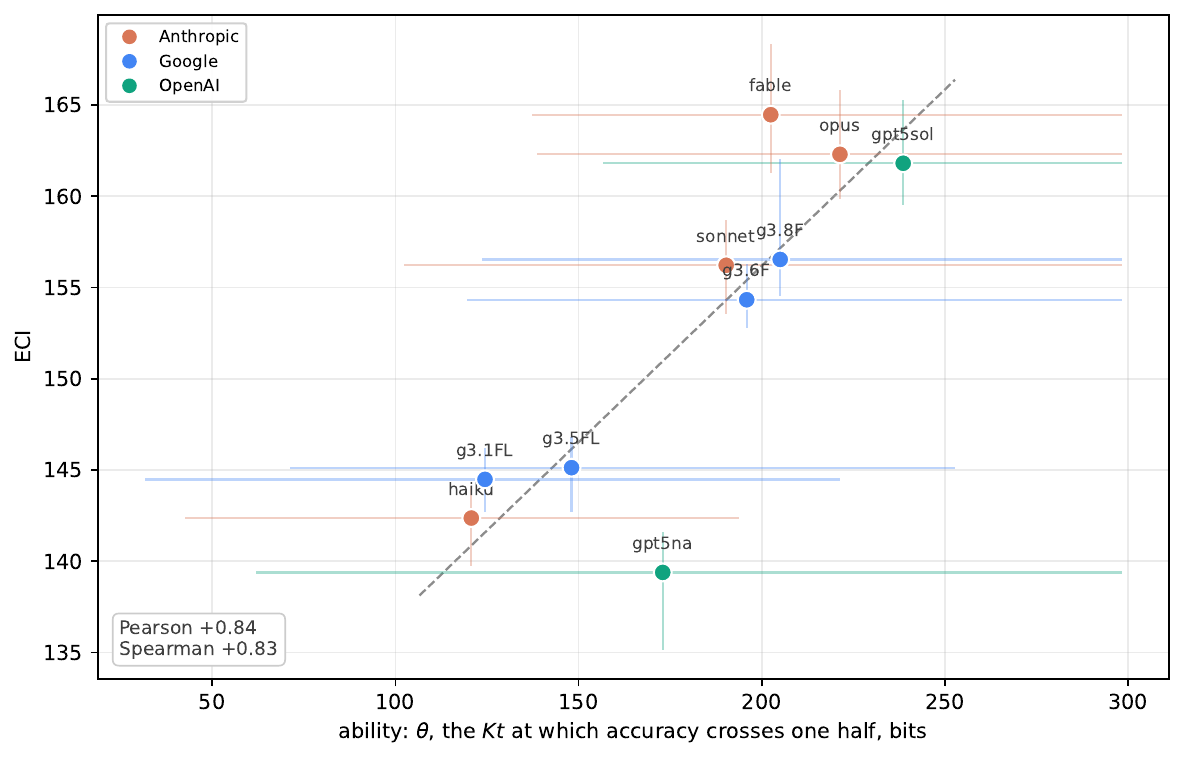}
\caption{Left: Agent characteristic curves for the 12 LLM, with $\Ktc$ provided by Qwen3-8B (32 certified examples), with a logistic fit, where the ability of each model $\theta$, measured in bits, is estimated by the point where performance is 0.5. Right: Comparison between each model's {\em ability} and the psychometric, populational score given by Epoch's Capability Index (ECI) \citep{ho2025rosetta}. Only 10 of our 12 are listed in ECI.}
\label{fig:abilitiesQwen8B}
\end{figure}

\Cref{tab:cert-metricsQwen8b} shows the correlations between the metrics and the scores of the 12 LLMs (plus all the scores together) for Qwen3-8B. The first thing we observe is that the magnitude of the correlations is smaller than those with Qwen3-14B, especially for the most powerful models (it is $-0.27$ for gpt5sol). This is clearly explained by the lower $\Ktc$ estimated  by Qwen3-8B, which is very visible in the histogram of the right of  \Cref{fig:curvesQwen8B}. One thing to highlight is that $\Kcc$ achieves better results than $\Ktc$ at the lower end, possibly because may items are certified with an empty program, and $\tau$ cannot distinguish between them, while the $c(x_{1:n})$ component of $\Ktc$ can give more refinement at this lower end. But in this reference machine, the simple character lengths of the seed and other programs collapse, especially for the frontier models.

\begin{table}[t]
\centering
\scriptsize\setlength{\tabcolsep}{2.0pt}
\caption{Spearman correlations between several metrics and full correct response for the 32 valid, unquestionable and certified items under Qwen3-8B. Each item had three repetitions per examinee.}
\label{tab:cert-metricsQwen8b}
\begin{tabular}{lrrrrrrrrrrrrr}
\toprule
measure & olmo2 & q1.5B & haiku & g3.1FL & gpt5na & g3.5FL & sonnet* & fable* & g3.6F & opus* & g3.8F & gpt5sol & \textsf{all} \\
\midrule
$r^0$ & $-0.32$ & $-0.13$ & $-0.02$ & $-0.07$ & $+0.23$ & $-0.03$ & $+0.09$ & $+0.04$ & $+0.10$ & $+0.01$ & $+0.22$ & $+0.17$ & $-0.09$ \\
$t^0$ & $-0.25$ & $-0.23$ & $-0.24$ & $-0.13$ & $-0.05$ & $-0.07$ & $+0.06$ & $-0.09$ & $-0.10$ & $+0.12$ & $-0.11$ & $+0.09$ & $-0.18$ \\
$r^0 + t^0$ & $-0.38$ & $-0.26$ & $-0.31$ & $-0.24$ & $-0.06$ & $-0.19$ & $-0.05$ & $-0.13$ & $-0.18$ & $-0.02$ & $-0.11$ & $+0.06$ & $-0.30$ \\
$\ell_{\mathrm{char}}(p^0)$ & $-0.61$ & $-0.45$ & $-0.27$ & $-0.28$ & $-0.14$ & $-0.22$ & $-0.24$ & $-0.25$ & $-0.31$ & $-0.15$ & $-0.26$ & $-0.03$ & $-0.41$ \\
$\ell_{\mathrm{char}}(p^{\triangledown})$ & $-0.65$ & $-0.54$ & $-0.17$ & $-0.22$ & $+0.03$ & $-0.20$ & $-0.16$ & $-0.25$ & $-0.28$ & $-0.16$ & $-0.19$ & $-0.01$ & $-0.41$ \\
$\ell_{\mathrm{char}}(p^{\blacktriangledown})$ & $-0.67$ & $-0.45$ & $-0.26$ & $-0.34$ & $-0.01$ & $-0.20$ & $-0.20$ & $-0.30$ & $-0.26$ & $-0.18$ & $-0.19$ & $-0.01$ & $-0.43$ \\
$\smash[t]{\hat{r}}_{\mathcal{M}}$ & $-0.22$ & $+0.04$ & $+0.13$ & $+0.06$ & $+0.13$ & $+0.04$ & $-0.03$ & $+0.20$ & $-0.05$ & $-0.05$ & $+0.08$ & $+0.14$ & $+0.02$ \\
$\smash[t]{\hat{t}}_{\mathcal{M}}$ & $-0.47$ & $-0.47$ & $-0.45$ & $-0.28$ & $-0.31$ & $-0.24$ & $-0.20$ & $-0.31$ & $-0.28$ & $-0.14$ & $-0.31$ & $-0.19$ & $-0.45$ \\
$\smash[t]{\hat{r}_{\mathcal{M}}+\hat{t}_{\mathcal{M}}}$ & $-0.55$ & $-0.42$ & $-0.43$ & $-0.32$ & $-0.26$ & $-0.28$ & $-0.29$ & $-0.23$ & $-0.38$ & $-0.24$ & $-0.29$ & $-0.10$ & $-0.47$ \\
$L_u$ & $-0.45$ & $-0.41$ & $-0.25$ & $-0.16$ & $-0.15$ & $-0.22$ & $-0.29$ & $-0.30$ & $-0.30$ & $-0.17$ & $-0.24$ & $-0.02$ & $-0.35$ \\
$c_\M(x_{1:n})$ & $-0.54$ & $-0.46$ & $-0.40$ & $-0.36$ & $-0.26$ & $-0.28$ & $-0.30$ & $-0.32$ & $-0.38$ & $-0.21$ & $-0.35$ & $-0.18$ & $-0.50$ \\
$\tau_{\mathcal{M}}$ & $-0.33$ & $-0.33$ & $-0.40$ & $-0.39$ & $-0.20$ & $-0.25$ & $-0.21$ & $-0.39$ & $-0.35$ & $-0.18$ & $-0.35$ & $-0.20$ & $-0.43$ \\
$\smash[t]{\hat{K}^{\V}_{\M}}(x_{1:n+k})$ & $-0.51$ & $-0.43$ & $-0.34$ & $-0.25$ & $-0.17$ & $-0.21$ & $-0.22$ & $-0.28$ & $-0.34$ & $-0.16$ & $-0.29$ & $-0.11$ & $-0.42$ \\
$\smash[t]{\hat{K}_{\mathcal{M}}}(x_{n+1:n+k} \mid x_{1:n})$ & $-0.48$ & $-0.59$ & $-0.47$ & $-0.37$ & $-0.29$ & $-0.38$ & $-0.34$ & $-0.42$ & $-0.40$ & $-0.31$ & $-0.39$ & $-0.28$ & $-0.55$ \\
$\smash[t]{\widehat{Kc}_{\mathcal{M}}}(x_{n+1:n+k} \mid x_{1:n})$ & $-0.53$ & $-0.60$ & $-0.47$ & $-0.36$ & $-0.29$ & $-0.36$ & $-0.34$ & $-0.38$ & $-0.42$ & $-0.28$ & $-0.39$ & $-0.24$ & $-0.56$ \\
$\smash[t]{\widehat{Kt}_{\mathcal{M}}}(x_{n+1:n+k} \mid x_{1:n})$ & $-0.45$ & $-0.56$ & $-0.46$ & $-0.36$ & $-0.28$ & $-0.34$ & $-0.33$ & $-0.42$ & $-0.40$ & $-0.29$ & $-0.39$ & $-0.27$ & $-0.53$ \\
\bottomrule
\end{tabular}
\end{table}

\Cref{fig:sample-qwen3-8b} shows how the $\Ktc$ metric changes between two reference machines of different family. Not only are many items not certified by Qwen3-8B, as compared to Phi-4-14B-reasoning, but their difficulty is higher. Qwen3-8B is a worse compressor, but mostly because many equivalent programs and hints that reduce $\Ktc$ in Phi-4-14B-reasoning do not yield the sequence for Qwen3-8B. This adds to  our answer to RQ6 in how the power of the reference machine affects the results.

\begin{figure}[h]
\centering
\includegraphics[width=.8\linewidth]{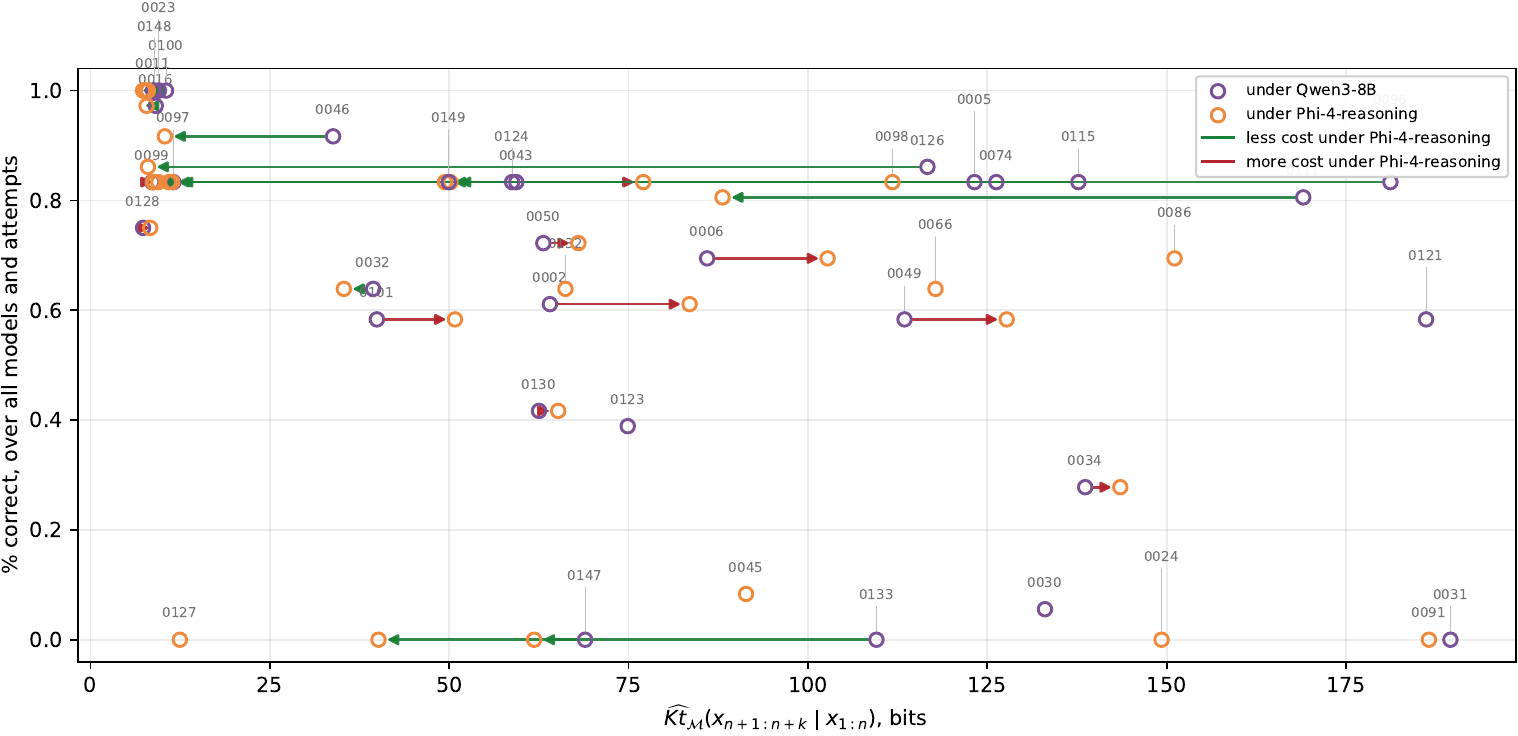}\\
\caption{Comparison of $\Ktc$ between two reference machines (Qwen-3-8B and Phi-4-14B-reasoning) versus the average results of 12 other LLMs on the unquestionable items in the battery. As predicted, success decreases with the theoretical complexity of the item. The weaker model (Qwen-3-8B, purple) brings more refinement in the lower end, but it is not able to cost many items on the upper end. An arrow pointing left means  Phi-4-14B-reasoning describes the item at less cost. Usually  Phi-4-14B-reasoning is a more effective coder and achieves lower values. Values near $Kt=0$ represent cases where the reference machine was able to continue the sequence correctly without any program or hint.}
\label{fig:sample-qwen3-8b}
\end{figure}

\Cref{fig:curvesQwen8B} (left) shows the empirical curves of the points for the 32 certified examples for Qwen3-8B. 
It shows a less clean monotone decline of performance versus Levin's $\Ktc$. %
The two local models fall after the second bin, but the rest have a comeback for the last two quantiles. Very powerful models do not go below 0.8. In reality, we are actually seeing the beginning of a much longer characteristic curve that this model cannot really estimate for frontier models. On the right we see the histogram of difficulties ($\Ktc$), and that explains the figure. Qwen3-8B really does not have sensitivity for large difficulty and powerful models.

\begin{figure}[h]
\centering
\includegraphics[width=0.49\linewidth]{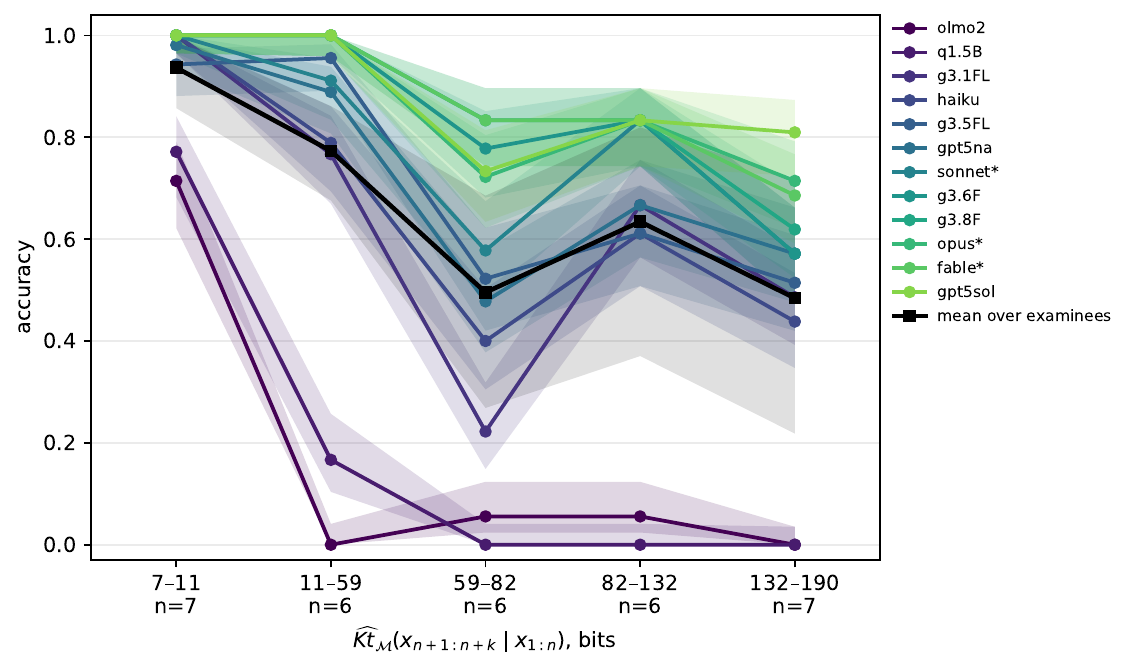}
\includegraphics[width=0.45\linewidth]{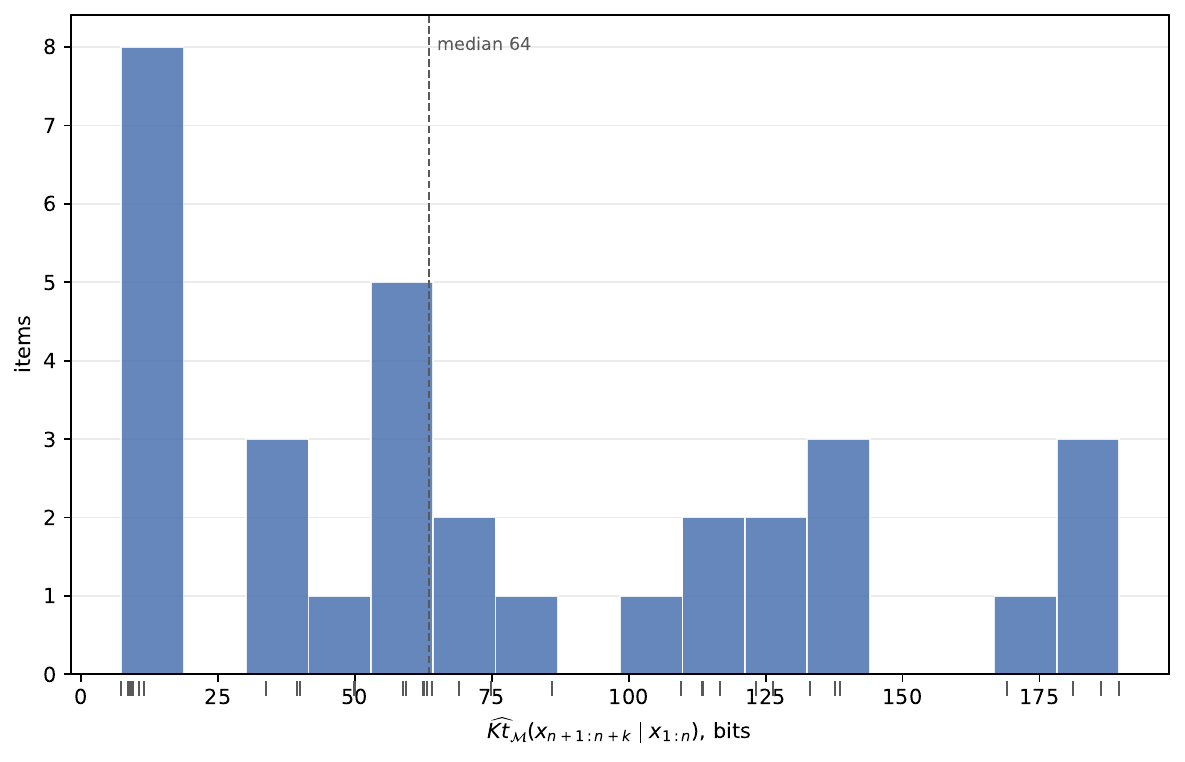}
\caption{Left: Empirical agent characteristic curves (same-size bins)  for the 12 LLM, with $\Ktc$ provided by Qwen3-8B (32 certified examples). Right: a histogram of difficulty ($\Ktc$) of the battery.}
\label{fig:curvesQwen8B}
\end{figure}

\section{Results for Phi-4-14B-reasoning}\label{app:phi-4}

Here we include all the plots and tables for Phi-4-14B-reasoning specifically. This model situates itself in between the previous two models. \Cref{tab:chain-phi-4-reasoning} shows the distribution grid of its 80 unquestionable items by measured ($\hat{r},\hat{t}$). We see, once again, that most items are concentrated in low and medium transformations and low retrieval. 

\begin{table}[t]
\centering
\small
\caption{Number of  \emph{valid} / \emph{unquestionable} / \emph{certified} items  under Phi-4-14B-reasoning, by the cell they fell in.}
\label{tab:chain-phi-4-reasoning}
\begin{tabular}{@{}lrrrrrrrrrrr@{}}
\toprule
$\hat{r}$ & $\hat{t}=0$ & $\hat{t}=1$ & $\hat{t}=2$ & $\hat{t}=3$ & $\hat{t}=4$ & $\hat{t}=5$ & $\hat{t}=6$ & $\hat{t}=7$ & $\hat{t}=8$ & $\hat{t}=9$ & all \\
\midrule
0 & 2/2/2 & 6/4/4 & 5/3/2 & 8/3/2 & 7/6/2 & 2/2/0 & 4/4/0 & 3/1/1 & -- & -- & 37/25/13 \\
1 & 7/7/5 & 9/7/7 & 8/4/2 & 4/4/2 & 8/5/2 & 3/2/0 & 4/3/1 & 1/1/0 & -- & -- & 44/33/19 \\
2 & 2/2/1 & 4/4/2 & 2/2/1 & 2/2/0 & -- & 1/1/0 & -- & -- & 1/0/0 & 1/1/0 & 13/12/4 \\
3 & -- & 5/5/1 & 2/2/1 & 1/1/0 & 2/1/0 & -- & -- & 1/0/0 & -- & -- & 11/9/2 \\
4 & -- & -- & 1/0/0 & -- & 1/1/0 & -- & -- & -- & -- & -- & 2/1/0 \\
\midrule
all & 11/11/8 & 24/20/14 & 18/11/6 & 15/10/4 & 18/13/4 & 6/5/0 & 8/7/1 & 5/2/1 & 1/0/0 & 1/1/0 & 107/80/38 \\
\bottomrule
\end{tabular}
\end{table}

Following with Phi-4-14B-reasoning, \Cref{fig:heatPhi14B} shows the percentage of items being solved by each examinee for each cell in the grid. 
Similarly to the same figure for the other models, high $t$ with low $r$ seems to be a specific area where models tend to fail.

\begin{figure}[h]
\centering
\includegraphics[width=1.0\linewidth]{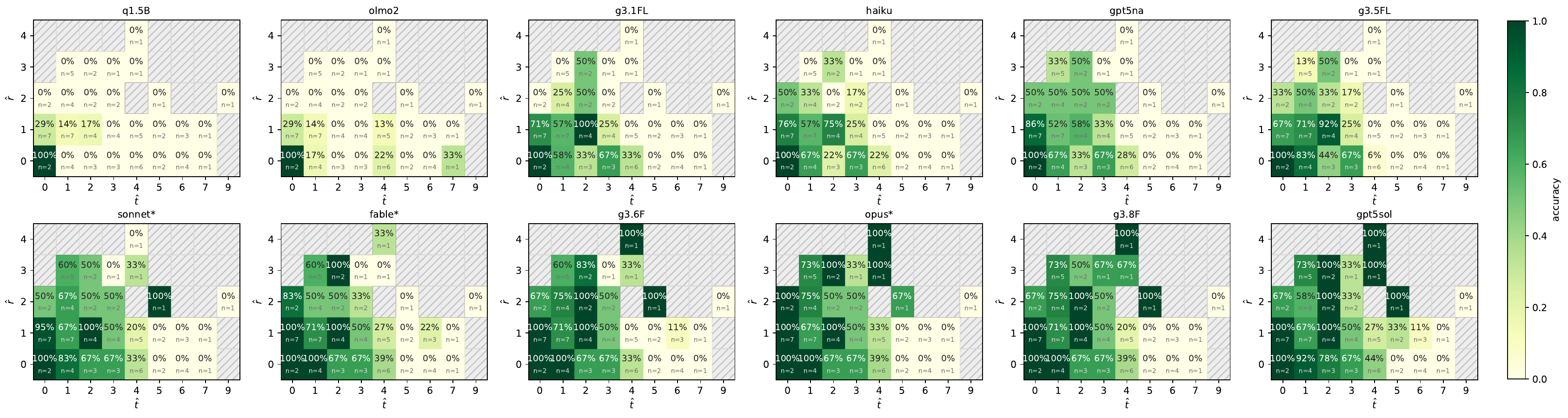}
\vspace{-0.5cm}
\caption{The 12 LLM examinees on ($\hat{r},\hat{t}$) grid and accuracy for the unquestionable examples for Phi-4-reasoning.}
\label{fig:heatPhi14B}
\end{figure}

\Cref{fig:abilitiesPhi-14b} shows the logistic fit  (left)) of the points for the 38 certified examples for Phi-4-14B-reasoning ($2N$ prior points added at (0,1) as anchor). The right plot compares the ability of each model with Item Response Theory estimates of the abilities of the same models according to Epoch's Capability Index (ECI) \citep{ho2025rosetta}. The correlations are very high again.

\begin{figure}[h]
\centering
\includegraphics[width=0.48\linewidth]{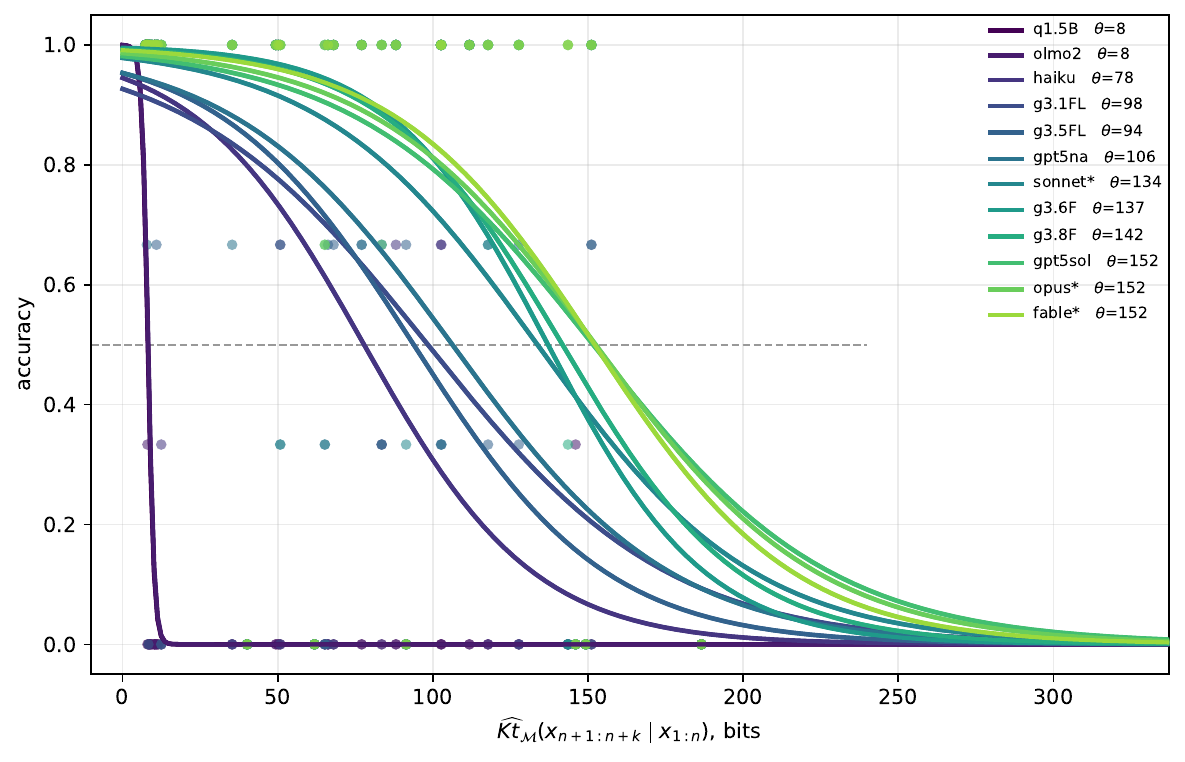}
\includegraphics[width=0.48\linewidth]{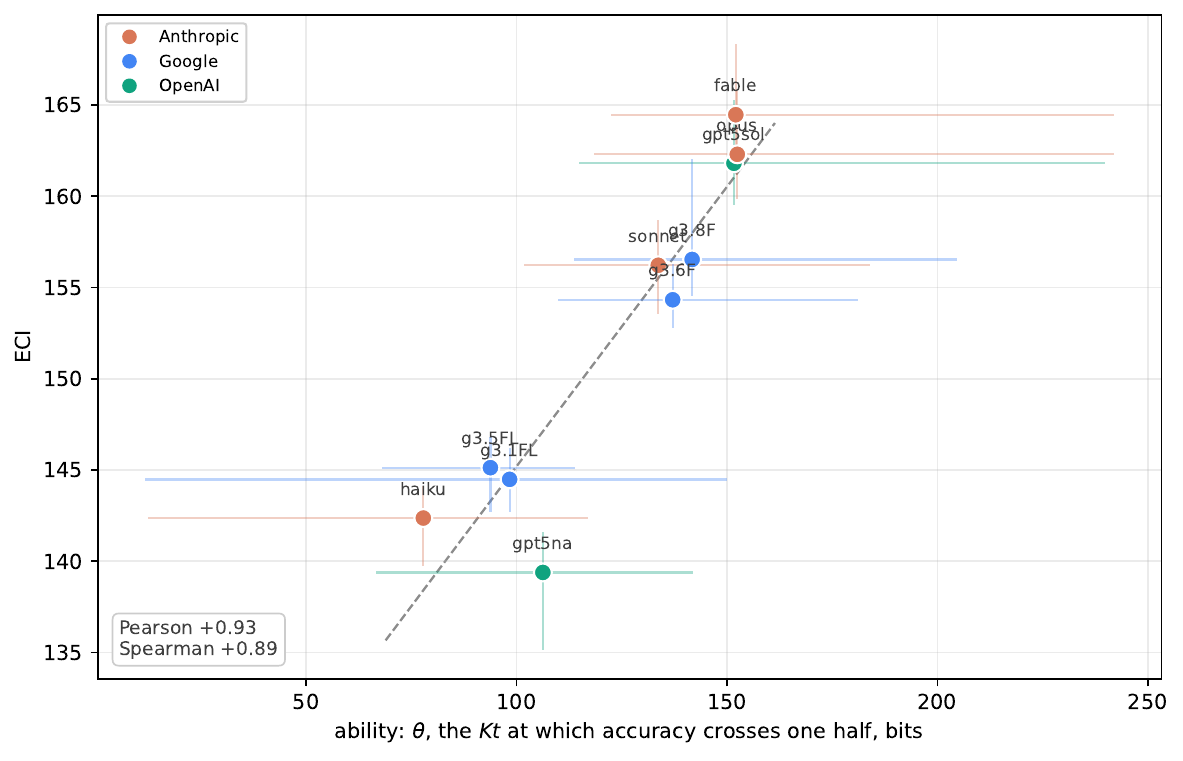}
\caption{Left: Agent characteristic curves for the 12 LLM, with $\Ktc$ provided by Phi-4-14B-reasoning (38 certified examples), with a logistic fit, where the ability of each model $\theta$, measured in bits, is estimated by the point where performance is 0.5. Right: Comparison between each model's {\em ability} and the psychometric, populational score given by Epoch's Capability Index (ECI) \citep{ho2025rosetta}. Only 10 of our 12 are listed in ECI.}
\label{fig:abilitiesPhi-14b}
\end{figure}

\Cref{tab:cert-metricsPhi14b} shows the correlations with the metrics and $\Ktc$ as derived by Phi-4-reasoning. The correlations are a bit in between Qwen3-8B and Qwen3-14B, which is consistent with a range of difficulties that is also in between. For this case, the dominance of $Kt$ over $Kc$ is very noticeable, and no metric really competes with $Kt$.
 In this reference machine, the simple character lengths of the seed and other programs are left about 0.20 points of correlation overall below $\Ktc$, again especially well behind for the frontier models.

\begin{table}[t]
\centering
\scriptsize\setlength{\tabcolsep}{2.0pt}
\caption{Spearman correlations between several metrics and full correct response for the 38 valid, unquestionable and certified items under Phi-4-reasoning. Each item had three repetitions per examinee.}
\label{tab:cert-metricsPhi14b}
\begin{tabular}{lrrrrrrrrrrrrr}
\toprule
measure & q1.5B & olmo2 & g3.1FL & haiku & gpt5na & g3.5FL & sonnet* & fable* & g3.6F & opus* & g3.8F & gpt5sol & \textsf{all} \\
\midrule
$r^0$ & $-0.26$ & $-0.34$ & $-0.09$ & $-0.12$ & $+0.13$ & $-0.21$ & $+0.01$ & $-0.05$ & $+0.03$ & $-0.06$ & $+0.02$ & $+0.00$ & $-0.19$ \\
$t^0$ & $-0.29$ & $-0.10$ & $+0.01$ & $-0.19$ & $-0.23$ & $-0.24$ & $-0.07$ & $-0.10$ & $-0.19$ & $-0.02$ & $-0.18$ & $-0.04$ & $-0.21$ \\
$r^0 + t^0$ & $-0.36$ & $-0.27$ & $-0.14$ & $-0.29$ & $-0.19$ & $-0.37$ & $-0.13$ & $-0.15$ & $-0.23$ & $-0.09$ & $-0.22$ & $-0.08$ & $-0.34$ \\
$\ell_{\mathrm{char}}(p^0)$ & $-0.49$ & $-0.41$ & $-0.14$ & $-0.26$ & $-0.20$ & $-0.33$ & $-0.26$ & $-0.19$ & $-0.28$ & $-0.14$ & $-0.26$ & $-0.09$ & $-0.36$ \\
$\ell_{\mathrm{char}}(p^{\triangledown})$ & $-0.54$ & $-0.60$ & $-0.21$ & $-0.25$ & $-0.09$ & $-0.33$ & $-0.20$ & $-0.27$ & $-0.29$ & $-0.16$ & $-0.29$ & $-0.12$ & $-0.43$ \\
$\ell_{\mathrm{char}}(p^{\blacktriangledown})$ & $-0.53$ & $-0.52$ & $-0.35$ & $-0.39$ & $-0.19$ & $-0.41$ & $-0.31$ & $-0.36$ & $-0.35$ & $-0.28$ & $-0.36$ & $-0.23$ & $-0.53$ \\
$\smash[t]{\hat{r}}_{\mathcal{M}}$ & $-0.10$ & $-0.36$ & $-0.14$ & $-0.09$ & $+0.06$ & $-0.09$ & $-0.11$ & $-0.07$ & $-0.09$ & $-0.09$ & $-0.09$ & $-0.02$ & $-0.16$ \\
$\smash[t]{\hat{t}}_{\mathcal{M}}$ & $-0.48$ & $-0.27$ & $-0.12$ & $-0.39$ & $-0.43$ & $-0.37$ & $-0.31$ & $-0.30$ & $-0.37$ & $-0.25$ & $-0.36$ & $-0.24$ & $-0.42$ \\
$\smash[t]{\hat{r}_{\mathcal{M}}+\hat{t}_{\mathcal{M}}}$ & $-0.46$ & $-0.33$ & $-0.28$ & $-0.48$ & $-0.39$ & $-0.39$ & $-0.37$ & $-0.33$ & $-0.41$ & $-0.30$ & $-0.40$ & $-0.27$ & $-0.49$ \\
$L_u$ & $-0.35$ & $-0.31$ & $+0.03$ & $-0.13$ & $-0.04$ & $-0.32$ & $-0.09$ & $-0.10$ & $-0.14$ & $-0.04$ & $-0.14$ & $+0.01$ & $-0.22$ \\
$c_\M(x_{1:n})$ & $-0.48$ & $-0.37$ & $-0.15$ & $-0.33$ & $-0.18$ & $-0.35$ & $-0.18$ & $-0.16$ & $-0.24$ & $-0.08$ & $-0.23$ & $-0.09$ & $-0.38$ \\
$\tau_{\mathcal{M}}$ & $-0.58$ & $-0.47$ & $-0.38$ & $-0.44$ & $-0.30$ & $-0.48$ & $-0.31$ & $-0.22$ & $-0.27$ & $-0.25$ & $-0.26$ & $-0.30$ & $-0.56$ \\
$\smash[t]{\hat{K}^{\V}_{\M}}(x_{1:n+k})$ & $-0.49$ & $-0.37$ & $-0.10$ & $-0.32$ & $-0.18$ & $-0.33$ & $-0.21$ & $-0.21$ & $-0.29$ & $-0.13$ & $-0.28$ & $-0.14$ & $-0.37$ \\
$\smash[t]{\hat{K}_{\mathcal{M}}}(x_{n+1:n+k} \mid x_{1:n})$ & $-0.42$ & $-0.25$ & $-0.44$ & $-0.68$ & $-0.63$ & $-0.61$ & $-0.50$ & $-0.39$ & $-0.46$ & $-0.39$ & $-0.46$ & $-0.41$ & $-0.66$ \\
$\smash[t]{\widehat{Kc}_{\mathcal{M}}}(x_{n+1:n+k} \mid x_{1:n})$ & $-0.54$ & $-0.37$ & $-0.30$ & $-0.54$ & $-0.41$ & $-0.49$ & $-0.35$ & $-0.27$ & $-0.35$ & $-0.25$ & $-0.34$ & $-0.26$ & $-0.56$ \\
$\smash[t]{\widehat{Kt}_{\mathcal{M}}}(x_{n+1:n+k} \mid x_{1:n})$ & $-0.58$ & $-0.40$ & $-0.46$ & $-0.68$ & $-0.60$ & $-0.64$ & $-0.49$ & $-0.38$ & $-0.45$ & $-0.38$ & $-0.44$ & $-0.40$ & $-0.72$ \\
\bottomrule
\end{tabular}
\end{table}

\Cref{fig:sample-phi-4-14B-reasoning} shows how the $\Ktc$ metric changes between two reference machines of a different family. Not only are a few items not certified by Phi-4-14B-reasoning, as compared to Qwen3-14B, but their difficulty is higher. Phi-4-14B-reasoning is a worse compressor, but in this cases it is not for those items that both certify (arrows move similarly in both direction) but almost ony  because many equivalent programs and hints that reduce $\Ktc$ in Qwen-3-14b do not yield the sequence for Phi-4-14B-reasoning. This completes our answer to RQ6 in how the power of the reference machine affects the results.

\begin{figure}[h]
\centering
\includegraphics[width=.8\linewidth]{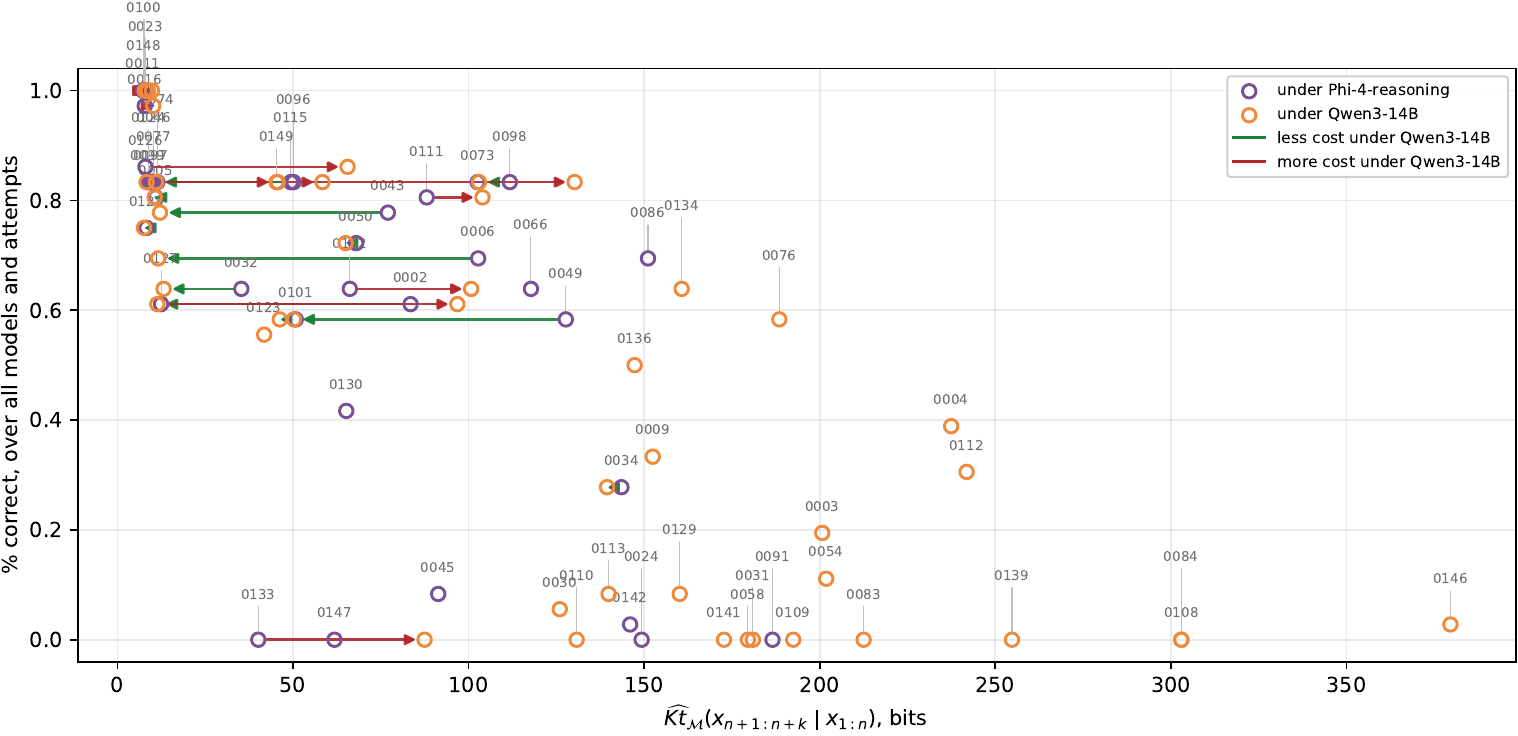}\\
\caption{Comparison of $\Ktc$ between two reference machines (Qwen-3-14B and Phi-4-14B-reasoning) versus the average results of 12 other LLMs on the unquestionable items in the battery. As predicted, success decreases with the theoretical complexity of the item. The weaker model (Phi-4-14B-reasoning, purple) brings more refinement in the lower end, but it is not able to cost many items on the upper end. An arrow pointing left means  Qwen-3-14B describes the item at less cost. Here, we do not see the dominance of  Qwen-3-14B and some arrows go left (in green), but some others go right (in red). Values near $Kt=0$ represent cases where the reference machine was able to continue the sequence correctly without any program or hint.}
\label{fig:sample-phi-4-14B-reasoning}
\end{figure}

Finally, \Cref{fig:curvesPhi14B} (left) shows the empirical curves of the points for the 38 certified examples for Phi-4-14B-reasoning. 
It shows an irregular, yet flat behaviour of performance until the last bin of Levin's $\Ktc$. %
The two local models fall after the second bin, but the rest only fall in the last bin, showing that the critical part for many models happens around difficulty 112. On the right we see the histogram of difficulties ($\Ktc$), and the area around 112 is where we have most spread. It is not only that this reference machine can only reach 187 as maximum $\Ktc$ but also that the spread and number of difficulties make the curves less smooth they should be. 

\begin{figure}[h]
\centering
\includegraphics[width=0.49\linewidth]{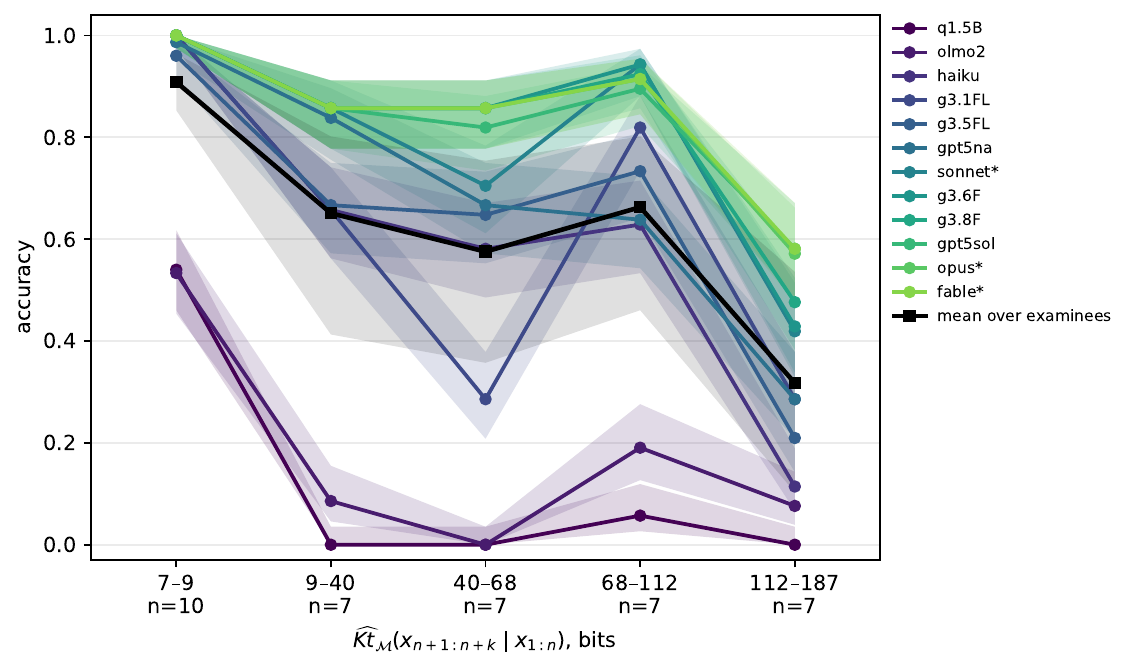}
\includegraphics[width=0.45\linewidth]{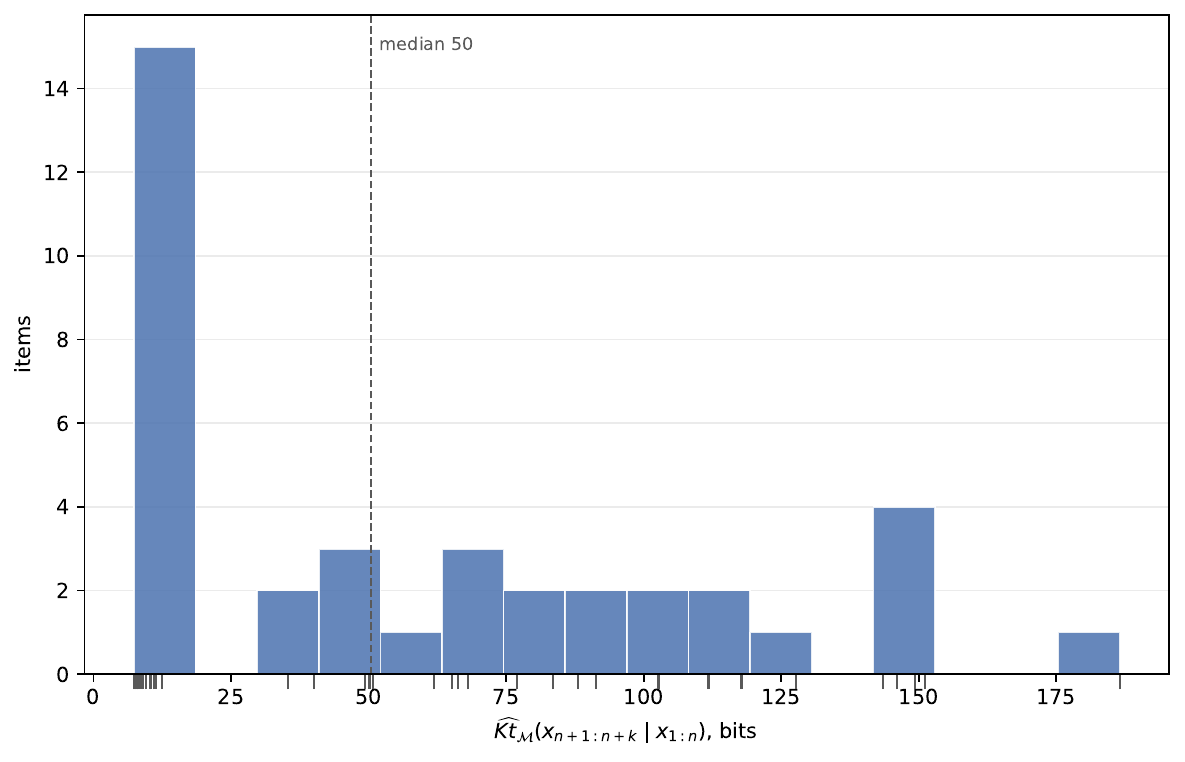}
\caption{Left:  Empirical agent characteristic curves (same-size bins) for the 12 LLM, with $\Ktc$ provided by Phi-4-14B-reasoning (38 certified examples). Left: empirical curve with same-size bins. Right: a histogram of difficulty ($\Ktc$) of the battery.}
\label{fig:curvesPhi14B}
\end{figure}

\section{Ratio Scales}\label{app:ratio}

If $\widehat{Kt}^{|}$ gives a ratio scale, abilities measured with different
reference machines should differ only by a unit, that is, be proportional,
as temperatures in kelvin and rankine are, and not merely related by a line
with an offset, as Celsius and Fahrenheit are. The Pearson correlations
between the abilities of the twelve examinees under the three machines are
0.985 or higher (Table~\ref{tab:stats-ratio}). Fitting a line between each
pair of machines, the intercepts are compatible with zero, at $+7$, $+1$ and
$-3$ bits, with 95\% bootstrap intervals over examinees that include zero,
and forcing the line through the origin leaves the fit essentially
unchanged, with $R^2$ of 0.988, 0.971 and 0.971 against 0.990, 0.971 and
0.972. Abilities under Qwen3-8B are about 1.46 times those under Qwen3-14B,
since the weaker machine describes the same items with more bits, and under
Phi-4-reasoning they are about the same as under Qwen3-14B. Without the two
small examinees, whose abilities are near zero under every machine, the
correlations are 0.94 to 0.98 and the intercepts remain compatible with
zero, although their intervals are wider.

\begin{table}[h]\centering\small
\caption{Pearson correlations between the abilities of the twelve examinees
under the three reference machines. This supports the idea of the metric being a ratio scale, with only the coefficient changing between machines. See \Cref{app:ratio} for justification about the common zero.}
\label{tab:stats-ratio}
\begin{tabular}{@{}lrrr@{}}\toprule
 & Qwen3-14B & Qwen3-8B & Phi-4-reasoning \\ \midrule
Qwen3-14B       & 1.000 & 0.995 & 0.985 \\
Qwen3-8B        & 0.995 & 1.000 & 0.986 \\
Phi-4-reasoning & 0.985 & 0.986 & 1.000 \\
\bottomrule\end{tabular}\end{table}

This is just an empirical check on the possibility of the metric being a ratio scale, which was not set as a research question of the paper. As future work, experimental comparison in terms of training and inference compute (scaling laws) with this metric could be performed, to see if the behaviour between ability and the logarithm of a meaningful combination of these two elements of compute could be additive. 

\section{The Reference Machine as an Examinee}
\label{app:selfref}

The choice of a reference machine $\M$ for $\widehat{Kt}_\M$ tries to capture human language and culture, but still the test ends up biased by the particularities of
$\M$. But how much biased the test is compared to other reference machines? Of course, in our experiments we did not use the three reference machines as examinees, because this would have contaminated the averages across models, and the overall correlations. But it is now time to see how much a reference machine favours itself, and this will indirectly bring some light the question about the particularities of different reference machines. On the key results, there are things we can observe without any extra experiment. Namely, the choice of reference machine barely changes how the examinees are ranked.
  Their abilities under the three machines agree at Spearman $+0.99$, $+0.92$
  and $+0.95$ (Appendix~\ref{app:stats}).

But the effect must appear when applied to itself. 
A trivial question would b whether a machine does well on the items it certifies. Yet the answer to this questions is known in advance. An item enters a
machine's battery only if that machine ran some description of it
successfully, so the battery is selected to be solvable by that machine. This
is a contamination by construction, and it is why the reference machines are
never scored alongside the other examinees in the body of the paper.

Instead, to get more insight on this phenomenon we thought of paired analysis. We can use the items that two machines both certify and
give them to both machines. Every item then has a value of $\widehat{Kt}$
under each machine and a result as an examinee. The question it
answers is whether a machine's own $\widehat{Kt}$ orders its answers better
than the other machine's $\widehat{Kt}$ does. Because the question is
symmetric, it needs both machines.

We chose Qwen3-14B and Phi-4-reasoning because they are the two models with similar number of parameters and the capability is close. Then we select the 28 items that both Qwen3-14B and Phi-4-reasoning certify. These usually have low $\Ktc$. The median $\widehat{\Ktc}$ over them is 13
bits under Qwen3-14B and 24 bits under Phi-4-reasoning. Six of the twelve other examinees answer more than 90\% of them correctly. It is important to understand that this is a special set of items.

Also, to understand if the comparison is fair, we have to explain how the two machines were run. As said above, both machines answered the same 28 items with the same prompt as every other
examinee (Figure~\ref{fig:examprompt}). Each was shown the same prefix, the
longest that any of the three reference machines needed for that item, and
asked for the next $\kf = 5$ terms, as all other examinees. Both were loaded in 4-bit NF4 with
reasoning enabled and used greedy decoding, as were used as reference machines. These are the coincidences. Now the differences. The two runs differed how the end token is identified and the token budget, in the
respects listed in Table~\ref{tab:selfref-conditions}. 
The stopping rule was needed for Phi-4-reasoning. Without it, every reply ran
to the full budget of 8000 tokens, %
because the model
reaches the answer early and then keeps restating it. To reduce the running time, we implemented a rule that stops a reply
as soon as the five terms appear as a run of comma-separated fields on one
line, or once a passage of 200 characters has been repeated five times after
the first 2000 tokens. The smaller budget on 15 items was needed to fit the
second part of the run into the time available. %

\begin{table}[t]\centering\small
\caption{The conditions under which the two reference machines were run as
examinees. The rows that differ are explained in the text.}
\label{tab:selfref-conditions}
\begin{tabular}{@{}l >{\raggedright\arraybackslash}p{4.4cm} >{\raggedright\arraybackslash}p{4.4cm}@{}}\toprule
 & \multicolumn{1}{c}{Qwen3-14B} & \multicolumn{1}{c}{Phi-4-reasoning} \\
\midrule
items, prefix, prompt   & \multicolumn{2}{c}{the same for both} \\
quantisation, reasoning & \multicolumn{2}{c}{4-bit NF4, enabled} \\
decoding                & \multicolumn{2}{c}{greedy} \\
\addlinespace[2pt]
stopping a reply  & at its own end token
                  & also when the five terms appear, or when it repeats
                    itself \\
token budget      & 8000 on all 28 items
                  & 8000 on 13 items, 3000 on 15 \\
\bottomrule\end{tabular}\end{table}

As we ran the two machines with greedy decoding, the result is deterministic, so we stopped running three (unnecessary) repetitions at some point. 
Repetitions would only measure the judge, since the replies themselves do not change. 
But for our exploration, the repetitions were almost always equivalent, so we decided not to use extra compute on this.  

The differences in Table~\ref{tab:selfref-conditions} pull in
opposite directions. The stopping rule is more lenient to Phi-4-reasoning,
since it ends a reply at the point where the answer appears, while the smaller
budget is less lenient. We cannot say which effect is larger.

Under this rule Qwen3-14B solves 21 of the 28 items and Phi-4-reasoning 15.
As a reference machine, Qwen3-14B continued 15 of the 28 with the empty
description, and as an examinee it solves 14 of those; the exception is
\texttt{0005}. Phi-4-reasoning continued 14 of the 28 with the empty
description and solves 10 of those. The prefix is not the reason for the four
it fails, since each was shown at least as many terms as it had needed. %
A likely cause is the prompt. As a reference
machine the model continues a sequence placed after the word \emph{Rule}. As
an examinee, the prompt mentions it may {\em think} before answering, and this model then
reasons at length, exceeding the budget more often than with the prompt that was used for verification in the annotator when used as a reference machine.

Given all these considerations, we can now look at Table~\ref{tab:selfref-items}, which lists the 28 items. For each machine, we show the
kind of description that certified the item (with no help, i.e., `empty', a partial indication, `hint', a full `rule', or the `seed' the composer first wrote), its $\widehat{\Ktc}$, and whether
the machine solves it as an examinee.

\begin{table}[t]\centering\small\setlength{\tabcolsep}{4pt}
\caption{The 28 items certified by both machines, in order of $\widehat{\Ktc}$
under Qwen3-14B. For each machine the table gives the kind of description that
certified the item, its $\widehat{\Ktc}$ in bits, and whether the machine solves
it as an examinee.}
\label{tab:selfref-items}
\begin{tabular}{@{}l lrc c lrc@{}}\toprule
 & \multicolumn{3}{c}{Qwen3-14B} && \multicolumn{3}{c}{Phi-4-reasoning} \\
\cmidrule(lr){2-4}\cmidrule(lr){6-8}
item & described by & $\widehat{\Ktc}$ & solved && described by & $\widehat{\Ktc}$ & solved \\
\midrule
\texttt{0128} & empty & 8 & yes && empty & 8 & no \\
\texttt{0011} & empty & 8 & yes && empty & 7 & yes \\
\texttt{0099} & empty & 8 & yes && empty & 9 & yes \\
\texttt{0148} & empty & 8 & yes && empty & 8 & yes \\
\texttt{0023} & empty & 9 & yes && empty & 8 & yes \\
\texttt{0100} & empty & 10 & yes && empty & 8 & yes \\
\texttt{0016} & empty & 10 & yes && empty & 8 & yes \\
\texttt{0046} & empty & 10 & yes && empty & 10 & yes \\
\texttt{0005} & empty & 11 & no && empty & 11 & yes \\
\texttt{0096} & empty & 11 & yes && hint & 50 & no \\
\texttt{0097} & empty & 11 & yes && empty & 9 & no \\
\texttt{0002} & empty & 11 & yes && hint & 84 & no \\
\texttt{0006} & empty & 12 & yes && hint & 103 & yes \\
\texttt{0043} & empty & 12 & yes && hint & 77 & yes \\
\texttt{0032} & empty & 13 & yes && hint & 35 & no \\
\texttt{0124} & rule & 46 & yes && empty & 9 & yes \\
\texttt{0101} & hint & 46 & yes && hint & 51 & no \\
\texttt{0049} & hint & 50 & no && rule & 128 & no \\
\texttt{0074} & hint & 58 & yes && empty & 11 & no \\
\texttt{0050} & hint & 65 & yes && seed & 68 & yes \\
\texttt{0126} & hint & 66 & yes && empty & 8 & yes \\
\texttt{0133} & hint & 87 & no && hint & 40 & no \\
\texttt{0127} & rule & 97 & no && empty & 12 & no \\
\texttt{0132} & rule & 101 & no && hint & 66 & no \\
\texttt{0098} & rule & 103 & yes && rule & 112 & yes \\
\texttt{0111} & hint & 104 & yes && hint & 88 & no \\
\texttt{0115} & rule & 130 & no && hint & 49 & yes \\
\texttt{0034} & rule & 139 & no && rule & 144 & no \\
\bottomrule\end{tabular}\end{table}

Now we are ready to answer the question of how much the reference machine biases its own evaluation. 
For each examinee we compute the Spearman correlation, over the 28 items,
between its outcome on each item and $\widehat{\Ktc}$, once with Qwen3-14B's
values and once with Phi-4-reasoning's. For the two reference machines the
outcome is whether the item was solved (1) or not (0). For the twelve other examinees, which we also include, whose replies do vary between attempts, it is the proportion of their three attempts answered correctly. 
A more negative correlation means that measure orders the items
better for that examinee, in the sense that the items it calls more expensive
are the ones the examinee fails. Ties are given their average rank.

\begin{table}[t]\centering\small
\caption{Spearman correlation between outcome and $\widehat{\Ktc}$ over the 28
items, with each machine's values of $\widehat{\Ktc}$. A more negative value
means the measure orders the items better for that examinee.}
\label{tab:selfref-corr}
\begin{tabular}{@{}lrrl@{}}\toprule
examinee & Qwen3-14B's $\widehat{Kt}$ $\:\:$ & Phi-4-reasoning's $\widehat{Kt}$
 & ordered better by \\ \midrule
Qwen3-14B               & $-0.51$ & $-0.33$ & its own measure, correlation difference = 0.18 \\
Phi-4-reasoning         & $-0.33$ & $-0.40$ & its own measure, correlation difference = 0.07 \\
\midrule
\addlinespace[2pt]
The 12 other LLMs, mean & $-0.35$ & $-0.33$ & neither, correlation difference = 0.02  \\
\bottomrule\end{tabular}\end{table}

Table~\ref{tab:selfref-corr} gives the result. For the twelve examinees that
are not reference machines, shown in the bottom line, the two measures are almost equally good, at
$-0.35$ and $-0.33$. %
But, for each reference machine, its own measure is the better one. The
margin is 0.18 for Qwen3-14B and 0.07 for Phi-4-reasoning.

We checked how much this depends on particular items. Restricted to the 25
items that Qwen3-8B also certifies, the pattern holds and Phi-4-reasoning's
margin widens. Qwen3-14B correlates $-0.46$ with its own measure and $-0.37$
with the other, and Phi-4-reasoning $-0.48$ with its own and $-0.34$ with the
other. Leaving out each of the 28 items in turn, Qwen3-14B's own measure is
the better one every time, by at least 0.10. Phi-4-reasoning's own measure is
also the better one every time, but by very little when \texttt{0126} is
left out.

So we can conclude what we expected: a reference machine is favoured by its own battery, but we can quantify it and understand the reasons better. The first reason is
the selection described above. A machine's battery holds only items that the
machine could run the given solution (in many cases with the empty program). The second reason is what we have shown. On items that both machines certify, a
machine's own $\widehat{Kt}$ orders its answers better than the other
machine's $\widehat{Kt}$ does. The effect appears for both machines, in the
expected direction, while the twelve other examinees find the two
measures almost equally good. It is clear for Qwen3-14B and small for
Phi-4-reasoning.

The evidence is not conclusive about the magnitude because the 28 items are among the least expensive in
the battery, so the range of $\widehat{Kt}$ over them is narrow. Also, as said, the two
machines were run under conditions that differ in the ways listed in
Table~\ref{tab:selfref-conditions}. And there is one reference machine on each
side, so the result rests on two comparisons.

But we stay reassured about our decision that the reference machine should not be scored as an examinee on a battery it helped
to build, both because of the selection bias performed by certification and because its own measure fits it
better than any other machine.

\section{Statistical Tests}
\label{app:stats}

We report intervals and tests for the main quantitative results. Intervals
are at 95\%. For correlations over items we use a percentile bootstrap
that resamples items, 4000 times, because the uncertainty comes from the composition of the battery. 
For the interval of the Pearson correlation with ECI we use the Fisher
$z$-transformation.
It covers the ten examinees that ECI
lists, since OLMo-2-7B and Qwen2.5-1.5B are not in it. For a proportion we
use the Clopper--Pearson interval.
All $p$-values come from one-sided
permutation tests in the predicted direction, with 10{,}000 permutations over
items and 20{,}000 over examinees\footnote{We implemented the tests directly, but we then used SciPy's functions \citep{virtanen2020scipy} and replicated the results, with very minor differences: {\tt spearmanr}, {\tt bootstrap} with the percentile method, {\tt permutation\_test}, and {\tt binomtest} with the exact interval, which is Clopper–Pearson. }.

\begin{table}[h]\centering\small
\caption{Intervals and tests for the main results. The first block corresponds
to the last column of Tables \ref{tab:cert-metricsQwen14b-full}, \ref{tab:cert-metricsQwen8b} and \ref{tab:cert-metricsPhi14b}, the second to \Cref{fig:abilitiesQwen14B} (right),
the third to the rankings in Appendix \ref{app:selfref}, and the last to the
supra window in \Cref{tab:windows}.}
\label{tab:stats}
\begin{tabular}{@{}lrrr@{}}\toprule
 & estimate & 95\% interval & $p$ \\ \midrule
\multicolumn{4}{@{}l}{\textit{Spearman correlation of $\widehat{Kt}^{|}$ with accuracy, over items}}\\
Qwen3-14B, 51 items        & $-0.79$ & $[-0.87, -0.67]$ & $<0.0001$ \\
Qwen3-8B, 32 items         & $-0.53$ & $[-0.75, -0.19]$ & $0.001$ \\
Phi-4-reasoning, 38 items  & $-0.72$ & $[-0.88, -0.48]$ & $<0.0001$ \\
\addlinespace[2pt]
\multicolumn{4}{@{}l}{\textit{Ability against ECI under Qwen3-14B, 10 examinees (subset in ECI)}}\\
Pearson                    & $+0.87$ & $[+0.53, +0.97]$ & $0.002$ \\
Spearman                   & $+0.76$ &                  & $0.007$ \\
\addlinespace[2pt]
\multicolumn{4}{@{}l}{\textit{Agreement of the ability rankings, Spearman, 12 examinees}}\\
Qwen3-14B and Qwen3-8B        & $+0.99$ & & $<0.001$ \\
Qwen3-14B and Phi-4-reasoning & $+0.92$ & & $<0.001$ \\
Qwen3-8B and Phi-4-reasoning  & $+0.95$ & & $<0.001$ \\
\addlinespace[2pt]
\multicolumn{4}{@{}l}{\textit{gpt5sol on the 10 items above 190 bits under Qwen3-14B}}\\
items solved on a majority of attempts & 2 of 10 & $[3\%, 56\%]$ & \\
\bottomrule\end{tabular}\end{table}

The correlation between $\widehat{Kt}^{|}$ and accuracy is significant for
all three reference machines. It is also significant at the 5\% level
for each of the twelve examinees taken separately, with values from $-0.37$
to $-0.75$ under Qwen3-14B. Ability correlates with ECI under all three
machines, with Pearson coefficients between $+0.84$ and $+0.93$ and $p <
0.01$ in each case, and the three machines rank the examinees in nearly the
same order.

Gemini 3.6 Flash is also the composer, verifier, typer and judge, which could
favour it as an examinee in several ways. Unlike the reference machines, which were set apart and never included as proper examinees (see Appendix \ref{app:selfref}), we kept Gemini 3.6 Flash in the pool because we wanted to observe the effect, and also because the extreme decision would have been to exclude all models of the same family. Instead we analyse its effect here. 
As judge, it might mark its own answers
leniently. Comparing the judge's verdicts with a mechanical comparison of the
five answer terms against the intended ones, over the 153 attempts per
examinee on the items certified under Qwen3-14B, the judge accepted one
answer from Gemini 3.6 Flash that the mechanical check rejects, against
between zero and three for each of the other examinees, and it never rejected
a correct answer. As generator, it might share habits with the sequences it
wrote. Against a line fitted between ability and ECI, Gemini 3.6 Flash lies
12 bits above, but GPT-5.6-sol and Gemini 3.8 Flash, which take no part in
the pipeline, lie 11 and 9 bits above, and the residuals of the ten examinees
range from $-26$ to $+21$ bits. Gemini 3.8 Flash also scores above it, at 143
bits against 140. Any advantage is therefore small, and within the spread of
the other examinees.

In the 
comparison of measures in Tables \ref{tab:cert-metricsQwen14b-full}, \ref{tab:cert-metricsQwen8b} and \ref{tab:cert-metricsPhi14b}, $\widehat{Kt}^{|}$ has the
highest correlation with accuracy. Under Qwen3-14B, a paired bootstrap over items does not
separate it from the other complexity measures. Its difference from the
unconditional $\widehat{K}^{\V}_{\M}$, $-0.79$ against $-0.69$, has an
interval of $[-0.23, +0.01]$, and its difference from $c(x_{1:n})$ one of
$[-0.26, +0.02]$. From the conditional $\widehat{K}^{|}$ and
$\widehat{Kc}$ it differs by only $0.01$ and $0.02$. The one significant
difference is with $\log_2\tau$ alone, $-0.79$ against $-0.65$, a difference
of $-0.14$ with interval $[-0.29, -0.01]$. This is natural, as many of these metrics express very similar things, and shows that the logarithm of the inference steps does not have an important effect yet, at this level of complexity, but this may change as more complex items and billions of reasoning tokens are used in the future.

One might expect the time term $\log_2\tau$ to matter more at the upper end,
where the computations are longer. But here is where the other term in $\Ktc$ is higher too. 
Under Qwen3-14B, 
adding it to
$\widehat{K}^{|}$ changes the correlation with accuracy by at most 0.02 in
the upper half of the items, and by up to 0.08 in the lower half, where it
is the only term that orders the items continued with no description. %
The logarithm grows slowly, so that even a
million tokens would add only 20 bits, and the largest annotation budget,
12000 tokens, caps it at about 13.6 bits. %
Using $\Kco$ rather than $\Ktc$ would be perfectly justified too, especially for this battery, and at these levels of inference tokens. It was not our goal to exclude other metrics, especially those based on Kolmogorov complexity or those that emphasise the benefits of costing using the reference machine. This is what we are advocating for. The research question was about checking whether $\Ktc$ works and it does. Also, we think that in agentic scenarios, the reasoning tokens compared to the size of the hints may grow considerably, and inference time may become much more relevant. Removing it in this first test and formulation, because it just adds little correlation, would be against the theoretical justification (such as \citep{levin1973,achille2026universal}) and would remove a potential key signal in the future.

The results about the supra window should also be taken with care. Its lower
end in \Cref{tab:windows} is where the fitted curve crosses one third, and the window
above it holds few items, 10 above 190 bits under Qwen3-14B. With this number, an
exact interval cannot exclude a rate above one third unless no item is
solved. Doubling the size of the battery would likely give much narrower intervals. Still, on these 10 items gpt5sol answers 27\% of its attempts correctly and
solves 2 on a majority of attempts, with an interval of $[3\%, 56\%]$. But this is the most powerful model. The
other five strongest examinees solve between 2 of 14 and 4 of 11 items in
their own windows. Seven of the ten items above 190 bits are solved on a
majority of attempts by none of these six. For the four weaker API examinees,
whose windows hold 20 to 27 items, the intervals do exclude a rate above one
third, with upper ends between 22\% and 32\%. So the position and length of
each window are uncertain, but a weak model such as Qwen3-14B can certify
items that the strongest examinees mostly fail.

The logistic fits add $2N$ pseudo-observations at $(0,1)$, twice as many as
there are items, which anchors each curve at full accuracy for an item of
zero bits. We do this as we would expect all models to succeed at items of difficulty close to zero and the results already point in that direction. In any case, the results for abilities or infra/supra windows are not sensitive to this choice. Refitting under
Qwen3-14B with $0$, $N/2$, $N$ or $4N$ such points leaves the order of the
twelve examinees unchanged (Spearman $+1.00$ against the fit with $2N$, and
$+0.99$ with $4N$), moves each ability by at most 6 bits, and keeps the
Pearson correlation with ECI between $+0.87$ and $+0.88$. Without any prior,
the supra windows of the six strongest examinees start 2 to 8 bits later, so
that of gpt5sol runs from 197 bits instead of 190.

\subsection{How many rivals are enough}
\label{app:stats-rivals}

The unquestionability check composes six rivals at each rung of the ladder.
Here we measure how close that number is to saturation, that is, how much
the verdicts would change with fewer.

A \emph{rung} is a prefix length $L$ in the ladder
$\Lambda$ at which rivals are composed. A \emph{rival} at $L$ is a
description $q$ that reproduces the prefix $x_{1:L}$ and departs from the
sequence at term $L+1$, as verified by $\mathcal{V}$. There are two kinds. A
\emph{composed rival} is written by the composer, which is asked for six at
each rung (a few fail verification, so some rungs have fewer). A \emph{structural rival} is not composed but constructed. It states the
$L$ terms of the prefix and then departs from them, in one of three ways, by
repeating the prefix from its first term, by repeating its last term, or by
continuing with a constant. The structural rivals exist at every rung. A
rival \emph{blocks} the item at $L$ if it costs no more than the least
expensive description of the whole sequence plus the tolerance,
$c_{\mathcal{M}}(q) \le \widehat{K} + \delta$. The item is unquestionable at
$L$ if no rival blocks there, and $L_u$ is the first rung at which this
holds.

For each reference machine we took the 101 valid items
that have composed rivals at some rung, and recomputed $L_u$ using only $k$
of the composed rivals at each rung, keeping all the structural rivals, for
$k = 1, \dots, 5$. Fewer rivals can only remove blockers, so the only
possible change is that a rung before $L_u$ is no longer blocked, which
makes $L_u$ shorter or makes a questionable item unquestionable. We report
two measures. The first is the number of items for which at least one choice
of $k$ rivals would change $L_u$, a worst case. The second is the expected
number of items whose $L_u$ changes when the $k$ rivals at each rung are
drawn at random from those composed. It is computed exactly and not by simulation. At a rung with $n$ composed
rivals, of which $b$ block, and no blocking structural rival, the probability
that $k$ random rivals include no blocker is
$\binom{n-b}{k} / \binom{n}{k}$. Being an expectation, this
second measure need not be a whole number.

\begin{table}[h]\centering\small
\caption{How the unquestionability verdicts depend on the number of composed
rivals, over the 101 valid items with composed rivals at some rung. The last two blocks give the number of items whose $L_u$ would change with
only $k$ of the composed rivals, for some choice of the $k$ (worst case) and
in expectation over a random choice.}
\label{tab:stats-rivals}
\begin{tabular}{@{}lrrr@{}}\toprule
 & Qwen3-14B & Qwen3-8B & Phi-4-reasoning \\ \midrule
rungs with composed rivals                     & 240 & 240 & 240 \\
rungs where a composed rival blocks            &  46 &  47 &  44 \\
\quad of which no structural rival blocks      &   1 &   3 &   1 \\
share of composed rivals that block there      & 65\% & 57\% & 67\% \\
median number of composed rivals that block    &   4 &   3 &   4 \\
\addlinespace[2pt]
\multicolumn{4}{@{}l}{\textit{items whose $L_u$ changes for some choice of $k$ rivals}}\\
$k = 1$ & 1 & 3 & 1 \\
$k = 3$ & 1 & 3 & 1 \\
$k = 4$ & 0 & 2 & 0 \\
$k = 5$ & 0 & 1 & 0 \\
\addlinespace[2pt]
\multicolumn{4}{@{}l}{\textit{expected number of items whose $L_u$ changes, random $k$ rivals}}\\
$k = 1$ & 0.6 & 2.1 & 0.6 \\
$k = 3$ & 0.1 & 0.8 & 0.1 \\
$k = 4$ & 0.0 & 0.4 & 0.0 \\
$k = 5$ & 0.0 & 0.2 & 0.0 \\
\bottomrule\end{tabular}\end{table}

Table~\ref{tab:stats-rivals} shows that six composed
rivals are close to saturation. Under Qwen3-14B and Phi-4-reasoning, no
choice of four of the six would change any verdict, and under Qwen3-8B at
most two items would change with four and one with five. Two facts explain
this. First, the structural rivals do most of the work. Whenever a composed
rival blocks, a structural rival blocks as well, except at one to three rungs
per machine. Second, where composed rivals block they overlap, with
a median of three or four of the six blocking, so that dropping some of them
rarely leaves a rung unblocked. More rivals of the same kind are therefore
unlikely to change any verdict.

This does not mean that the procedure finds every rival, but the level of compression is what matter most. We could also manually change $\delta$, if would like to reduce further reduced false positives.

\section{Worked Examples}\label{app:examples}

We presented the pipeline succinctly in the paper (\Cref{sec:pipeline}). The next section describes it in full detail. In between, it is very helpful to see how a few examples evolved during the pipeline, with actual descriptions and costs, as a trace.

\subsection{Item 0123: \texorpdfstring{$\widehat{\Ktc}=41.8$}{Kt=41.8}}
\label{app:example}

We follow our first item, \texttt{0123}, through every stage under Qwen3-14B, and
then summarise what happened to it under the other two reference machines. 

\paragraph{Generation and validation.}

The generator asked for the cell $(r_0, t_0) = (3, 0)$, three retrievals and
no transformation, with the theme \emph{michelin star tier levels} from the
domain \emph{cuisine}. The composer returned the {\em seed}
\begin{quote}\ttfamily
alternating between michelin star tiers, basic taste modalities, and days of
the week in order
\end{quote}
and the sequence it produces, whose first terms are
\begin{quote}\ttfamily
1 star, sweet, monday, 2 stars, sour, tuesday, 3 stars, salty, wednesday,
1 star, bitter, thursday, 2 stars, umami, friday, 3 stars, sweet, saturday,
\ldots
\end{quote}
The three tables have periods 3, 5 and 7, so the sequence repeats only after
105 groups. The item was generated with $N = 70$ terms. The verifier ran the
seed cold times and reproduced all 70 terms each time, so the item is
valid.

The composer then wrote nine equivalents, of which seven reproduced the
sequence on verification. They range from the {\em bare} version 
\begin{quote}\ttfamily
Alternate items from Michelin star tiers, basic tastes, and lowercase days of
the week.
\end{quote}
to pseudocode such as ``\texttt{A, B, C = Michelin star tiers, basic tastes,
lowercase days of the week; term 3k+1 is A[k\%3], 3k+2 is B[k\%5], 3k+3 is
C[k\%7]}''. It also wrote 21 rivals at the ladder $\Lambda = \{5, 13, 47,
60\}$. At $L = 5$ they depart at term 6, for example by running the days
backwards (giving \emph{sunday}) or by using the months (giving
\emph{march}). At $L = 13$ the prefix ends just before the fifth taste, and all
six rivals depart at exactly that term. One uses the four classic Western
tastes and gives \emph{sweet}, one the six Ayurvedic tastes and gives
\emph{pungent}, and the others give \emph{savory}, \emph{spicy},
\emph{fatty} and \emph{astringent}.

\paragraph{Unquestionability under Qwen3-14B.}

The least expensive equivalent is the bare statement above, at
$\widehat{K}^{\V}_{\M}(x) = 131.5$ bits, and so it is $p^\blacktriangledown$.
With the clause constant $\delta = 20.6$ bits, a rival blocks the item at $L$
if it costs at most 152.1 bits.

At $L = 5$ the item is blocked by the three structural rivals. Repeating the
five terms shown costs 138.9 bits, repeating the last term costs 148.2 and
continuing with a constant costs 135.7, all inside the bar. At $L = 13$ the
same three rivals cost between 165.0 and 168.7 bits, and the simplest rival of
any kind is the constant one at 165.0. None of the six composed rivals comes
closer. The item is therefore unquestionable at $L_u = 13$, with 13.0 bits to
spare.

\paragraph{Measured \texorpdfstring{$\hat r$}{r} and \texorpdfstring{$\hat t$}{t}.}

The typer read $p^\blacktriangledown$ five times and split it identically each
time into three retrievals (\emph{Michelin star tiers}, \emph{basic tastes}
and \emph{days of the week}), one operation (\emph{lowercase}) and control
(\emph{alternate items from} and the joining punctuation). This gives
$\hat r = 3$ and $\hat t = 1$ with no spread, against the $(3, 0)$ that was
asked for. The one operation is the conversion to lowercase, which the seed
left implicit.

\paragraph{Certification under Qwen3-14B.}

The annotator showed Qwen3-14B the first 13 terms and asked for the next 20.
Table~\ref{tab:example-annot} lists the candidates it ran. Each is costed
conditionally on the prefix, so a hint only has to supply what the 13 terms do
not already show. The prefix shows the three star tiers and four days in full,
but only four of the tastes.

\begin{table}[h]\centering\small
\caption{The candidates run for \texttt{0123} under Qwen3-14B, with their
conditional cost $c_\M(p \mid x_{1:13})$, the tokens used, and how many of the
20 asked terms came out correct and in order.}
\label{tab:example-annot}
\begin{tabular}{@{}lrrrl@{}}\toprule
description & bits & tokens & terms & outcome \\ \midrule
(empty)              & 0.0  & 5820 & 0  & failed, after its budget was raised to 6000 \\
basic tastes         & 30.8 & 3723 & 0  & failed \\
5 basic tastes       & 31.6 & 1184 & 20 & certified, within the first budget \\
The 5 basic tastes   & 39.5 & 3136 & 20 & certified, after its budget was raised \\
\bottomrule\end{tabular}\end{table}

The empty description failed, and so did \emph{basic tastes}, and both
failed in the same way. Each time Qwen3-14B continued with \emph{sweet,
friday, 3 stars, sour, saturday}, getting the stars and the days right and
cycling back to \emph{sweet} after four tastes, since four is all the prefix
shows. \emph{5 basic tastes} costs less than one bit more and succeeded at
once, because the number fixes a fifth taste, umami. It is the winning conditional description, with $\widehat{K}^{|} = 31.6$ bits
and $\tau = 1184$ tokens.

Two less expensive hints, \emph{days of the week} at 26.3 bits and \emph{general
knowledge} at 29.9, were costed but not run, given the limited number of attempts of the algorithm. Neither says anything about the
tastes, so neither is likely to have certified, but the case shows why the
value reported for $\widehat{K}^{|}$ is an upper bound. The unconditional
equivalents, from 119.4 bits upwards, were not run either.

\paragraph{The item's measures.}

Table~\ref{tab:example-metrics} gives the measures for the item. The
difference between the unconditional and conditional costs, 99.9 bits, is the
part of the description that the prefix makes unnecessary. Its
$\widehat{Kt}^{|}$ of 41.8 bits places it well below the median of the 51
certified items, which is 101 bits.

\begin{table}[h]\centering\small
\caption{The measures for \texttt{0123} under Qwen3-14B, in bits.}
\label{tab:example-metrics}
\begin{tabular}{@{}lr@{}}\toprule
$\widehat{K}^{\V}_{\M}(x)$, the whole sequence, unconditional & 131.5 \\
$c_\M(x_{1:13})$, the prefix itself                            & 97.7 \\
$\widehat{K}^{|}$, the winning hint given the prefix           & 31.6 \\
$\log_2 \tau$, with $\tau = 1184$ tokens                       & 10.2 \\
$\widehat{Kt}^{|} = \widehat{K}^{|} + \log_2 \tau$             & 41.8 \\
$\widehat{Kc}^{|} = \widehat{K}^{|} + c_\M(x_{1:13})$          & 129.3 \\
\bottomrule\end{tabular}\end{table}

\paragraph{The examinees.}

Each examinee saw the same 13 terms and was asked for the next five,
\emph{umami, friday, 3 stars, sweet, saturday}, three times.
Table~\ref{tab:example-exam} gives the results under full marking. Fourteen of
the 36 attempts are correct, and four examinees answer correctly on a majority
of their attempts.

\begin{table}[h]\centering\small
\caption{Attempts answered correctly on \texttt{0123}, out of three, with a
typical wrong answer where there is one.}
\label{tab:example-exam}
\begin{tabular}{@{}lcl@{}}\toprule
examinee & correct & a wrong answer \\ \midrule
fable*, gpt5sol, g3.8F & 3 of 3 & \\
g3.6F                  & 2 of 3 & sweet, friday, 3 stars, sour, saturday \\
opus*, gpt5na          & 1 of 3 & sweet, friday, 3 stars, sour, saturday \\
g3.5FL                 & 1 of 3 & 3 stars, sweet, friday, 1 star, sour \\
sonnet*                & 0 of 3 & sweet, friday, 3 stars, sour, saturday \\
haiku                  & 0 of 3 & salty, friday, 3 stars, bitter, saturday \\
g3.1FL                 & 0 of 3 & 3 stars, umami, friday, 1 star, sweet \\
olmo2                  & 0 of 3 & sweet, sour, salty, bitter, two stars \\
q1.5B                  & 0 of 3 & sour, Wednesday, salty, Thursday, Friday \\
\bottomrule\end{tabular}\end{table}

The most common wrong answer, \emph{sweet, friday, 3 stars, sour, saturday},
keeps the stars and the days right and cycles four tastes, returning to
\emph{sweet} where \emph{umami} belongs. It is the continuation of the
four-taste rival that the composer wrote at $L = 13$, and the one Qwen3-14B
itself gave when it had no hint or a hint without the number. Gemini 3.5
Flash-Lite and Gemini 3.1 Flash-Lite sometimes answer out of step, starting
from the wrong position in the cycle. The item therefore
separates the examinees on one piece of knowledge, whether there is a fifth
basic taste, which is the piece its hint supplies.

\paragraph{Under the other two reference machines.}

Qwen3-8B also finds the item unquestionable at $L_u = 13$, with
$\widehat{K}^{\V}_{\M}(x) = 127.6$ bits. Its cheapest rival at 13 costs 47.0
bits more than that, against a $\delta$ of 22.7, which leaves 24.3 bits of
slack. It certified the item as well, but with a longer hint,
\emph{Interleave 1-3 stars, five basic tastes, and days of the week}, at 63.6
bits and 2528 tokens, which gives $\widehat{Kt}^{|} = 74.9$ bits. Like
Qwen3-14B, it failed with the empty description and with \emph{Michelin star
tiers} alone.

Phi-4-reasoning needed a longer prefix. At $L = 13$ its structural rivals cost
149.5 and 151.9 bits, inside its bar of 156.6, so the item was blocked. By
$L = 26$ the structural rivals were outside the bar, but no rivals had been
composed at 26, so the next length that could be tested was 47, where the
item became unquestionable. Certification then failed. The seed and one
equivalent ran to the limit of 1500 tokens without producing a correct term,
and a second equivalent and three hints stopped within the limit, also with
none correct. Only the listing of the continuation succeeded. That is a fallback and not a
certification, so the item is not in Phi-4-reasoning's battery.

\subsection{Item 0049: \texorpdfstring{$\widehat{\Ktc}= 50.3$}{Kt= 50.3}}

Item \texttt{0123} is almost pure retrieval. Here we follow
\texttt{0049}, which has one retrieval and three transformations, again
under Qwen3-14B first and then under the other two machines. It too certified
with a short hint.

\paragraph{Generation and validation.}

The generator asked for the cell $(r_0, t_0) = (1, 3)$ with the theme
\emph{months of the year} from the domain \emph{chronology}. The composer
returned the seed
\begin{quote}\ttfamily
from n=1; take the nth month name in order, reverse it, keep the first 3
letters, and append n
\end{quote}
which gives
\begin{quote}\ttfamily
yra1, yra2, hcr3, lir4, yam5, enu6, ylu7, tsu8, reb9, reb10, reb11, reb12,
yra13, yra14, hcr15, \ldots
\end{quote}
Keeping three letters hides part of each name. January and February both give
\emph{yra}, and September to December all give \emph{reb}, so the letters alone
do not identify the month and only the index keeps the terms distinct. The
item was generated with $N = 71$ terms, and the verifier reproduced all of
them.

The composer wrote nine equivalents and all nine reproduced the sequence, from
\emph{Reverses each month's name, takes the first 3 letters, and appends the
term index n.}\ to pseudocode such as \texttt{a\_n = rev(lower(nth English
month))[1..3] + n}. It also wrote 16 rivals at the ladder $\Lambda = \{5, 13,
23\}$. The five rivals at $L = 5$ all depart at term 6, June, and each keeps
the rule and adds an exception. One uses the whole reversed name for months
of fewer than five letters and gives \emph{enuj6}. Another leaves months ending in \emph{e} unreversed and gives
\emph{jun6}. A third shifts every letter forward when $n$
is a multiple of 6 and gives \emph{fov6}. The rivals at $L = 13$ change which
month is taken after the first year, for example the $(n-12)^2$-th, which gives
\emph{lir14}.

\paragraph{Unquestionability under Qwen3-14B.}

The least expensive description of the whole sequence is the first equivalent
above, at $\widehat{K}^{\V}_{\M}(x) = 128.3$ bits. The bare statement costs
150.1 bits, so here it is not $p^\blacktriangledown$. With $\delta = 20.6$ bits,
a rival blocks the item if it costs at most 148.9 bits.

The item is unquestionable at the first length tried, $L_u = 5$. A structural
rival has to state the five terms, and strings such as \emph{yra1, yra2, hcr3,
lir4, yam5} are expensive to state. Repeating them costs 162.9 bits, repeating
the last one 162.2 and continuing with a constant 159.5, all outside the bar,
and each composed rival costs more again, because its exception has to be
stated as well. The slack is 10.5 bits.

\paragraph{Measured \texorpdfstring{$\hat r$}{r} and \texorpdfstring{$\hat t$}{t}.}

The typer read $p^\blacktriangledown$ five times and gave the same answer each
time. It found one retrieval, the months, and three operations, reversing the
name, keeping the first three letters and appending the index. This gives
$\hat r = 1$ and $\hat t = 3$ with no spread, exactly the cell that was asked
for.

\paragraph{Certification under Qwen3-14B.}

The annotator showed Qwen3-14B the first five terms and asked for the next 20.
Table~\ref{tab:example2-annot} lists the candidates it ran.

\begin{table}[h]\centering\small
\caption{The candidates run for \texttt{0049} under Qwen3-14B, with their
conditional cost $c_\M(p \mid x_{1:5})$, the tokens used, and how many of the 20
asked terms came out correct and in order.}
\label{tab:example2-annot}
\begin{tabular}{@{}lrrrl@{}}\toprule
description & bits & tokens & terms & outcome \\ \midrule
(empty)                      & 0.0  & 3000 & 0  & failed \\
nth month                    & 38.4 & 6000 & 4  & failed, after its budget was raised to 6000 \\
Month names reversed         & 39.5 & 1792 & 20 & certified, within the first budget \\
English month names reversed & 44.2 & 2496 & 20 & certified, within the first budget \\
\bottomrule\end{tabular}\end{table}

With no description, Qwen3-14B did not find the months at all. It treated the
letters as numbers and looked for differences between them, as in \emph{l
$\to$ y (+13), i $\to$ a ($-8$)}. Told \emph{nth month}, it recognised the
months and worked out that the letters come from the end of each name, but
made mistakes and got only four of the 20 terms right, even with 6000 tokens.
\emph{Month names reversed} costs one bit more and succeeded within the first
budget, because it names the reversal, the step the model could not find by
itself. The prefix already shows the truncation to three letters and the
index, so the hint does not need to mention them. It is the winning
conditional description, with $\widehat{K}^{|} = 39.5$ bits and $\tau = 1792$
tokens.

Four cheaper hints were costed but not run. Two of them, \emph{months} at 23.2
bits and \emph{months of the year} at 26.4, say no more than \emph{nth month},
which failed. The other two, \emph{Months spelled backwards} at 37.2 bits and
\emph{Months reversed} at 38.0, say almost exactly what the winning hint says
and would quite possibly have certified. This is a concrete case of the upper
bound discussed in Appendix~\ref{app:disagree}. The value of $\widehat{Kt}^{|}$ reported here is
likely to be a few bits too high.

\paragraph{The item's measures.}

\begin{table}[h]\centering\small
\caption{The measures for \texttt{0049} under Qwen3-14B, in bits.}
\label{tab:example2-metrics}
\begin{tabular}{@{}lr@{}}\toprule
$\widehat{K}^{\V}_{\M}(x)$, the whole sequence, unconditional & 128.3 \\
$c_\M(x_{1:5})$, the prefix itself                             & 100.6 \\
$\widehat{K}^{|}$, the winning hint given the prefix           & 39.5 \\
$\log_2 \tau$, with $\tau = 1792$ tokens                       & 10.8 \\
$\widehat{Kt}^{|} = \widehat{K}^{|} + \log_2 \tau$             & 50.3 \\
$\widehat{Kc}^{|} = \widehat{K}^{|} + c_\M(x_{1:5})$           & 140.1 \\
\bottomrule\end{tabular}\end{table}

Table~\ref{tab:example2-metrics} gives the measures. The prefix makes 88.8 bits
of the description unnecessary. The prefix is also expensive in itself. Its five terms cost 100.6 bits, about
as much as the thirteen terms of the first example (Appendix~\ref{app:example}),
because strings such as \emph{hcr3} are unlikely text. With $\widehat{Kt}^{|} = 50.3$ bits the item is still below the median of
the certified items, 101 bits.

\paragraph{The examinees.}

Each examinee saw the five terms and was asked for \emph{enu6, ylu7, tsu8,
reb9, reb10}. Twenty-one of the 36 attempts are correct, and seven examinees
answer correctly on a majority of their attempts.

\begin{table}[h]\centering\small
\caption{Attempts answered correctly on \texttt{0049}, out of three, with a
typical wrong answer where there is one.}
\label{tab:example2-exam}
\begin{tabular}{@{}lcl@{}}\toprule
examinee & correct & a wrong answer \\ \midrule
g3.6F, g3.8F, gpt5sol, opus*, sonnet*, fable* & 3 of 3 & \\
gpt5na  & 2 of 3 & yra6, yra7, hcr8, lir9, yam10 \\
g3.5FL  & 1 of 3 & yra6, yra7, hcr8, lir9, yam10 \\
haiku   & 0 of 3 & hcr6, lir7, yam8, yra9, yra10 \\
g3.1FL  & 0 of 3 & une6, uly7, ust8, ber9, ber10 \\
olmo2   & 0 of 3 & yra6, yra7, yra8, yra9, yra10 \\
q1.5B   & 0 of 3 & d, e, f, g, \ldots \\
\bottomrule\end{tabular}\end{table}

The wrong answers again follow the rivals. The most common one, \emph{yra6,
yra7, hcr8, lir9, yam10}, repeats the five letter groups shown with new
indices, which is the continuation of the structural rival that repeats the
prefix. This is the rival the item had to exclude to be unquestionable. Gemini 3.1
Flash-Lite takes the last three letters of each month and writes them
forwards. Its first term, \emph{une6}, is also the first term of one of the
composed rivals. The six strongest examinees solve the item on every attempt.

\paragraph{Under the other two reference machines.}

Both other machines also find the item unquestionable at $L_u = 5$, Qwen3-8B
with $\widehat{K}^{\V}_{\M}(x) = 114.8$ bits and 19.9 bits of slack, and
Phi-4-reasoning with 110.7 bits and 30.8 bits of slack. They differ from
Qwen3-14B in certification. For Qwen3-8B every hint failed, including
\emph{first 3 letters of reversed months}, and only the full equivalent
\emph{Reverses each month's name, takes the first 3 letters, and appends the
term index n.}\ certified, at 102.1 bits and 2688 tokens. For Phi-4-reasoning
every hint failed as well, one equivalent reached 10 of the 20 terms, and
another certified, at 117.0 bits and 1728 tokens. The item therefore has
$\widehat{Kt}^{|} = 50.3$ bits under Qwen3-14B, 113.5 under Qwen3-8B and 127.7
under Phi-4-reasoning. The three machines agree that five terms are enough to
fix the continuation, but only Qwen3-14B can continue it from a three-word
hint, and the other two need the whole rule. This is the kind of difference
between reference machines that Figures~4, 8 and~12 show across the battery.

\subsection{Item 0003: \texorpdfstring{$\widehat{\Ktc}=200.7$}{Kt=200.7}}
\label{app:example3}

The two previous items are at the easy end of the scale. Here we follow
\texttt{0003}, which is certified under Qwen3-14B at $\widehat{Kt}^{|} =
200.7$ bits, inside the supra window of every examinee in Table~4, and which
only two of the twelve examinees solve.

\paragraph{Generation and validation.}

The generator asked for the cell $(r_0, t_0) = (1, 4)$ with the theme
\emph{bovine milk fat percentages} from the domain \emph{farming}. The composer
returned the seed
\begin{quote}\ttfamily
A ledger starts with fat percentages [1, 3]. Pop the lowest, emit its element
symbol and ' pct fat'. Add double the percentage modulo 118 as a new debt,
filter out 0, and sort ascending.
\end{quote}
The theme appears only as decoration, in the suffix \emph{pct fat}. What the
rule describes is two numbers, starting at 1 and 3. At each step the smaller
one is printed as the symbol of the chemical element with that atomic number,
and is then replaced by its double modulo 118, the number of elements in the
periodic table. The first terms are
\begin{quote}\ttfamily
h pct fat, he pct fat, li pct fat, be pct fat, c pct fat, o pct fat, mg pct
fat, s pct fat, \ldots
\end{quote}
with atomic numbers 1, 2, 3, 4, 6, 8, 12, 16, which is the two doubling
sequences from 1 and from 3 merged in order. The sequence looks like the
periodic table until it skips boron. At term 13 the doubling of 64 wraps round
to 10, neon, and from there the atomic numbers look erratic. By term 37 the two
numbers have fallen into the same cycle, eight steps apart, so every block of
eight atomic numbers is printed twice in succession. The item was generated
with $N = 72$ terms.

The composer wrote 17 equivalents, but only six reproduced the sequence when
the verifier ran them. The queue is easy to state loosely and hard to state
exactly. It also wrote
15 rivals at the ladder $\Lambda = \{5, 13, 60\}$. At $L = 5$ they read the atomic numbers 1, 2, 3, 4, 6 in other ways. The
divisors of 12 give magnesium next, the divisors of 36 fluorine, the integers
not divisible by 5 nitrogen, and the integers ordered by Euler's totient
boron. At $L = 13$ they keep the queue and change the
doubling after the wrap, and at $L = 60$ they change what happens near the top
of the table.

\paragraph{Unquestionability under Qwen3-14B.}

The least expensive description of the whole sequence is an equivalent,
\emph{Start with 1 and 3. Output lowercase chemical symbol for smallest n plus
' pct fat'; replace n with 2n mod 118 (discarding 0); repeat.}, at
$\widehat{K}^{\V}_{\M}(x) = 272.1$ bits, so the bar for a rival is 292.7 bits.

At $L = 5$ eight rivals come inside it, the four number-theoretic readings
among them. At $L = 13$ none of the composed rivals does, but three structural
ones still cost less than the bar. From $L = 27$ on the structural rivals cost
more than the bar as well, but no rivals were composed at 27, so the next
length that could be tested was 60. There the simplest rival is a composed one at 373.5 bits, and the item is
unquestionable at $L_u = 60$ with 80.8 bits to spare. The prefix of 60 terms is therefore set by the ladder and not by the
item, which would probably have been unquestionable from 27 terms. This
matters for the examinees, as we will see.

\paragraph{Measured \texorpdfstring{$\hat r$}{r} and \texorpdfstring{$\hat t$}{t}.}

The typer read $p^\blacktriangledown$ five times and gave $\hat r = 1$ and
$\hat t = 5$ each time. The one retrieval is the chemical elements, and the
operations are those of the queue and the formatting. This is one more
operation than the four the cell asked for.

\paragraph{Certification under Qwen3-14B.}

Qwen3-14B was shown the first 60 terms and asked for the next $k_u =
\min(20, 72 - 60) = 12$. Table~\ref{tab:example3-annot} lists the candidates it
ran.

\begin{table}[h]\centering\small
\caption{The candidates run for \texttt{0003} under Qwen3-14B, with their
conditional cost $c_\M(p \mid x_{1:60})$, the tokens used, and how many of the
12 asked terms came out correct and in order.}
\label{tab:example3-annot}
\begin{tabular}{@{}lrrrl@{}}\toprule
description & bits & tokens & terms & outcome \\ \midrule
(empty)                     & 0.0   & 3000  & 0  & failed \\
Periodic table of elements  & 18.1  & 3000  & 0  & failed \\
periodic symbol x           & 38.2  & 3000  & 0  & failed \\
bovine milk fat percentages & 44.7  & 3000  & 0  & failed \\
Settles milk fat debts \ldots & 172.1 & 3000 & 8 & failed \\
L = [1, 3]; loop: \ldots      & 187.4 & 10464 & 12 & certified, at the third budget \\
\bottomrule\end{tabular}\end{table}

Every hint failed. With the theme as a hint, Qwen3-14B found that the atomic
numbers double, writing \emph{O (8), S (16), Ge (32), Gd (64)}, but it did not
find the second number or the modulus. A loose equivalent, \emph{Settles milk
fat debts in order by element symbol, doubling and wrapping percentages to
queue new debts}, reached eight of the twelve terms. Only the full algorithm
certified,
\begin{quote}\ttfamily
\texttt{L = [1, 3]; loop: x = min(L); L.remove(x); L.append(2*x \% 118);}\\
\texttt{output lowercase periodic symbol x + ' pct fat'.}
\end{quote}
and even then it failed with 3000 and with 6000 tokens and succeeded with
12000, using 10464. Five simpler hints, such as \emph{Periodic table} and
\emph{Element symbols}, were costed but not run. Given that the other element
hints failed, they are unlikely to have certified.

\paragraph{The item's measures.}

\begin{table}[h]\centering\small
\caption{The measures for \texttt{0003} under Qwen3-14B, in bits.}
\label{tab:example3-metrics}
\begin{tabular}{@{}lr@{}}\toprule
$\widehat{K}^{\V}_{\M}(x)$, the whole sequence, unconditional & 272.1 \\
$c_\M(x_{1:60})$, the prefix itself                            & 463.1 \\
$\widehat{K}^{|}$, the winning description given the prefix    & 187.4 \\
$\log_2 \tau$, with $\tau = 10464$ tokens                      & 13.4 \\
$\widehat{Kt}^{|} = \widehat{K}^{|} + \log_2 \tau$             & 200.7 \\
$\widehat{Kc}^{|} = \widehat{K}^{|} + c_\M(x_{1:60})$          & 650.5 \\
\bottomrule\end{tabular}\end{table}

Table~\ref{tab:example3-metrics} gives the measures. Here the prefix saves only
84.7 bits of a 272.1-bit description, because the rule has to be stated in full
even after 60 terms. The computation also counts for more than in the two previous items. Here
$\log_2 \tau$ adds 13.4 bits, against about 10 for them.

\paragraph{The examinees.}

The examinees were shown the same 60 terms and asked for the next five,
\emph{cn, sg, pu, yb, ti}, with atomic numbers 112, 106, 94, 70 and 22. These
five begin the second copy of the last block of eight, so they can be found by
running the rule, or by noticing that every block appears twice.

\begin{table}[h]\centering\small
\caption{Attempts answered correctly on \texttt{0003}, out of three, and the
kind of answer given. Atomic numbers are in parentheses.}
\label{tab:example3-exam}
\begin{tabular}{@{}lcl@{}}\toprule
examinee & correct & answers \\ \midrule
sonnet*      & 3 of 3 & cn, sg, pu, yb, ti (112, 106, 94, 70, 22) \\
opus*        & 2 of 3 & the same, and one attempt with no answer \\
fable*       & 1 of 3 & the same, the doubled answer, and counting on \\
gpt5sol, g3.8F, g3.6F & 0 of 3 & mostly the doubled answer, lv, fl, ds, no, rn (116, 114, 110, 102, 86) \\
g3.5FL       & 0 of 3 & the doubled answer, and counting on \\
gpt5na, g3.1FL & 0 of 3 & counting on, pr, nd, pm, sm, eu (59, 60, 61, 62, 63) \\
haiku, olmo2, q1.5B & 0 of 3 & other answers \\
\bottomrule\end{tabular}\end{table}

The answers sort into three levels of understanding
(Table~\ref{tab:example3-exam}). The weakest reading continues the periodic
table from the last element shown, cerium, giving the elements 59 to 63. The middle one, given by gpt5sol on every attempt and by several of the
strongest Gemini and Claude models, doubles the last atomic number modulo 118
and keeps doubling, $58 \to 116 \to 114 \to 110 \to 102 \to 86$. These examinees found
the doubling and the modulus, which is most of the rule, but treated the
sequence as a single number and missed the second one. Only Claude Sonnet 5*
and Claude Opus 5* give the correct continuation on a majority of their
attempts.

One attempt is marked in the results as correct although it is not. The first
of Gemini 3.8 Flash's three attempts gives the doubled answer, identical to its
other two, but the judge scored it five out of five. Counting it as wrong, 6
of the 36 attempts are correct.

The ladder affects this item in a second way. Because the examinees see 60
terms and not the 27 that would probably have sufficed, the repetition of the
blocks is visible, and an examinee can answer the five terms asked from the
repetition alone without finding the rule. The reference machine was asked for
twelve terms, and the last four of them start a new block that only the rule
predicts. So the item as examined is likely to be somewhat easier than its
$\widehat{Kt}^{|}$ says. %

\paragraph{Under the other two reference machines.}

Both other machines also find the item unquestionable at $L_u = 60$, with
103.4 bits to spare under Qwen3-8B and 53.6 under Phi-4-reasoning. Neither
certifies it. Under both, the hints, the loose equivalent and the full
algorithm mostly produced the first eight of the twelve terms and then went
wrong, which is what an examinee gets from the repetition of the blocks
without the rule. This includes Phi-4-reasoning with the hint \emph{Double
atomic number modulo 118}, and both machines with the full algorithm. Their
budgets did not reach the 10464 tokens that Qwen3-14B needed. Only the listing
of the continuation succeeded, which is a fallback and not a certification,
so the item is in the battery of Qwen3-14B alone.

\subsection{Item 0084: \texorpdfstring{$\widehat{\Ktc}= 302.9$}{Kt= 302.9}}
\label{app:example4}

The last item, \texttt{0084}, is one of the two hardest certified items
that no examinee answers correctly on any attempt. %
This one is only certified under Qwen3-14B at
$\widehat{Kt}^{|} = 302.9$ bits. %

\paragraph{Generation and validation.}

The generator asked for the cell $(r_0, t_0) = (1, 7)$ with the theme
\emph{snooker ball point values} from the domain \emph{games}. The composer
returned the seed
\begin{quote}\ttfamily
For term n from 1, take the n-th snooker ball in point-value order (1: red, 2:
yellow, 3: green, 4: brown, 5: blue, 6: pink, 7: black, cycling every 7).
Multiply n by the ball's point value, square the result, take its digit sum,
add n to this sum, reverse the digits of the sum, divide by 2 if even, and
format as the ball name and final number.
\end{quote}
Each term has two parts. The ball cycles through the seven colours and is easy
to continue. The number needs six steps. For $n = 4$ the ball is brown, worth 4,
and $4 \times 4 = 16$, $16^2 = 256$, whose digit sum is 13, plus 4 gives 17,
reversed 71, which is odd, so the term is \emph{brown 71}. The first terms are
\begin{quote}\ttfamily
red 1, yellow 9, green 21, brown 71, blue 81, pink 21, black 41, red 81,
yellow 81, green 91, brown 3, blue 6, \ldots
\end{quote}
The item was generated with $N = 72$ terms.

The composer wrote 17 equivalents and eleven reproduced the sequence, all of
them close to pseudocode. A loose summary, \emph{Applies seven sequential
numerical and formatting transformations to each index using snooker ball
point values}, did not. It also wrote 15 rivals at the ladder $\Lambda = \{5,
13, 60\}$. At $L = 5$ five of the six keep the colours and change the
arithmetic, giving numbers such as $-139$ and 152 at the sixth term. At
$L = 13$ they change the cycle, for instance to thirteen colours. At $L = 60$
they treat the 60 terms shown as a period in one way or another, and one of
them simply starts again from \emph{red 1}.

\paragraph{Unquestionability under Qwen3-14B.}

The least expensive description of the whole sequence is an equivalent at
$\widehat{K}^{\V}_{\M}(x) = 302.7$ bits, so the bar for a rival is 323.3 bits.
The structural rivals come inside it at every length up to 27, where the
cheapest costs 313.9 bits. They are outside from 38 on, but no rivals were
composed at 38 or 49. At 60 the cheapest rival is a composed one, at 328.7
bits, only 5.4 bits above the bar, and the item is unquestionable at
$L_u = 60$.

\paragraph{Measured \texorpdfstring{$\hat r$}{r} and \texorpdfstring{$\hat t$}{t}.}

The typer found one retrieval, the snooker balls in point-value order, on every
reading. It counted five operations on four of its five readings and six on the
other, which gives $\hat t = 5.2$ with a spread of 0.4. The seed describes six
arithmetic steps, while the cell asked for seven.

\paragraph{Certification under Qwen3-14B.}

Qwen3-14B was shown the first 60 terms and asked for the next 12.
Table~\ref{tab:example4-annot} lists the ten candidates it ran, with the three
hints on one line.

\begin{table}[h]\centering\small
\caption{The candidates run for \texttt{0084} under Qwen3-14B, with their
conditional cost $c_\M(p \mid x_{1:60})$, the tokens used, and how many of the
12 asked terms came out correct and in order.}
\label{tab:example4-annot}
\begin{tabular}{@{}lrrrl@{}}\toprule
description & bits & tokens & terms & outcome \\ \midrule
(empty)                          & 0.0          & 3000  & 0  & failed \\
three hints                      & 79.4--92.2   & 3000  & 0  & failed \\
Applies seven sequential \ldots  & 146.3        & 3000  & 0  & failed \\
an equivalent in pseudocode      & 242.2        & 6000  & 8  & failed \\
three more equivalents           & 251.0--282.1 & up to 12000 & 0 & failed \\
the seed                         & 290.4        & 6000  & 12 & certified, at the second budget \\
\bottomrule\end{tabular}\end{table}

With no description, Qwen3-14B treated the number as a counter and continued
\emph{blue 426, pink 427, black 428}. The three hints failed within 3000
tokens. The loose summary failed too, and its reply contains \emph{blue 27,
pink 37, black 9, red 57}, which is the start of the wrong answer given by
three of the strongest examinees. Four equivalents in pseudocode failed, one
reaching eight of the twelve terms and one failing even with 12000 tokens. Only
the seed, the most expensive candidate run, certified.

The seed's run ended at the budget of 6000 tokens with the answer already in
its working, so $\tau$ is recorded as 6000. Counted by hand, the answer was
complete after about 5916 tokens, which would lower $\widehat{Kt}^{|}$ by 0.02
bits. %

\paragraph{The item's measures.}

\begin{table}[h]\centering\small
\caption{The measures for \texttt{0084} under Qwen3-14B, in bits.}
\label{tab:example4-metrics}
\begin{tabular}{@{}lr@{}}\toprule
$\widehat{K}^{\V}_{\M}(x)$, the whole sequence, unconditional & 302.7 \\
$c_\M(x_{1:60})$, the prefix itself                            & 455.7 \\
$\widehat{K}^{|}$, the seed given the prefix                   & 290.4 \\
$\log_2 \tau$, with $\tau = 6000$ tokens                       & 12.6 \\
$\widehat{Kt}^{|} = \widehat{K}^{|} + \log_2 \tau$             & 302.9 \\
$\widehat{Kc}^{|} = \widehat{K}^{|} + c_\M(x_{1:60})$          & 746.1 \\
\bottomrule\end{tabular}\end{table}

Table~\ref{tab:example4-metrics} gives the measures. The prefix saves only 12.3
bits, so after 60 terms the rule costs almost as much as it does with none.
The colours are obvious from the prefix, but the numbers carry almost no
information about how they were made. %
This rule is expensive to state, and $\widehat{Kt}^{|}$ places the
item among the hardest in the battery, which the results confirm.

\paragraph{The examinees.}

The examinees were shown the 60 terms and asked for \emph{blue 4, pink 49,
black 9, red 19, yellow 9}. None answered correctly on any of the 36 attempts.
Nearly all of them continued the colours correctly and got every number wrong,
and Table~\ref{tab:example4-exam} shows what they gave.

\begin{table}[h]\centering\small
\caption{Answers to \texttt{0084}. No attempt is correct.}
\label{tab:example4-exam}
\begin{tabular}{@{}ll@{}}\toprule
examinee & an answer \\ \midrule
gpt5sol, opus*, fable* & blue 27, pink 37, black 9, red 57, yellow 77 \\
sonnet*                & blue 45, pink 65, black 20, red 90, yellow 10 \\
g3.8F                  & blue 81, pink 21, black 41, red 81, yellow 81 \\
g3.1FL                 & blue 27, pink 47, black 67, red 87, yellow 7 \\
haiku                  & blue 47, pink 67, black 19, red 95, yellow 47 \\
olmo2                  & 81, 16, 31, 63, 95 \\
q1.5B                  & Red 101, Yellow 101, Green 101, Brown 101, Blue 101 \\
\bottomrule\end{tabular}\end{table}

gpt5sol, Claude Opus 5* and Claude Fable 5.1* each give the same wrong
continuation on at least one attempt, the one the loose summary produced under
Qwen3-14B. Its third term, \emph{black 9}, happens to be correct, but the first
is wrong, so it scores nothing under either marking. Gemini 3.8 Flash repeats
the numbers of terms 5 to 9, as if the sequence had period 56, and Gemini 3.1
Flash-Lite adds 20 each time, modulo 100.

\paragraph{Under the other two reference machines.}

Both other machines find the item unquestionable at $L_u = 60$, Qwen3-8B with
13.4 bits to spare and Phi-4-reasoning with 32.0. Neither certifies it. Their
hints and equivalents stopped at budgets of 1500 or 3000 tokens with no
correct term, well short of the 6000 that the seed needed under Qwen3-14B.
So the item is in the battery
of Qwen3-14B alone.

\section{The Full Pipeline}\label{app:pipeline}

The pipeline was summarised in \Cref{fig:pipeline}. To limit the clutter, two secondary processes are not included in the figure, but we need to include them here. The first one is a calibration of several parameters: bits per character $\gamma$, $c_0$, $\delta$ and $\beta$, for a rough estimation of $K$ used by the generator and unquestionability modules. More precisely, these are necessary for the prefix and continuation estimation, during generation, so we know at what lengths to generate the rivals, and these parameters are part of the unquestionability condition too. The second process is a measurement of the values of $(r,t)$ once the item has gone through unquestionability. It is at this point where we have more compact programs for the item that may change the components of the items, either by cancellation (e.g., two reversals being concatenated) or because one is lost during a chain of transformations. Occasionally, $(\hat{r},\hat{t})$ can end up being larger than originally, not only lower. Let us go through these stages in order. 

\subsection{The calibrator}

The calibrator measures 379 fixed probes of a diversity of sequences of different families, under each machine that is selected, and reports several
quantities, given in \Cref{tab:calibration}. First, $\gamma$, bits per character, which depends on the tokeniser as much as on the power of machine. The table
gives both the median and the 80th percentile, because the generator uses the
percentile.
Second, $\delta$ is the cost of the shortest exception clause the machine can state,
which is the shortest way of writing ``and then something else''. This is the
only calibration constant in %
\Cref{cor:compressibility-cond}.
Third, 
$c_0$ and $\beta$ come from one fit. The listing of a sequence is written out
at prefix lengths 5, 10, 15, 20, 25 and 30, each is costed by the machine, and
a line $\lambda(L) = \beta L + c_0$ is fitted to those six points. $c_0$ is
the part that does not grow with the prefix, the wrapper and the framing.
$\beta$ is what one further term costs.

\begin{table}[t]
\centering\small\setlength{\tabcolsep}{5pt}
\caption{The calibration constants, measured on 379 fixed probes for six LLMs. The last three were measured for robustness, but not used as reference machines. $\gamma$ is bits per character, reported as the
median and the 80th percentile over the probes. $\delta$ is the cost of the
shortest exception clause. $c_0$ and $\beta$ come from fitting $\lambda(L) = \beta L +
c_0$ to the cost of listing a sequence at prefix lengths 5 to 30. $\beta$ is
reported as the median over the sequence families whose fit reaches
$r^2 \geq 0.75$, with the range over those families in brackets and the number
of families that qualify out of five. The pooled row is what the generator read: the most conservative value of each
constant over the four machines calibrated at the time (Qwen3-8B, gemma-3-4b,
OLMo-2-7B and Llama-3.1-8B), that is, the largest $\gamma_{80}$ and $\delta$,
and the smallest $c_0$ and $\beta$. Its $\beta$ is the smallest value of any
qualifying family, Llama-3.1-8B's on element names.}
\label{tab:calibration}
\begin{tabular}{@{}lrrrrrc@{}}
\toprule
 & $\gamma$ & $\gamma_{80}$ & $\delta$ & $c_0$ & $\beta$ & fits \\
 & \multicolumn{2}{c}{bits/char} & bits & bits & bits/term & \\
\midrule
\multicolumn{7}{@{}l}{\textit{reference machines}}\\
Qwen3-14B        & 1.57 & 2.21 & 20.6 & 121.2 & \phantom{0}4.11 \;[1.1, 10.6]  & 4/5 \\
Qwen3-8B         & 1.65 & 2.19 & 22.7 & 129.6 & 11.87 \;[6.0, 17.7]            & 2/5 \\
Phi-4-reasoning  & 1.79 & 2.16 & 16.3 & 129.8 & \phantom{0}9.93 \;[3.6, 16.3]  & 2/5 \\
\addlinespace[3pt]
\midrule
\multicolumn{7}{@{}l}{\textit{other language models to make the estimates more robust}}\\
gemma-3-4b-it    & 2.52 & 3.31 & 22.3 & 159.6 & \phantom{0}6.08 \;[3.4, 22.3]  & 4/5 \\
OLMo-2-7B        & 1.99 & 2.42 & 18.7 & 119.2 & 10.07 \;[6.0, 23.5]            & 4/5 \\
Llama-3.1-8B     & 1.40 & 2.04 & 17.1 & 116.6 & \phantom{0}4.88 \;[1.0, 17.4]  & 5/5 \\
\midrule
pooled, as used  & 3.31 & --- & 22.7 & 116.6 & \phantom{0}0.98 & --- \\
\bottomrule
\end{tabular}
\end{table}

The fit is repeated for nine families of probe sequences, so there are nine
values of $\beta$ per machine and they differ by more than an order of
magnitude. For instance, under Qwen3-14B, counting numbers gives $\beta = 0.55$ but interleaved
two-retrieval sequences give $10.62$. A term that is the next integer costs almost
nothing to write down while term drawn from two unrelated retrievals costs a great
deal.
 
Only five of the nine families are used, those that have at least forty different terms, because that is what typical items look like. Counting and weekdays are costed and reported but excluded, since a $\beta$
taken from them would be an order of magnitude low. And a family is used only
if its fit reaches $r^2 \geq 0.75$. The last column of
\Cref{tab:calibration} gives how many of the five qualify. Qwen3-8B and
Phi-4-reasoning have only two each, so their $\beta$ is a median of two values
that differ by a factor of three or four. %
 
The $\beta$ reported in the table is therefore a median over the qualifying
families, with the range over those families in brackets. A third value,
$\beta_{\mathrm{high}}$, is the largest of them. It is used as a stopping
rule and not as an estimate.

With this, the generator is ready to read the last row of \Cref{tab:calibration}, because the generator is independent from any reference machine in particular. This is $\delta_{\mathrm{gen}}$= 22.7. On the contrary, the specific $\delta_\M$ for each reference machine (20.6, 22.7, 16.3) are used in the unquestionable stage.

\subsection{The Generator}
\label{app:generator}

The generator produces candidate items, iteratively, by growing and consulting a pool of themes of several domains, and progressively adjusting the generation to the ($r,t$) grid filling the gaps. It is meant to be called in successive turns to make grow the battery at any time. The generator does not use the reference
machine $\M$ but a composer $\C$, which can be any commodity LLM for creating candidates.

At the start, with an empty battery, the generator asks the
composer for themes, which are two to five words naming a subject with no
rule and no verb. For instance, ``chess openings'' and ``Beaufort scale'' qualify
but ``take the planets in order'' does not. Themes are kept in a pool that
persists across runs, and grows when the existing themes have been used many times already. Themes are clustered into domains. With the current battery, the pool holds 350 themes over 61 domains, the most
common being cuisine, sport, law, farming and games. No domain accounts
for more than a tenth of it. See \Cref{tab:topics}.

Each theme is stored with three labels supplied by a judge: whether its terms
are words or numbers, whether it is elementary, meaning a reader is expected
to know it outright, and whether it is generative, meaning its members can be
derived instead of recalled. The pool currently holds 159 word-termed themes
against 191 number-termed, and 161 generative against 189 not. A theme is
drawn for a slot, its use is counted, and it is then retired when it has been used more than a number of times. This can be modified through a parameter through successive calls of the generator.

For the themes, well-known texts are requested explicitly (e.g., a song, the beginning of a book or a speech). These are clear cases of retrieval, which depend on whether the passage has been memorised by the examinee. Only texts first
published before 1900 and originally in English are admitted, which keeps the
tokenised surface stable across models and avoids translations and copyright issues. All this is indicated in the 
\texttt{THEMES\_PROMPT} (see
\Cref{app:prompts}).  Separately, numeric themes are skipped with probability \texttt{THEME\_NUMERIC\_SKIP} to avoid a high proportion of numeric sequences. 

Items are requested by cell in the grid, which is a target number of retrieval steps
$r^{0}$ and transformation steps $t^{0}$, together with a theme.
As the values of $r$ and $t$ are measured at later stages of the process, the desired and actual values may diverge. Because of that, the request is calibrated with a linear fit over the items already
admitted, mapping what is asked for to what the later measures. %
The correction on $t$ is
capped at two steps. Either or both corrections can be disabled at any time (especially for the first items).

Through a specific prompt that gets $r$, $t$ and the $theme$ (see  \texttt{GEN\_PROMPT} in
\Cref{app:prompts})  the composer $\C$ returns a seed rule $p^{0}$ in natural language together
with the sequence $x$ it produces. The verifier $\V$, which in our experiments was the same model
in a separate context, with no sight of the first generation, then runs the
rule cold as a standalone prompt. The item is valid when a second cold run
reproduces the sequence the first produced. Only at this point we have some guarantees that we have a potentially useful item.

The item is screened for internal consistency, for overlap with items
already in the battery, and for not sharing a prefix with another item. For instance, two items whose first thirteen terms agree cannot both be unquestionable at that length.
An item that fails a screen is returned to the composer with the fault explained for {\em repair}, and the composer tries again up to five times, before the item is discarded. 
Some faults cannot be repaired and the prompt includes some of them as worked negative examples. %

The generator composes many candidates for every slot it wants to fill, but only a few may have reached this point, when they are added to the battery. Being strict here is useful, as the generator is less expensive than other parts of the process. At the moment, about two thirds of candidates end up being valid, and between half and three quarters of those become unquestionable, depending on the reference machine.

The length $N$ of a sequence is fixed before anything is run. From the
intended cell $(r,t)$ we take the expected length of its rule in characters,
and multiplying by $\gamma$ gives an estimate $\widehat{K}$ of its cost in
bits. \Cref{cor:compressibility-cond}, solved for the prefix, then gives the number of terms
such a rule is likely to need before it can be unquestionable,
$\hat L = (\widehat{K} + \delta - c_0)/\beta$. $N$ is this prefix plus
\texttt{K\_ADMIN} $= 12$ further terms, so that at least twelve terms are left
to ask for once the prefix is shown. $N$ is then kept between 28 and 72. The
lower limit leaves room for the first two rungs of the ladder, 5 and 13, with
twelve terms after them. The upper limit comes from a test of how long a
sequence the verifier could run out, and it is why the ladder stops at 60.
With the constants in \Cref{tab:calibration}, a rule of 40 characters gives $\hat L = 40$
and $N = 52$, and any rule longer than about 46 characters reaches $N = 72$.

It is only at this point that the item is added to the battery, and the generator starts preparing things for the next stages. First, the composer is asked for eight {\em equivalents}, which are other statements of the same rule, and for six {\em rivals} at each prefix length on the `ladder' (a series of the numbers of terms that a retried sequentially, instead of trying one by one, which would be very expensive). Formally a ladder $\Lambda$ is a sequence of increasing natural numbers.  Two lengths are used for
every item, 5 and 13, and two more are placed from the estimate in
\Cref{cor:compressibility-cond}. 

 For each item, a line is fitted, $\beta L + c_0$ unconditionally, in the frame with $K^{\V}_{\M}(x)$ and $\delta$. $\beta$ is what one
further term costs under $\M$ in this item's vocabulary. $c_0$ is the part
that does not grow with $L$, the wrapper and the framing, and at the lengths
in use it is usually most of the cost. For \texttt{0030}, $\beta = 3.95$
and $c_0 = 105.7$, so $\lambda(27) = 212.4$ against a limit of $174.1$, and
the rival is excluded by 38 bits. The term $\beta L$ alone would have left it
inside.

The condition holds of this rival when $K^{\V}_{\M}(x) + \delta < \beta L +
c_0$, so
\[
  L \;>\; \frac{K^{\V}_{\M}(x) + \delta - c_0}{\beta} ,
\]

This gives a lower limit on the prefix, not its exact length. Below the limit
the item is certainly questionable, because the rival that lists the prefix
is cheaper than the rule. Above it that rival is excluded, but the item still
has to be checked against the composed rivals, and one of them may block it
until a longer prefix. The estimate places the
last two points of $\Lambda$, one just above it and one further out by a
factor. An item whose estimate falls at 24 is given $\Lambda = \{5, 13, 26,
45\}$.
As said, the ladder always contains 5 and 13. Two more rungs are added from the
estimate $\hat L$: one just above it, at $\hat L + 2$, and one further out, at
$1.8\,\hat L + 2$ rounded, in case the estimate is too low, both capped at 60.

So for various $L \in \Lambda$, we generate rivals. A rival at $L$ is a description $q$ with
\[
  [\,\V(q)\,]^{2@3}_{1:L} = x_{1:L}
  \qquad\text{and}\qquad
  [\,\V(q)\,]^{2@3}_{L+1} \neq x_{L+1},
\]
A rival reproduces the prefix exactly and then deviates from it. Both equivalents and rivals are verified by $\V$ in the same way as the seed. 

   As $\V$ is stochastic, the same description given twice need not
   return the same sequence, while $\M$ (for any of the reference machines we are using in this paper) decodes greedily and is run
   once.
Accordingly, for verification we repeat the execution again. Namely, 
$[\,\V(p)\,]^{2@3}$
denotes running the seed, the equivalent or rival candidates three times and taking the value returned by at least two of
them, undefined if all three differ. %

The parameter $\beta_{\mathrm{high}}$ is used for one other purpose, which is to stop
spending on an item. Rivals account for about 70\% of what an item
costs to generate, so there is no point playing with them for a rule that could not
certify at any prefix it will be shown at. The largest $\beta$ is optimistic as the cost of abandoning a rival wrongly
is an item lost while the cost of not abandoning one is a few wasted calls.
 
All this is convoluted but responds to optimising the generator for efficiency. It is important, though, to clarify that this does not affect any of the definitional concepts, and without these heuristics the code would be much cleaner. With many more resources, we could explore all prefixes up to 400 terms without making any estimates. With the current approximation, We may lose some candidates, but the generator is much more efficient to run.

Rivals are very important for unquestionability, and we always include the `structural' rival, which is the program that states the prefix and then
deviates from it as an exception. This rival exists at every length by construction.

\subsection{The three filters and $(r,t)$ measurement}
\label{sec:filters}

An item gets its $\Ktc$ difficulty only if it passes three stages: validity, unquestionability and certification, in order. 

After generation, the rivals have been tested as several $L$, and the goal of the unquestionability step is to determine the minimum $\Lu$, the
length at which the item becomes unquestionable. Two counts of terms follow
$\Lu$ and should not be confused. $\ku = \min(20,\, N - \Lu)$ is how many
terms a description must produce for the third filter, whatever the item has
left. $\kf = 5$ is how many terms an examinee is asked for. 

While the generator has executed the seed on the composer, we now need to determine if it is long enough for $N$. This involves two independent executions of the seed:
\[
  [\,\V(\pseed)\,]^{2@3}_{1:N} \;=\; [\,\V(\pseed)\,]^{2@3}_{1:N} \;=\; x_{1:N},
\]
with the second run carried to 100 terms and the appropriate prefix compared over the first $N$. This is the {\em validity} step. No
reference machine takes part. %

The next stage is {\em unquestionability}. The item is unquestionable at $L$ under $\M$ when every rival at $L$ satisfies
\[
  \cost_\M(q) \;>\; K^{\V}_{\M}(x) + \delta ,
  \qquad
  K^{\V}_{\M}(x) = \min\{\cost_\M(p) : [\,\V(p)\,]^{2@3} = x_{1:N}\},
\]
with $\delta$ the clause constant of $\M$. Both costs are unconditional, for
the reason given in Appendix~\ref{app:whyunconditional}.

In this phase, equivalent programs (and the seed) are run on $\M$ to produce an estimate of $\widehat{K}(x)$. Then we find the lowest $L$ that is unquestionable, but $L$ does not range over every length. As said before, rivals are composed to diverge at a
ladder $\Lambda$ fixed before the item is judged. %
Then
\[
  \Lu \;=\; \min \{\, L \in \Lambda \;:\;
  \cost_\M(q) > K^{\V}_{\M}(x) + \delta \text{ for every rival at } L \,\},
\]
undefined if no $L \in \Lambda$ qualifies, in which case the item is
questionable and goes no further. A length outside $\Lambda$ is never tried,
because a rival must have been composed to diverge for those specific steps in the ladder.

A rival blocks an item when $\cost_\M(q) \le \widehat{K}(x) + \delta$, so
$\delta$ sets how close a rival has to come. Phi-4 has the narrowest of the
three and its condition is therefore the strictest, while Qwen3-8B has the
widest. %

The retrieval and transformation values that were used for generation $(r_0,t_0)$ may not match the description (program) that is used to characterise the item. To remedy this, we measure them again by choosing $\pK$, the description
that achieved $K^{\V}_{\M}(x)$. If no equivalent $\pK$ exists then the bare statement (a simplified version of the seed) and the seed are used. Then an LLM judge splits that description into
components of three types: retrieval, meaning anything that needs something to be remembered (the alphabet and the natural numbers are excluded); operations, meaning anything done to the terms that
needs nothing remembered; and control, which is the scaffolding. $\hat r$ is
the number of distinct retrieval components and $\hat t$ the number of
distinct operations. The call to the judge is made five times at temperature $0.4$ and the
median reported with the spread kept. %
Nothing downstream
depends on this new measurement $(\hat{r},\hat{s})$: they place an item in the $(r,t)$ grid for reporting
coverage and for choosing what to generate next, but no definition uses them.

The final stage is certification, conducted by a module we call the annotator, where the evaluation and verification of programs is conditional. 

The
item is certified at $\Lu$ under $\M$ when some description reproduces the
continuation from it:
\[
  \exists\, p \;:\;
  \,\M(p \mid x_{1:\Lu})\, \;=\; x_{\Lu+1:\Lu+\ku},
  \qquad \ku = \min(20,\, N - \Lu).
\]
This tries to approximate $\Kco  = \cost_\M(\pKc \mid x_{1:\Lu}) = \min_p \cost_\M(p \mid x_{1:\Lu})$. 

The empty description is also valid, when $\M$ continues the sequence
unaided. In this case, it certifies like any other and the cost is zero.

For items where nothing certified, the listing $p^\lambda$ supplies a value for $\Kcc$ as
a fallback. It states $\ku$ terms and is correct by construction, so it is not
required to run, and $\ku$ is the length at which it is costed. That value is
an upper bound, so the item is considered not certified, and not used in the
analysis.

The two main filters, unquestionability and certification rest on different quantifiers and one is conditional
and the other is not. Unquestionability is universally quantified: one rival is enough to discard the item. Certification is existentially quantified: as soon as one description reproduces the continuation the item is certified. But in both cases, the more resources are put in the search the better, to be safer about unquestionability and to have a good estimate of the metrics.

\section{The Prompts}
\label{app:prompts}

Code prompts are fixed text with a few slots filled at call time that are used to govern the behaviour of the LLMs. Several modules in the pipeline use prompts for the LLMs involved in the process: the composer $\C$ and verifier $\V$ in the generation and validation stages, $\M$ during the unquestionability and certification stages, and $\E$ and $\J$ during the evaluation. The unquestionability process only makes forward pass for costing programs using $\M$, but still uses simple prompts for that. Prompts are crucial in the behaviour of the pipeline. Because of that, prompts are
hashed with  the hash being written into the run record, so the results can be traced back to the specific wording that produced it. \Cref{tab:prompts}
lists the prompt with the opening words of each. The full prompts can be found in the code.

\begin{table}[t]
\centering\small\setlength{\tabcolsep}{4pt}
\newcolumntype{P}[1]{>{\raggedright\arraybackslash}p{#1}}
\caption{The prompts, where they are used, and an excerpt. Slots in braces are filled at
call time.}
\label{tab:prompts}
\begin{tabular}{@{}l l P{7.6cm}@{}}
\toprule
prompt & notebook & opens \\
\midrule
\texttt{THEMES\_PROMPT} & Generator
  & List \{k\} THEMES for sequence puzzles. A theme is two to five words
    naming subject matter \ldots \\
\texttt{GEN\_PROMPT} & Generator
  & THEME: \{theme\} \quad Invent one sequence puzzle around it. The theme
    says what the item is about \ldots \\
\texttt{RUN\_PROMPT} & Generator
  & Follow this rule and list the sequence it produces, \{k\} terms. \\
\texttt{EQUIV\_PROMPT} & Generator
  & Here is a rule and the sequence it generates. \ldots \\
\texttt{RIVAL\_PROMPT} & Generator
  & Here are the first \{L\} terms of a sequence: \{prefix\} \ldots \\
\texttt{REPAIR\_PROMPT} & Generator
  & Here is a rule you wrote for a sequence puzzle. It was thrown away, and
    the fault was \ldots \\
\texttt{DOMAIN\_PROMPT} & Generator
  & For each subject below, name three of its members, then say whether
    those members \ldots \\
\addlinespace[2pt]
\texttt{PRICE\_TPL} & Unquest.
  & Here is the beginning of a sequence: \{terms\} / A rule that generates
    it: \quad (costing frame) \\
\texttt{PRINT\_TPL} & Unquest.
  & Print: \quad (costing frame, for the structural rival) \\
\addlinespace[2pt]
\texttt{F\_rule\_slot} & Annotator
  & Sequence: \{x\_pref\} / Rule: \quad (costing frame; not the same wording
    as \texttt{PRICE\_TPL}) \\
\texttt{HINT\_PROMPT} & Annotator
  & Here are the first \{n\} terms of a sequence: \{prefix\} \ldots asks the
    composer for a hint, without showing it the seed \\
\texttt{VERIFY\_PROMPT} & Annotator
  & A model was asked to continue a sequence. The correct next terms are:
    \{x\_post\} \ldots Did it actually produce those terms, in order, as its
    answer or in its working? Returns a count, and decides certification \\
\texttt{TAU\_PROMPT} & Annotator
  & A model was asked to continue a sequence. The correct next terms are:
    \{x\_post\} \ldots Quote the shortest exact phrase from its text that
    ENDS at the point where it had produced all of those terms for the first
    time. Returns a location, and fixes where $\tau$ stops \\
\texttt{JUDGE\_PROMPT} & Annotator
  & A small language model was asked to continue a sequence and ran out of
    room before \ldots grants or refuses a further token budget \\
\addlinespace[2pt]
\texttt{EXAM\_PROMPT} & Evaluator
  & Continue the sequence with the next \{k\} \ldots (see \Cref{fig:examprompt}) \\
\bottomrule
\end{tabular}
\end{table}

The composer is asked for themes and mechanisms in two separate calls.
\texttt{THEMES\_PROMPT} asks only for a subject matter, two to five words, with
no rule and no verb, and requires about a quarter of each list to be
well-known texts published before 1900 in English. \texttt{GEN\_PROMPT} then
asks for one sequence built around one of those themes, at a target number of
retrieval and transformation steps and within a character budget. It states
one hard requirement: a single rule must fix every term with ties broken
by a stated convention (to help the composer, it includes worked examples of common faults). One such fault is a rule that counts the letters of a recalled word, which asks the machine for something it cannot see (though thinking can help with that). Another is a counter with a fixed word appended, which decorates a count without adding a mechanism. The prompt also lists the themes the battery already uses, and asks the composer to refuse a theme if it cannot work with it.

\texttt{RUN\_PROMPT} gives the verifier a rule and asks for the sequence it
produces. The same prompt is used for the seed, the equivalent and the rivals,
which is what makes the three comparable. \texttt{EQUIV\_PROMPT} asks for
other statements of the same rule, \texttt{RIVAL\_PROMPT} for rules that
reproduce a given prefix and then depart from it, and
\texttt{REPAIR\_PROMPT} returns an item to the composer with the fault named.

The unquestionability and annotation modules include a series of prompts for costing mostly.

The remaining prompts are for marking. \texttt{JUDGE\_PROMPT} in the
evaluator shows the prefix, the intended continuation and a reply, and asks
how many terms were produced in order before the first mistake. The
annotator has a version of the same prompt that also sees the previous
attempt, and answers whether more terms were derived this time, which is how
a further token budget is granted. \texttt{TAU\_PROMPT} reads a reply and
reports where the answer begins, separating the tokens spent reasoning from
the tokens spent answering.

Finally, the examinee sees \texttt{EXAM\_PROMPT} and nothing else. It is shown in
full in \Cref{fig:examprompt}.

\begin{figure}[t]
\centering
\fbox{\begin{minipage}{0.92\linewidth}
\footnotesize\ttfamily
Continue the sequence with the next \{k\} terms only.\\
Sequence: \{seq\}\\
Think if you need to, but end your reply with the answer alone on the final
line, the \{k\} terms separated by commas and nothing else.\\
Next \{k\} terms:
\end{minipage}}
\caption{\texttt{EXAM\_PROMPT} is the prompt shown to examinees.
\texttt{\{seq\}} is the prefix $x_{1:n}$, comma separated, and
\texttt{\{k\}} is $\kf$, which is 5 throughout. }
\label{fig:examprompt}
\end{figure}

The examinee is not told the rule, the domain, or how many terms the item
holds. The instruction about the final line allows a reply to be marked
without parsing the model's reasoning.

\section{Code and Main Files}
\label{app:notebooks}

The code is organised into seven python notebooks that follow the pipeline of \Cref{fig:pipeline} and \Cref{sec:pipeline}. They are independent files that exchange information through files, containing items, programs, results, etc. \Cref{tab:notebooks} gives the inputs and outputs of each notebook.

Three of the notebooks take one reference machine per run and write in the same file under a tag derived from its name, leaving the entries of other machines untouched. This means that a fourth machine can be
added later without repeating any earlier work. Unquestionability generates independent files, but as it is writing on {\tt \small ctest\_battery.json}, they are not parallelisable.

The files are held as a dataset, refreshed at the start of each run and pushed at the end (sometimes at checkpoints), to be able to run on environments that are erased between sessions (e.g., Colab or Kaggle code).

\begin{table}[t]
\centering\small\setlength{\tabcolsep}{4pt}
\newcolumntype{P}[1]{>{\raggedright\arraybackslash}p{#1}}
\caption{What each notebook reads and writes. \texttt{<tag>} is the reference
machine: \texttt{qwen3-14b}, \texttt{qwen3-8b} or \texttt{phi-4-reasoning}.
Every notebook appends one row to \texttt{ctest\_runs.csv}, which records the GPU/CPU times, wall clock times, the 
build, the settings and the prompt hashes.}
\label{tab:notebooks}
\begin{tabular}{@{}l P{3.4cm} P{3.2cm} P{4.3cm}@{}}
\toprule
notebook & what it does & reads & writes \\
\midrule

\textsc{Calibrator}
  & measures 379 fixed probes under one machine to obtain $\delta$, $\beta$
    and $c_0$
  & ---
  & \texttt{ctest\_calibration.json} \\
\addlinespace[3pt]

\textsc{Generator}
  & composes items, verifies them, screens and repairs them, and asks for
    equivalents and rivals
  & \texttt{battery}, \texttt{themes}, \texttt{calibration}
  & \texttt{ctest\_battery.json}, \texttt{ctest\_programs.json},
    \texttt{ctest\_themes.json}, \texttt{ctest\_prompts.json} \\
\addlinespace[3pt]

\textsc{Unquestionability}
  & costs every candidate under one machine and decides $\Lu$ and
    $\widehat{K}$
  & \texttt{battery}, \texttt{programs}, \texttt{calibration}
  & \texttt{ctest\_pricing\_}\allowbreak\texttt{<tag>.json}, and $n_M$, $K$
    and slack onto \texttt{ctest\_battery.json} \\
\addlinespace[3pt]

\textsc{RT-Measure}
  & reads $\pK$ and classifies its retrieval and transformation steps
  & \texttt{battery}, \texttt{pricing\_<tag>}
  & $\hat r$ and $\hat t$ onto \texttt{ctest\_battery.json} \\
\addlinespace[3pt]

\textsc{Annotator}
  & runs candidate descriptions under the machine and records which of them
    certify
  & \texttt{battery}, \texttt{programs}, \texttt{pricing\_<tag>}
  & \texttt{ctest\_metrics.json} \\
\addlinespace[3pt]

\textsc{Evaluator}
  & administers the certified items to the examinees
  & \texttt{battery}, \texttt{metrics}
  & \texttt{ctest\_results.json} \\
\addlinespace[3pt]

\textsc{Results}
  & produces every table and figure
  & all of the above
  & \texttt{figs/} \\

\bottomrule
\end{tabular}
\end{table}

\section{Parameters}
\label{app:params}

\Cref{tab:params} lists the parameters that are used to manoeuvre the execution of the code. The generator
values control how many items are attempted and how much is spent on each.
The annotator values control how much work a candidate description is allowed
before it is recorded as having failed. These parameters allow for an adaptive series of runs to populate the battery, certifying as many items as possible while keeping the use of the resources under control.

\begin{table}[t]
\centering\small\setlength{\tabcolsep}{5pt}
\newcolumntype{P}[1]{>{\raggedright\arraybackslash}p{#1}}
\caption{The main settings of the generator, the annotator and the evaluator.}
\label{tab:params}
\begin{tabular}{@{}l r P{10.2cm}@{}}
\toprule
setting & value & what it does \\
\midrule
\multicolumn{3}{@{}l}{\textsc{Generator}}\\
\addlinespace[2pt]
\texttt{N\_PER\_CELL}   & 1--3  & items wanted in each $(r,t)$ cell \\
\texttt{OVERSAMPLE}     & 12    & candidates composed per item wanted \\
\texttt{TERMS\_MIN/MAX} & 28/72 & length of a sequence, by the load it carries \\
\texttt{K\_ADMIN}       & 12    & terms held back, so an item can be
                                  administered at any prefix up to 60 \\
\texttt{EXTEND\_REPS}   & 3     & cold runs of a rule, at temperatures
                                  0.0, 0.3 and 0.6 \\
\texttt{EXTEND\_WINS}   & 2     & of those runs that must agree on every term
                                  for the item to be valid \\
\texttt{REPAIR\_ROUNDS} & 5     & times a faulty item is returned to the
                                  composer before it is dropped \\
\texttt{N\_EQUIVALENTS} & 8     & other statements of the rule requested \\
\texttt{N\_RIVALS}      & 6     & rivals requested at each prefix \\
\texttt{RIVAL\_PREFIXES}& 5, 13 & prefixes rivals are always requested at, plus two from the estimate,
$\hat L + 2$ and $1.8\,\hat L + 2$, capped at 60                                  
                                  \\
\texttt{THEME\_REUSE}   & 1     & items allowed to use one theme \\
\texttt{CALIBRATE\_CAP\_R} & 0  & correction allowed on the retrieval target;
                                  0 means the target is left as it is \\
\texttt{CALIBRATE\_CAP\_T} & 2  & the same for transformation \\
\addlinespace[4pt]
\multicolumn{3}{@{}l}{\textsc{Annotator}}\\
\addlinespace[2pt]
\texttt{K\_MAX}         & 20    & terms the reference machine is asked to
                                  produce \\
\texttt{budgets}        & 1500, 3000, 6000 & token budgets for an item with a
                                  short prefix; 3000, 6000, 12000 for a long
                                  one \\
\texttt{hint\_budgets}  & 1500, 3000 & the same for a hint, which is tried
                                  before the full descriptions \\
\texttt{TOP\_N}         & 3     & cheapest candidates run per item \\
\texttt{RESCUE}         & on    & a description that nearly succeeded is
                                  retried with more room \\
\texttt{BUDGET\_JUDGE}  & gemini-3.5-flash-lite &
                                  decides whether to grant a further budget \\
\texttt{SCORING\_JUDGE} & gemini-3.6-flash & counts the
                                  terms produced before the first mistake \\
\addlinespace[4pt]
\multicolumn{3}{@{}l}{\textsc{Evaluator}}\\
\addlinespace[2pt]
$\kf$                   & 5     & terms an examinee is asked for \\
\texttt{REPS}           & 3     & attempts per item per examinee \\
\texttt{examinee budget} & 2048, 8000, 16000 &  token budget for an examinee's reply: 16000 for GPT-5.6-sol, Claude Opus 5,
  Sonnet 5 and Fable 5.1, 8000 for the other API models, and 2048 for
  OLMo-2-7B and Qwen2.5-1.5B
\\
\bottomrule
\end{tabular}
\end{table}

\FloatBarrier  %
\section{Disagreements between the Measure and the Scores}
\label{app:disagree}

In \Cref{tab:cert-metricsQwen14b} we saw $\Ktc$ reaches $-0.79$ Spearman correlation against accuracy over the 51 items certified under
Qwen3-14B. The items where it disagrees with the scores are informative,
because the kinds of disagreement have different causes. Let us inspect them.

\subsection{Items with high \texorpdfstring{$\Ktc$}{Kt} and good scores}\label{sec:highKtgoodscores}.

$\Ktc$ is a minimum over the candidate descriptions that were actually run.
The annotator tries a subset of conditional hints and unconditional equivalents from a finite set, and is also limited by budget constraints,  
so many candidates that would have cost less are sometimes never executed until the end. The first cause of a high $\Ktc$ and good scores can therefore originate from a search that stopped, and the
value we report for $\Ktc$ is a loose upper bound. This case can be minimised with more search for candidates, and some exploration with higher budget managed to reduce the measured $Kt$ of  \texttt{0076}, \texttt{0134}, \texttt{0112} and \texttt{0004}, which seem overestimated when looking at \Cref{fig:sample}. Note that a possible shift to lower values would likely make the slopes of the characteristic curves steeper and the abilities slightly lower, which could even make the supra windows larger, not smaller.

A second cause is a description the reference machine cannot derive from memory and must be named, i.e., a failed retrieval. The cost is long under the reference machine, and its length measures the
table, instead of the lower difficulty of actually retrieving it by other examinees that can derive it from memory. Both $c_\M(p)$ and
$\cost_\M(p \mid x_{1:n})$ come out high while the item is easy. This happens at the top of our range, but less often than it may seem. Of
the 10 items above 190 bits, two, 0004 and 0112, have descriptions that spell
out tables and are solved on most attempts by the six strongest examinees.
The other eight are answered correctly on at most 19\% of attempts overall.
Six of them have five or more operations, and in the remaining two the
description is long but the examinees fail as well. Accuracy also falls
through every band of $\widehat{Kt}^{|}$, including the top one: the six
strongest examinees average 98\%, 67\%, 65\%, 37\% and 19\% in the bands
0--50, 50--100, 100--150, 150--190 and above 190 bits. More details in \Cref{app:harditems}

The third cause is rare. Two certified items have $\tau$ equal to their budget,    their runs reached the budget with the answer already in their working, so
   $\tau$ is slightly too high, by 0.02 bits in the one we checked by hand. %
A short description may need a lot of computation, but runs out of tokens in the reference machine, forcing the certification at a much more expensive, but efficient description (e.g., one that enumerates some of the terms or intermediate steps). In other words, we pay description bits for execution time. 

\subsection{Items with low \texorpdfstring{$\Ktc$}{Kt} and poor scores}\label{app:lowktpoorscores}

This direction matters more, because it suggests the item is easy for the reference machine but difficult for the rest of the world.

One possible cause is the tokeniser. One item may have an advantageous tokenisation for the reference machine that does not work for other models. 

The serious case is a rule that names a one-way function. The first byte of a
standard hash, or a stream cipher keyed on a word, can be stated in a line. Consequently, 
$\Kco$ is small, while the terms will not reveal any information about the rule. Also, a good cipher is
inexpensive to evaluate forward, which is a design requirement, and the
difficulty is in inversion, which $\tau$ does not measure.

What keeps such items out of this battery is that the reference machine is a
language model with no interpreter. It cannot evaluate a hash forward either,
so no description of that kind certifies.

One way of keeping that situation at bay could be to say explicit that no cryptographic or pseudorandom number generators should be used. But this would still allow for similar behaviours by unexpected (almost) one-way functions. 

Interestingly, we calculate the conditional complexity given the prefix, $\Ktc$, and this can be compared with the unconditional $K$ or $Kt$. 
For a pseudorandom
sequence the prefix is by construction indistinguishable from noise and the two costs coincide. A description that is
inexpensive and gains nothing from the prefix therefore has the profile of a
one-way item, and this could be a signal to inspect or exclude the item.

Four of the 51 items certified for Qwen3-14B have that profile, with a conditional cost
below 80 bits and the two costs within 20 bits of each other. Three of the
four are easy, at
58\%, 83\% and 86\% success.

This comparison has been performed with $c_\M(p)$ for a $p$ that verifies with $\V$ and $c_\M(p|x_{1:n})$ of a $p$ that verifies with $\M$. Future work should measure both things with the same verification machine, minimising both metrics separately and then calculate the difference. In general, this is a traditional objection to Kolmogorov complexity as difficulty, as approximations of $K$ or $Kt$ are overestimations, but this case is an underestimation. But removing those easy instances no model can solve would boil down to surrender to the populational approach we aimed to solve in this paper. It may simply be that some stronger models start cracking the item with low $Kt$ eventually, reconciling $Kt$ with the observed difficulty. Or there is a new concept that renders the apparently chaotic sequences into a nice pattern, as happened in science with many phenomena that looked chaotic, from the movement of the planets to hereditary traits. Indeed, it is not known for certain that one-way functions really exist: this is ``arguably the most important problem in computer theory" \citep{levin2003tale}.   

\section{The upper end of the Qwen3-14B scale}
\label{app:harditems}

Appendix \ref{sec:highKtgoodscores} notes that $\widehat{Kt}^{|}$ can be high because the reference
machine cannot recall a table and has to be given it in the description. For the Qwen3-14B scale, the reference machine that has the widest supra-window, there may be concerns that the items there are idiosyncratic. Let us explore these items in more detail.

At the top of the scale there is the first question of whether
$\widehat{Kt}^{|}$ there measures the difficulty of the item or only what
Qwen3-14B does not know, especially for retrieval. If it were mostly the latter, the items at the top would be easy for
the examinees that have more general knowledge, partly because they have more parameters. This is not what we observe from the data.

\begin{table}[h]\centering\small
\caption{The 10 items above 190 bits under Qwen3-14B, with the measured
$\hat r$ and $\hat t$, the kind of description that certified them, and the
share of attempts answered correctly by all twelve examinees and by the six
strongest (gpt5sol, opus*, fable*, sonnet*, g3.8F and g3.6F).}
\label{tab:upper-items}
\begin{tabular}{@{}lrrrlrr@{}}\toprule
item & $\widehat{Kt}^{|}$ & $\hat r$ & $\hat t$ & certified by & all twelve & six strongest \\ \midrule
0146 & 380 & 0 & 3   & equivalent &  3\% &  6\% \\
0108 & 303 & 2 & 2   & equivalent &  0\% &  0\% \\
0084 & 303 & 1 & 5.2 & seed       &  0\% &  0\% \\
0139 & 255 & 0 & 5.8 & equivalent &  0\% &  0\% \\
0112 & 242 & 3 & 2.2 & equivalent & 31\% & 61\% \\
0004 & 237 & 2 & 5   & equivalent & 39\% & 78\% \\
0083 & 213 & 0 & 5.2 & equivalent &  0\% &  0\% \\
0054 & 202 & 1 & 5   & equivalent & 11\% & 11\% \\
0003 & 201 & 1 & 5   & equivalent & 19\% & 39\% \\
0109 & 192 & 0 & 5   & equivalent &  0\% &  0\% \\
\bottomrule\end{tabular}\end{table}

Table~\ref{tab:upper-items} lists the 10 items above 190 bits. Only two of
them, \texttt{0004} and \texttt{0112}, fit the concern. Their descriptions spell out tables,
the pig materials and the days of the week in one and the chess pieces and
the rainbow colours in the other, and the six strongest examinees solve them
on most attempts. These are two of the items that \Cref{sec:highKtgoodscores} already
identifies as overestimated, and for which further search lowered the
measured value. The other eight are answered correctly on at most 19\% of
attempts overall. Six of them have five or more operations, so their cost
comes from transformation and not from retrieval. In the remaining two the
description is long but the examinees fail as well. In \texttt{0146} it spells out
four short lists of words, and the examinees answer correctly on 1 of 36
attempts.

\begin{table}[h]\centering\small
\caption{Average accuracy by band of $\widehat{Kt}^{|}$ under Qwen3-14B,
over the 51 certified items.}
\label{tab:upper-bands}
\begin{tabular}{@{}lrrrrr@{}}\toprule
band (bits) & 0--50 & 50--100 & 100--150 & 150--190 & above 190 \\ \midrule
items           & 19    & 6     & 9     & 7     & 10    \\
all twelve      & 82\%  & 50\%  & 45\%  & 23\%  & 10\%  \\
six strongest   & 98\%  & 67\%  & 65\%  & 37\%  & 19\%  \\
\bottomrule\end{tabular}\end{table}

Table~\ref{tab:upper-bands} shows the same from the other side. Accuracy falls
through every band, and the top band is the hardest, for the strongest
examinees as for all of them. If the top of the scale mostly reflected what
the reference machine does not know, the top band would not be the one the
examinees fail most.

Three further points bear on the question. First, $\widehat{Kt}^{|}$ is
relative to its reference machine by design, as any measure of Kolmogorov
complexity is relative to its universal machine. What the machine does not
know costs bits to state, and this is the crystallised component that the
retrieval dimension keeps on purpose. An examinee that knows a table saves
those bits, as a person who knows the planets finds a sequence of planets
easy. Second, the choice of reference machine barely changes how the
examinees are ranked. Qwen3-8B and Phi-4-reasoning, which differ from
Qwen3-14B in size and in lineage and training data respectively, order the
twelve examinees almost as Qwen3-14B does, at Spearman $+0.99$ and $+0.92$
(Appendix~\ref{app:stats}). This would be unlikely if the scale reflected the particular
gaps of one machine. Third, as an estimate of the reference machine's own true $Kt$,
$\widehat{Kt}^{|}$ is an upper bound, and it can only err upwards. A
description the search did not find makes the value higher than it should
be, never lower. A longer search lowers such values, as further exploration
with a larger budget did for \texttt{0004} and \texttt{0112}.

Several reference machines can also be combined by taking, for each item,
the smallest $\widehat{Kt}^{|}$ among the machines that certify it, which
removes the gaps of any one machine. On the 62 items certified by at least
one of the three machines, this minimum comes from Qwen3-14B for 32 items,
from Phi-4-reasoning for 24 and from Qwen3-8B for 6. Its Spearman
correlation with accuracy is $-0.79$, as strong as that of Qwen3-14B alone
on its 51 items, and the abilities it gives correlate with ECI at Pearson
$+0.81$ and Spearman $+0.76$. At the top it changes nothing, since each of
the 10 items above 190 bits is certified by Qwen3-14B alone, so its value
there rests on that machine. Testing the top in the same way needs a larger
battery with more reference machines. This combination could be done in other ways, giving more weight to the reference machines in their areas of highest discrimination, or based on diversity to minimise the subjectivity or bias or choosing one single machine. The infra- and supra-intelligence windows were introduced to illustrate that we may need a staircase of reference machines to cover wider windows. 

The most expensive certified items also have conditional costs close to
their unconditional ones (\Cref{tab:items-machines}). \texttt{0146} costs 367 bits given the prefix and
364 without it, \texttt{0084} costs 290 and 303, and \texttt{0139} costs 243 and 263 (all these numbers without $\tau$). The
heuristic in Appendix~\ref{app:lowktpoorscores} flags items whose description is inexpensive and
gains nothing from the prefix, because for them a low $\widehat{Kt}^{|}$ would
hide a hard item. Here the descriptions are expensive, so
$\widehat{Kt}^{|}$ is already high, and the examinees fail these items, as it
predicts. Across the 51 certified items, the less the prefix saves, the
harder the item. The share $\widehat{K}^{|}/\widehat{K}$ correlates with
accuracy at Spearman $-0.61$, and the 16 items for which the prefix saves
less than a tenth of the description average 22\% accuracy, against 86\% for
the 15 continued with no description at all. A prefix that reveals little
about the rule is what makes an inverse problem hard, and $\widehat{Kt}^{|}$
records it. Of the four items that do have the profile of Appendix~\ref{app:lowktpoorscores}, with
both costs below 80 bits and within 20 bits of each other, the only one the
examinees fail is \texttt{0133}, whose short prefix admits a simpler
reading, the nursery rhyme (\Cref{tab:items-machines}). The
conditional cost can also exceed the unconditional one by a few bits, as in
\texttt{0146}, because the two are computed with different prompts and checked by
different machines, as Appendix~\ref{app:lowktpoorscores} notes.

\section{Unconditional Cost in the Unquestionability Filter}
\label{app:whyunconditional}

The previous discussion about one-way functions and the idea of comparing the conditional and unconditional costs, as well as the problem of an item that was certified at 5 because the unconditional program worked, suggests some tension in our pipeline, combining unconditional complexity for unquestionability and conditional complexity for certification. Why do not we just use the conditional version for all?

The difficulty is conditional. $K_\M^{\,\M}(\nxt \mid \shown)$
uses $\cost_\M(p \mid \shown)$, with the sequence first and the program after. The conditioning is the examinee's situation, and allows for partial hints that are sufficient when the prefix discloses much of the pattern already. A sequence starting with the planets does not need to code the planets as a hint, as they are obvious.

Unquestionability, on the other hand, uses $\cost_\M(p^\star) + \delta$, with both terms unconditional, as in
\Cref{cor:compressibility-cond}. Conditioning would destroy the comparison, for the following
reason.

Among the programs that reproduce $\shown$ and then diverge there is always
the listing rival, which prints $x_1,\dots,x_n$ and then prints anything other
than $x_{n+1}$. Costed unconditionally it comes to about $\beta n + c_0$,
which grows with $n$. That growth is what makes the necessary condition
\Cref{def:unq}. a statement about prefix length, and what allows a longer
prefix to exclude rivals.

Costed conditionally on $\shown$, the listing rival is nearly free, because a
language model can copy tokens it can already see, with a `listing' program such as ``{\tt repeat the input forever}'', which is consistent with the input by definition. Then
$\cost_\M(\text{listing} \mid \shown) \approx 0$ for every $n$, the budget
$\cost_\M(p^\star) + \delta$ is never small enough to exclude it, and every
item is questionable at every prefix whatever its rule. The quantity stops
depending on the item.

The two frames answer different questions. $\cost_\M(\cdot \mid \shown)$ asks
how hard the item is for a machine that has seen it, which is the difficulty.
$\cost_\M(\cdot)$ asks how short a program is in absolute terms, which is what
a claim about all inexpensive programs has to quantify over. An
unquestionability claim concerns programs written by someone who was never
shown the sequence.

The implementation follows this. The unquestionability notebook costs the
rivals, the least expensive program and the listing rival with
$\cost_\M(p)$, unconditionally. Only $\Kco$  and $\Ktc$, which are explicitly conditional, use the sequence as context. %

\section{Computing Resources Used and Needed for Reproduction}\label{sec:resources}

We report the resources used by the project and an estimate of what a clean
reproduction would need. Both are approximate, and we give each as a central
value with a range.

The GPU time is taken from the weekly quotas of the notebook services we
used, since each notebook's own log records a run only when it ends normally
and so misses every run that was stopped, timed out or failed. Over the
period of the logged runs the quotas give 355 GPU `T4' hours on one service, and
the other adds between 30 and 67 hours, depending on our estimates as this first service did not record usage per project. %
Constrained by our goal of making the test sustainable and reproducible, all GPU work ran on a single NVIDIA T4. Laptops in use for about ten hours a day per month in total, for orchestration and writing. 
Part of the code and the analysis were developed assisted by AI assistants,
OpenAI 5.6Sol, Gemini 3.1Pro, Claude Opus 5 (mostly) and 5.5, over about 6{,}500 model calls. We
count its use from the saved transcripts: about 10 million tokens produced,
between 7 and 13 million once the reasoning that is not saved is included,
and about 1.9 billion tokens of context read back across the calls, most of
it from the providers' caches when available, asking the models or estimating from our own logs and recollection.  
Remote (API) model calls are
counted in tokens, from the log for the composer, verifier, typer and judges
and from the results file for the examinees. Because some runs are missing
from the log, including those for earlier versions of the battery, we take
the recorded tokens as a lower bound, and twice that as an upper bound.

\begin{table}[t]\centering\small
\caption{Resources used by the project. Tokens are in millions and include
reasoning tokens. Energy and emissions are estimated under the assumptions in
the text.}
\label{tab:resources-project}
\begin{tabular}{@{}lrr@{}}\toprule
 & central & range \\ \midrule
GPU hours                         & 400  & 385 to 422 \\
laptop hours                      & 300  & \\
tokens sent, M                    & 30   & 20 to 40 \\
tokens received, M                & 180  & 120 to 240 \\
assistant, tokens produced, M     & 10   & 7 to 13 \\
assistant, context read, G        & 1.9  & 1.0 to 1.9 \\
energy, kWh                       & 122  & 66 to 342 \\
primary energy, GJ                & 1.2  & 0.7 to 3.4 \\
emissions, kg CO$_2$e             & 49   & 26 to 138 \\
\bottomrule\end{tabular}\end{table}

We follow the AI Energy Score initiative~\citep{aienergyscore}, which
reports the energy of inference per thousand queries and converts energy into
environmental impacts with factors from the ADEME Base Empreinte database.
Energy differs on the source. We assume 100\,W while the GPU is attached, which is
the T4's 70\,W limit and a share of the host, multiplied by 1.1 for the data
centre's cooling and power supply. We assume 45\,W for laptops, with a
range of 30 to 60\,W. The remote calls cannot be measured from outside, and
the models we called are proprietary, so they are not among those AI Energy
Score has measured. For them we assume 0.3\,Wh per thousand tokens received,
including the provider's overheads, with a range of 0.1 to 1.0\,Wh, and a
tenth of that for tokens sent.  
Over the 30{,}107 remote calls that were
recorded, this comes to about 1.2\,kWh per thousand queries, a high figure
because most of our queries asked for long reasoning. 
For the assistant we use the same assumption for the tokens it produced, and
a hundredth of it for context read from the cache, which is reused and not
recomputed. This adds about 9\,kWh, between 2 and 33\,kWh. 
The remote calls
account for most of the uncertainty, and at the central value for about half
of the energy.

For impacts we use the worldwide factors from AI Energy Score too, which
are 9.99\,MJ of primary energy and $7.378\times10^{-8}$\,kg Sb eq of abiotic
resource depletion per kWh. 
The project's central estimate of 122\,kWh then corresponds to 1.2\,GJ of
primary energy and 9\,mg Sb eq. 
   For emissions we use 402\,g of CO$_2$ equivalent per kWh, the factor AI Energy
   Score gives for combined-cycle natural gas turbines.

Table~\ref{tab:resources-clean} estimates the resources for building and
evaluating a battery of 100 valid items without failures or restarts. The
rates come from our measurement from runs that completed. Generating 100 valid items takes about
140 candidates, since 71 per cent of the saved items were valid. Each
reference machine then costs every item, types it, and annotates the items it
finds unquestionable, about 70 of the 100. Annotation dominates the GPU time,
at 25 to 50 minutes an item, most of it on items that no description
certifies and so run through the whole ladder of token budgets. The evaluation
covers the items unquestionable under any machine, about 80, for twelve
examinees reached through an API at three attempts each and two local
examinees at one attempt.

\begin{table}[t]\centering\small\setlength{\tabcolsep}{4pt}
\caption{Estimated resources for a clean reproduction on 100 valid items.
GPU hours are central values with ranges. Tokens are in millions. Other hours are CPU hours mostly. The last
two rows give the totals for one reference machine and for the three used in
this paper.}
\label{tab:resources-clean}
\begin{tabular}{@{}lrrrr@{}}\toprule
stage & GPU h & other h & tokens sent & tokens received \\ \midrule
generation, once              &            & 11 & 4.1 & 28 \\
per reference machine         & 43 (31--61) &  1 & 0.8 & 1.9 \\
evaluation, once              &  2          & 20 & 1.4 & 13 \\
\midrule
one reference machine         & 45 (32--63)  & 32 & 6 & 43 \\
three reference machines      & 130 (94--184) & 34 & 8 & 47 \\
\bottomrule\end{tabular}\end{table}

With three reference machines, the estimate is about 130 GPU hours, 55
million tokens, between 15 and 68\,kWh, between 0.15 and 0.68\,GJ of
primary energy, and between 6 and 27\,kg of CO$_2$ equivalent. The project used about three times as much GPU time and nearly
four times as many tokens, the rest going to development, failures and
restarts.
Two reasons account for much of the difference and can be turned into lessons learnt. A reference machine
that reasons at length needs a rule that stops its reply once the answer has
appeared, or every reply runs to the token budget. And a local examinee
decoded greedily gives the same reply each time, so one attempt is enough.

A reproduction needs one GPU with 16\,GB of memory, which is enough for the
reference machines in 4-bit quantisation, somewhere to keep the state files
between sessions, and API access to Gemini 3.6 Flash and 3.5 Flash-Lite for the composer, verifier, typer
and judges and to the provider of each examinee. At the rate reported by the
generator, the Gemini calls cost about 140 US dollars. The cost of the
examinees depends on which are chosen. The twelve used here received about 6
million tokens between them. On a single GPU the work takes about five and a
half days without interruption, but generation and evaluation can be parallelised, and different reference machines can be costed separately. %
A reproduction does not need the AI assistant, since the code is released.

In summary, the numbers in this section support responsible and sustainable research, performing the experiments that are necessary for our research questions, and for research questions that matter. Also, the whole point of the paper was to show that we can evaluate frontier models with a small number of items, and their generation has a reasonable cost, and energy consumption. 

\section{The Full Battery}\label{sec:battery}

We generated 151 items. This had to be, by design, higher than the effective items in the battery. While all the 151 items are kept, only the 107 valid items matter: the description of the remaining 44 were never consistently replicated. Then, for each particular reference machine, only a subset became unquestionable and from them, the actual battery for each $\M$, which only contains those that are also certifiable on that machine. Many reasons other than unquestionability explain why items fall in the process, and the comparison of the full 151-item list with the lists of certified items shows some biases. Apart from the tokens required for reproducing the sequence during inference (thinking) time, the tokeniser may also explain why many sequences that were expressed with non-alphanumerical patterns, such as `---*---" struggled to make the cut. More powerful reference machines could compensate for that, or in a future, tokenisers may change or even disappear, leading to a more natural notion of term as a finite string of characters rather than tokens.

Let us show first a summary of domains and operations in \Cref{tab:topics}. We see that cuisine, farming, sport, law and folklore top the retrieval list. Then, for operations, different kinds of sum and accumulations, modulus, multiplication, etc., top the transformations list, while string transformation such as reverse appear in mid table. Note that these are the mechanisms that can be named. There are some other specific transformations that do not fall in any of the categories.

\begin{table}[t]\centering\small\setlength{\tabcolsep}{4pt}
\caption{A summary of the 107 valid items in the battery (out of 151 in total) by domain of the theme (as used by the retrieval) and operation (as used by transformation). The left
column counts items by the domain of the theme they were built from. The column sums up to 107. The right column counts items whose rule contains a
given kind of operation, as found by the $r,t$-measurement procedure (i.e., according to $\hat{t}$ and the shortest unconditional program). An item may contain several, so the
right column sums to more than 107.}
\label{tab:topics}
\begin{tabular}{@{}lr@{\hskip 2em}lr@{}}\toprule
domain of the theme & items & operation in the rule & items \\ \midrule
cuisine & 11 & sum or add & 63 \\
farming & 10 & modulus & 50 \\
sport & 9 & digit sum & 36 \\
law & 9 & multiply & 31 \\
folklore & 9 & square or power & 30 \\
navigation & 6 & reverse & 26 \\
rhymes & 5 & cycle or rotate & 26 \\
music & 5 & concatenate & 20 \\
trades & 4 & index a list & 19 \\
crafts & 4 & count letters & 14 \\
astronomy & 3 & interleave & 11 \\
games & 3 & sort or order & 10 \\
textiles & 2 & first or last letter & 6 \\
liturgy & 2 &  &  \\
math & 2 &  &  \\
optics & 2 &  &  \\
poetry & 2 &  &  \\
geology & 2 &  &  \\
anatomy & 1 &  &  \\
geometry & 1 &  &  \\
forestry & 1 &  &  \\
chronology & 1 &  &  \\
drama & 1 &  &  \\
classics & 1 &  &  \\
telecom & 1 &  &  \\
calendar & 1 &  &  \\
mathematics & 1 &  &  \\
verse & 1 &  &  \\
history & 1 &  &  \\
nature & 1 &  &  \\
speeches & 1 &  &  \\
mythology & 1 &  &  \\
novels & 1 &  &  \\
engineering & 1 &  &  \\
civics & 1 &  &  \\
\midrule
35 domains & 107 & & \\
\bottomrule
\end{tabular}\end{table}

The next, massive \Cref{tab:items-all} includes the 107 valid items,
the id, their original $r^0$ and $t^0$, theme, first terms, and seed. We see a diversity in structure, content and mechanisms. Of course, some are small variations, such as item 0100 being a simple enumeration, or the theme being mere decoration, such as 0126. But a sequence such as ``o-1, o-2, ...'' may represent a distractor for very weak models, as compared to the trivial ``1, 2, ...", so a number of items of this kind are beneficial.

\begingroup\footnotesize\setlength{\tabcolsep}{3pt}
\setlength{\LTcapwidth}{\textwidth}
\begin{longtable}{@{}l c P{0.155\textwidth} P{0.295\textwidth} P{0.385\textwidth}@{}}
\caption{The 107 valid items. $r^{0}$ and $t^{0}$ are the retrieval and
transformation steps the cell asked for, not the measured ones. A value of $-1$ means the cell set no target for that component. The theme
was a prompt to the composer and does not always describe the result. We show the first few terms as an examinee would see them, and the seed is the rule
as the composer wrote it, both shortened to fit.}\label{tab:items-all}\\
\toprule id & $r^0 t^0$ & theme & first terms & seed \\ \midrule \endfirsthead
\multicolumn{5}{@{}l}{\footnotesize\itshape Table \thetable, continued}\\
\toprule id & $r^0 t^0$ & theme & first terms & seed \\ \midrule \endhead
\midrule \multicolumn{5}{r@{}}{\footnotesize continued} \\ \endfoot
\bottomrule \endlastfoot
\texttt{0002} & 0\,1 & embroidery thread strand counts & 0 skeins 1 strands, 0 skeins 2 strands, 0 skeins 3 strands, 0 skeins 4 strands, 0 skeins 5 strands, 1 skeins... & term n from 1 is n divided by 6 as 'a skeins b strands' \\
\texttt{0003} & 1\,4 & bovine milk fat percentages & h pct fat, he pct fat, li pct fat, be pct fat, c pct fat, o pct fat & A ledger starts with fat percentages [1, 3]. Pop the lowest, emit its element symbol and ' pct fat'. Add double the percentage modulo 118 as a new debt, filter out... \\
\texttt{0004} & 2\,5 & three little pigs building materials & rtsy11, itsy24, irby38, rtsy51, itsy63, irby76 & Start total at 0. For step n, take the n-th pig material (straw, sticks, brick) and n-th day (monday to sunday). Add their combined lengths to total. Each term is t... \\
\texttt{0005} & 1\,2 & baseball field position numbers & pitcher, catcher, first base, second base, third base, shortstop & count from 1 skipping numbers with 0; replace each digit with its baseball field position \\
\texttt{0006} & 0\,1 & islamic daily prayer sequence & 2, 6, 10, 13, 17, 19 & running total of raka'at for the daily prayers: 2, 4, 4, 3, 4 repeating \\
\texttt{0007} & 0\,2 & airport runway magnetic headings & 02, 04, 06, 08, 10, 12 & from n=1, add n to its digit sum, take modulo 36 (0 as 36), formatted as 2 digits \\
\texttt{0009} & 2\,2 & finger names on hand & bmuht alf, xedni bra, elddim cha, gnir del, yknip ech, bmuht fox & Pair fingers thumb to pinky with NATO alphabet; reverse the finger name and truncate the NATO word to 3 letters. \\
\texttt{0010} & 1\,2 & decathlon event competition schedule & 100 metres, long jump, shot put, high jump, 400 metres, 110 metres hurdles & queue of the decathlon events in order. emit front event; if day 1 (first 5) requeue once, else twice. \\
\texttt{0011} & 1\,0 & crop rotation plot years & hydrogen, helium, lithium, beryllium, boron, carbon & the chemical elements in order of atomic number \\
\texttt{0012} & 1\,2 & american wire gauge size numbers & 4, 7, 9, 10, 11, 13 & american wire gauge sizes in order cycling from 4/0 to 40; each term is the running total of their digit sums \\
\texttt{0014} & 0\,4 & circle of fifths keynotes & g1, e1, c1, a1, f1, d1 & Queue of notes c,d,e,f,g,a,b and k=1. Split queue into ceil(len/2) and rest, swap halves, pop front note: emit with k, shift letter forward by 4 notes, append to ba... \\
\texttt{0016} & 0\,0 & espresso roast level darkness & agtron 99, agtron 98, agtron 97, agtron 96, agtron 95, agtron 94 & count from 1; each term is agtron followed by a space and 100 minus n \\
\texttt{0020} & 2\,1 & canonical hours of prayer & matins h, lauds eh, prime il, terce eb, sext b, nones c & canonical hours of prayer in order, each followed by a space and the reversed element symbol by atomic number \\
\texttt{0021} & 0\,5 & standard egg weight gradings & p, s, s p, m, m p, m s & Egg grades p,s,m,l,x,j. Start empty. Each step: 1. Push p to track A. 2. Merge adjacent equal wagons (p+p-\texttt{>}s, s+s-\texttt{>}m, m+m-\texttt{>}l, l+l-\texttt{>}x, x+x-\texttt{>}j). 3. Shunt heaviest to... \\
\texttt{0023} & 0\,0 & natural counting numbers & 1, 2, 3, 4, 5, 6 & count from 1 \\
\texttt{0024} & 1\,3 & stellar spectral classification\linebreak[1]\hspace{0pt}{} letters & b, g, a, k, f, o & from n=1; square n, reverse digits, take remainder mod 7, and take that 0-based item from stellar spectral types \\
\texttt{0025} & 0\,6 & us constitutional\linebreak[1]\hspace{0pt}{} amendment numbers & amendment 2, amendment 5, amendment 10, amendment 8, amendment 8, amendment 10 & count n from 1; square n, sum its digits, reverse the digits, take remainder modulo 27, add 1, prepend 'amendment ' \\
\texttt{0026} & 0\,4 & candy cooking temperature stages & 0-0-1, 0-0-0, 0-0-0, 0-1-0, 1-2-0, 4-2-0 & Vessels Thread (cap $\infty$), Ball (5), Crack (3) start at 0. At step n: 1) Add n to Thread; 2) Fill Ball from Thread; 3) Fill Crack from Ball, then empty Crack if full;... \\
\texttt{0027} & 1\,1 & circle of fifths keys & c 1, f 1, bb 1, eb 1, ab 1, db 1 & circle of fifths major keys from c using flats, repeated; each followed by a space and cycle number \\
\texttt{0029} & 0\,6 & circle of fifths key signatures & c, g, d, a, e, b & From n=0, fifths step 4n mod 7 gives letter from cdefgab, appended with floor(n/7) plus signs + \\
\texttt{0030} & 0\,3 & crochet hook letter designations & b hook, e hook, j hook, h hook, h hook, j hook & count n from 1; square n, sum its digits, and emit that letter of the alphabet starting at b with ' hook' \\
\texttt{0031} & 0\,3 & plimsoll line draft marks & 10m w, 12m tf, 16m w, 24m tf, 31m f, 36m tf & draft starts at 10; next is draft + digit sum + 1. Term is draft, 'm ', and mark from tf, f, t, s, w, wna by draft mod 6 \\
\texttt{0032} & 0\,1 & crochet stitch height order & 12345, 12360, 12372, 12387, 12408, 12423 & a(1) = 12345; a(n) = a(n-1) + the digit sum of a(n-1) \\
\texttt{0034} & 0\,5 & rainbow light spectrum colors & orange 1, orange 8, violet 27, red 28, red 35, indigo 54 & From n=1: square n, sum its digits, multiply by n; emit the spectrum color (0=red, 1=orange, 2=yellow, 3=green, 4=blue, 5=indigo, 6=violet) for product mod 7 and th... \\
\texttt{0041} & 0\,6 & polygon side counts & 3, 4-3, 5-3-4, 6-4-3, 7-3-4, 8-4-3 & Main track and siding start empty. At step n (n=1,2...), wagon with polygon side count n+2 arrives: 1. Odd wagon goes to siding, even to main track. 2. Sort siding... \\
\texttt{0042} & 1\,4 & tree ring annual counts & 1 mm, 4 mm, 8 mm, 10 mm, 12 mm, 15 mm & a tree trunk starts at 0 mm; each year add (letter count of the nth chemical element mod 4 + 1) mm; emit total with ' mm' \\
\texttt{0043} & 2\,4 & planets distance from sun & mercury delta, venus echo, earth foxtrot, mercury golf, venus hotel, earth india & mercury, venus, earth start at track cells 1, 2, 3. Each step, advance the lowest token by its planet rank to the next empty cell, emitting planet and cell's NATO w... \\
\texttt{0045} & 1\,0 & ten in the bed countdown & duke, marquess, marquess, earl, earl, earl & british peerage ranks in descending order with growing run lengths \\
\texttt{0046} & 1\,1 & leather thickness in ounces & hydrogen 1 oz, helium 2 oz, lithium 3 oz, beryllium 4 oz, boron 5 oz, carbon 6 oz & the chemical elements in order, each followed by a space and its atomic number in oz \\
\texttt{0047} & 0\,4 & standard poker hand ranks & 1 one pair, 2 straight, 3 royal flush, 4 four of a kind, 5 four of a kind, 6 royal flush & from n=1: square n, sum its digits, take remainder mod 10, add 1, and emit n with that poker hand rank \\
\texttt{0049} & 1\,3 & months of the year & yra1, yra2, hcr3, lir4, yam5, enu6 & from n=1; take the nth month name in order, reverse it, keep the first 3 letters, and append n \\
\texttt{0050} & 2\,0 & british peerage rank order & duke, red, marquess, orange, earl, yellow & alternate between british peerage ranks in descending order and colors of the rainbow in order \\
\texttt{0054} & 1\,5 & united kingdom peerage rank order & kx1, sxs2, l3, nt4, xrn5, x6 & For step n, cycle UK peerage ranks: reverse, replace vowels with x, slice first (n\%3+1) chars, rotate left 1, append n. \\
\texttt{0058} & 2\,3 & us presidential election years & bbdf, qodf, jodf, ssbf, csbf, bsbf & us presidential election years in order; substitute digits using 'republican' (0=r, 1=e, 2=p, 3=u, 4=b, 5=l, 6=i, 7=c, 8=a, 9=n), reverse the string, and shift each... \\
\texttt{0061} & 1\,6 & rebar diameter designation numbers & 1 00100, 2 01110, 3 11001, 4 00111, 5 11100, 6 10011 & Rule-30 cellular automaton on a 5-cell ring from 00100; each term is term number and binary state. \\
\texttt{0064} & 1\,6 & norse myth nine realms & 1:1\texttt{|}, 2:\texttt{|}21, 3:3\texttt{|}21, 4:\texttt{|}2143, 5:5\texttt{|}2143, 6:21\texttt{|}4365 & A,B="","". s=str(1+(n-1)\%\linebreak[1]\hspace{0pt}{}6). Even len(A+B): A+=s; odd: B+=s+A[:1]; A=A[1:]. While len(B)\texttt{>}4+n//10\linebreak[1]\hspace{0pt}{}: A+=B[0]; B=B[1:]. A=A[-5:]. Output n:A\texttt{|}B \\
\texttt{0066} & 2\,1 & first verse of jabberwocky & brillig7, aries5, slithy6, taurus6, toves5, gemini6 & alternate the nonsense words of Jabberwocky's first verse with the zodiac signs from aries; each term is the word with its letter count appended \\
\texttt{0067} & 2\,5 & the opening of hamlet & bernardo\texttt{>}franc\linebreak[1]\hspace{0pt}{}isco:6, horatio\texttt{>}marcel\linebreak[1]\hspace{0pt}{}lus:4, bernardo\texttt{>}ghost\linebreak[1]\hspace{0pt}{}:7, francisco\texttt{>}hora\linebreak[1]\hspace{0pt}{}tio:5, bernardo\texttt{>}marce\linebreak[1]\hspace{0pt}{}llus:8, franc... & 5 Hamlet Act 1 Scene 1 characters in appearance order start in a queue with initial scores 5 to 1 from 5 bridge suits in descending order. Each step, dequeue the fr... \\
\texttt{0072} & 1\,0 & united states court hierarchy & jay, rutledge, cushing, wilson, blair, iredell & the us supreme court justices by surname in order of appointment \\
\texttt{0073} & 1\,2 & beer bitterness ibu ratings & 6 ibu, 12 ibu, 18 ibu, 24 ibu, 30 ibu, 36 ibu & the chemical elements in order; each term is 5 times the atomic number plus its digit sum, followed by ' ibu' \\
\texttt{0074} & 2\,1 & soil horizon layer letters & o-h, a-he, e-li, b-be, c-b, r-c & cycling master soil horizon letters (o, a, e, b, c, r) joined by hyphen to chemical element symbol by atomic number \\
\texttt{0076} & 2\,5 & three billy goats gruff sizes & red march, orange march, yellow june, green april, blue june, indigo april & 9-cell ring ooooxoooo: step Rule 30 CA. Term k emits k-th rainbow color and lit-count-th month. \\
\texttt{0077} & 1\,3 & prologue to canterbury tales & or1, wor2, r3, wr4, wyor5, or6 & 4 colors (W=1\texttt{<}Y=2\texttt{<}O=3\texttt{<}R\linebreak[1]\hspace{0pt}{}=4) start left; boat moves 2 coolest right, 1 coolest left; reset when empty; left+index \\
\texttt{0082} & 0\,7 & sumo wrestling division ranks & **\#\#, *\#\#\#, **\#\#, *\#\#\#, *\#\#\#, \#\#\#\# & for step n (from 1): 1. select sumo division by (n-1)\%6 + 1; 2. append n; 3. reverse string; 4. replace vowels with 1-based alphabet positions; 5. sum all numeric d... \\
\texttt{0083} & 0\,7 & cardinal direction points & w -1 0, n -1 1, w -2 1, e -1 1, e 0 1, w -1 1 & Start at (0,0). For n=1,2...: 1. square n, 2. sum digits, 3. multiply by 3, 4. mod 4, 5. map 0:n, 1:e, 2:s, 3:w, 6. step 1 unit to (x,y), 7. emit '\texttt{<}dir\texttt{>} \texttt{<}x\texttt{>} \texttt{<}y\texttt{>}'. \\
\texttt{0084} & 1\,7 & snooker ball point values & red 1, yellow 9, green 21, brown 71, blue 81, pink 21 & For term n from 1, take the n-th snooker ball in point-value order (1: red, 2: yellow, 3: green, 4: brown, 5: blue, 6: pink, 7: black, cycling every 7). Multiply n... \\
\texttt{0086} & 1\,7 & thirty two compass points & wwww-2, tttttt-4, wwwwww-6, tttt-8, --10, tttt-12 & Cycle wine table. Add vowels to running sum (start 0). Repeat 1st letter sum\textasciicircum{}2 mod 10 times (- if 0), then -sum. \\
\texttt{0089} & 1\,7 & coffee roast level gradations & ...*..., ..*...., ..*...., ....*.., *......, .*..... & for term n from 1: lookup roast level at n mod 8 (1: green, 2: cinnamon, 3: new england, 4: american, 5: city, 6: full city, 7: vienna, 8: french); remove vowels; c... \\
\texttt{0090} & 0\,7 & primal beef cut locations & *chuck, obrisket, .shank, \#rib, *plate, \#loin & from n=1, select the n-th cut from [chuck, brisket, shank, rib, plate, loin, flank, round] cycling every 8; take its length; square it; add n; sum the digits; take... \\
\texttt{0091} & 1\,7 & marathon kilometer milestone markers & km1 \texttt{>}, km2 v, km3 \texttt{>}, km4 \textasciicircum{}, km5 \texttt{>}, km6 \texttt{>} & For kilometer index n from 1, take the n-th chemical element by atomic number; count the letters in its name; add n; square the result; sum the digits of the square... \\
\texttt{0095} & 0\,0 & olympic medal rank order & 1 gold, 2 silver, 3 bronze, 4 gold, 5 silver, 6 bronze & count n from 1; each term is n, a space, and gold, silver, bronze in repeating medal rank order \\
\texttt{0096} & 0\,0 & solomon grundy weekday events & born 1, christened 2, married 3, took ill 4, grew worse 5, died 6 & solomon grundy events in order repeating, each with day count n \\
\texttt{0097} & 0\,0 & plywood veneer quality grades & a/a/a, a/a/b, a/a/c, a/a/d, a/b/a, a/b/b & 3-ply veneer grades using letters a to d in alphabetical order, slash-separate\linebreak[1]\hspace{0pt}{}d \\
\texttt{0098} & 1\,0 & mohs hardness scale benchmark minerals & talc a, gypsum b, calcite c, fluorite d, apatite e, orthoclase f & the Mohs hardness scale benchmark minerals in order repeating, each paired with the alphabet in order \\
\texttt{0099} & 0\,-1 & the chemical elements by atomic number, as symbols, sta... & na, mg, al, si, p, s & the chemical element symbols in lowercase by atomic number, starting from sodium \\
\texttt{0100} & -1\,-1 & the years, counting up from some year in the last centu... & 1900, 1901, 1902, 1903, 1904, 1905 & the years in order, counting up from 1900 \\
\texttt{0101} & 0\,-1 & the letters of a very well known passage or title, one... & t, o, b, e, o, r & the letters of 'to be or not to be that is the question' in order \\
\texttt{0102} & 2\,9 & international maritime signal flags & [dn*8], [hc*4], [r*n9], [**s5], [l*k10], [l*g6] & For term n starting at 1: multiply n by 5; add 3; sum the digits; if the digit sum is even select from 26 ICS flag names (alfa to zulu, 0-indexed) at digit sum mod... \\
\texttt{0103} & 1\,7 & wheels on the bus actions & wednesday beep, wednesday beep, tuesday swish, friday move, sunday shh, tuesday swish & s starts at 1. For step n from 1: square s, add n, sum the digits, reverse those digits, take remainder mod 7 to get index i. Retrieve the day at index i from the d... \\
\texttt{0104} & 2\,8 & climbing route difficulty grades & f61, f63, f06, f08, h501, f09 & At step n (1-indexed), retrieve the n-th French grade (cycling 15 items: 5a, 5b, 5c, 6a, 6a+, 6b, 6b+, 6c, 6c+, 7a, 7a+, 7b, 7b+, 7c, 7c+) and n-th YDS grade (cycli... \\
\texttt{0105} & 0\,7 & geological epoch chronological order & hol, mio, ple, hol, ple, mio & From n=1: square n, add n, sum digits, multiply by 7, reverse digits, take modulo 7, map 0..6 to pal, eoc, oli, mio, pli, ple, hol. \\
\texttt{0106} & 3\,9 & olympic podium place medals & \#\#, \#.., \#.., ., ., \#. & At step n (1-indexed), cycle through Olympic medals (gold, silver, bronze), motto (citius, altius, fortius), and ring colors (blue, yellow, black, green, red) to ge... \\
\texttt{0107} & 3\,9 & hungry caterpillar fruit sequence & [\texttt{>}], [\texttt{<}], [\textasciicircum{}], [\textasciicircum{}], [\texttt{<}], [\texttt{<}] & For step n from 1, take nth season (spring..winte\linebreak[1]\hspace{0pt}{}r), month (january..dece\linebreak[1]\hspace{0pt}{}mber), zodiac (aries..pisces\linebreak[1]\hspace{0pt}{}) mod table lengths. Concatenate lowercase, drop vowels a,e,i,o... \\
\texttt{0108} & 2\,9 & usda soil texture classes & beer-14 ibu, beer-14 ibu, beer-16 ibu, beer-13 ibu, beer-12 ibu, beer-19 ibu & Step n: 0-based floor((n-1)/2)\linebreak[1]\hspace{0pt}{}\%12 in T1(Jan-Dec) if n odd else T2(Aries-Pisce\linebreak[1]\hspace{0pt}{}s); drop aeiou; count chars; *n; sum digits; \%50; +10; prepend 'beer-'; append ' ibu'. \\
\texttt{0109} & -1\,8 & seasons of the year & 0-1-0-0, 0-1-0-3, 0-1-0-54, 0-2-0-54, 0-2-0-105, 0-2-0-108 & four seasonal counters (spring, summer, autumn, winter) start at 0; for step n from 1: square n, add n, sum digits, multiply by 5, reverse digits, mod 4 picks seaso... \\
\texttt{0110} & 2\,2 & fibonacci sequence term values & 1 lph, 1 bt, 2 gmm, 3 dlt, 5 psln, 8 zt & for step n, join the n-th fibonacci number mod 10, a space, and the n-th greek letter with vowels dropped \\
\texttt{0111} & 3\,1 & this old man rhyming targets & thumb-do-mercu\linebreak[1]\hspace{0pt}{}ry, shoe-re-venus,\linebreak[1]\hspace{0pt}{} knee-mi-earth,\linebreak[1]\hspace{0pt}{} door-fa-mars, hive-sol-jupit\linebreak[1]\hspace{0pt}{}er, sticks-la-satu\linebreak[1]\hspace{0pt}{}rn & join with hyphens the cycling terms of: This Old Man targets (thumb to home), solfege (do to ti), and planets outward \\
\texttt{0112} & 3\,2 & chess move ply counts & 15 plies red rook, 36 plies orange knight, 54 plies yellow bishop, 180 plies green queen, 0 plies blue king,... & For term n from 1, take the nth back-rank chess piece, its point value, and the nth rainbow color. Multiply the value by the color length and n, emitting '\texttt{<}product\texttt{>}... \\
\texttt{0113} & 1\,6 & rainbow color spectrum order & 1, 22, 63, 44, 51, 63 & For term n (from 1), take the n-th color from the rainbow colors in spectrum order (red, orange, yellow, green, blue, indigo, violet, cycling every 7): 1. count its... \\
\texttt{0114} & 1\,3 & glass bead millimeter diameters & 3mm, 5mm, 7mm, 10mm, 13mm, 16mm & from n=1, take the nth standard round glass bead diameter in millimeters cycling through 2, 3, 4, 6, 8, 10, 12, 14, 16, 18, 20; add n; take remainder modulo 19; app... \\
\texttt{0115} & 1\,4 & electrical wire gauge numbers & 160000 ga, 360000 ga, 360000 ga, 160000 ga, 25 ga, 144 ga & for step n from 1, take the nth gauge from standard AWG wire gauges in slash notation (4/0, 3/0, 2/0, 1/0, 1, 2...); replace '/' with '0' to get an integer; multipl... \\
\texttt{0116} & 2\,7 & usda beef quality grades & pc-04, cr-25, sb-61, sr-96, cf-01, us-00 & S starts at 0. For step n=1,2...: 1. multiply n-th USDA quality grade position (1-8) by (3n-2)-th beef primal cut position (1-8); 2. add n; 3. sum digits; 4. add to... \\
\texttt{0117} & 1\,6 & nautical knot speed increments & o-o-o-o, none, o-o, o-o-o-o-o, o-o, o-o-o-o-o-o & For step n from 1, take the n-th zodiac sign (cycling 12). Count its consonants, add n, multiply by 3, floor divide by 2, take mod 7, and print that many 'o's joine... \\
\texttt{0118} & 1\,2 & baking flour extraction percentages & 46\%, 56\%, 66\%, 81\%, 111\%, 151\% & the standard French flour types in ascending order, repeating; each term adds the 1-based cycle count to the type number and appends '\%' \\
\texttt{0119} & 3\,4 & us felony class levels & \#, \#\#, \#\#\#, \#\#\#\#\#\#\#\#, \#, \#\#\# & Alternate in order through three tables (cycling each): US federal felony classes, US federal misdemeanor classes, and US federal prison security levels. At term n... \\
\texttt{0120} & 3\,4 & meat doneness internal temperatures & a s r*d 5, b w *r*ng* 1, c n y*ll*w 7, d e gr**n 5, e s bl** 4, a w *nd*g* 8 & Cycle A..E (ranks 5..1), N,E,S,W (opposites S,W,N,E), 7 rainbow colors. Stock=10. Step: stock=(stock+9\linebreak[1]\hspace{0pt}{}-rank)\%9. Emit class, opposite dir, color (vowels *), stock. \\
\texttt{0121} & 3\,4 & bovine permanent teeth emergence & 1\#\#, 1\#\#\#, 1\#\#\#\#, 2\#\#\#\#\#\#\#\#, 2\#\#\#\#\#\#\#\#\#, 2 & Concatenate tables T1, T2, T3 (len 30). For n=1..72: v=L[((ceil(n/3\linebreak[1]\hspace{0pt}{})-1) mod 30)+1]. Ignore non-digits in v. Append ((digit\_sum(v)\linebreak[1]\hspace{0pt}{}\textasciicircum{}2+n) mod 10) '\#' chars to v. \\
\texttt{0122} & 1\,2 & sheep shearing pattern steps & 1y, 1h, 1l, 1k, 1d, 1k & sheep shearing steps (belly, crutch, tail, neck, head, flank, side); term is pass number plus reversed step truncated to pass \\
\texttt{0123} & 3\,0 & michelin star tier levels & 1 star, sweet, monday, 2 stars, sour, tuesday & alternating between michelin star tiers, basic taste modalities, and days of the week in order \\
\texttt{0124} & 1\,-1 & the whole numbers, counting up in steps of the same siz... & delaware, pennsylvania, new jersey, georgia, connecticut, massachusetts & the us states in order of admission to the union \\
\texttt{0125} & 3\,0 & tractor power take off speeds & 540 rpm, type 1, mercury, 750 rpm, type 2, venus & alternate: pto speeds (540 rpm, 750 rpm, 1000 rpm), iso pto categories (type 1, type 2, type 3, type 4), planets outward \\
\texttt{0126} & -1\,0 & us military officer ranks & o-1, o-2, o-3, o-4, o-5, o-6 & the letter o followed by a dash and n from 1 \\
\texttt{0127} & 2\,0 & numbers from one to ten & one, triangle, two, quadrilateral,\linebreak[1]\hspace{0pt}{} three, pentagon & alternate between English numbers from one to ten and polygons with 3 to 10 sides in order \\
\texttt{0128} & 1\,-1 & the whole numbers, counting up from any start & genesis, exodus, leviticus, numbers, deuteronomy, joshua & the books of the protestant bible in order \\
\texttt{0129} & 3\,1 & opening stanzas of the raven & wea, mer, jan, lor, ven, feb & alternate in order through the line-ending words of stanza 1 of The Raven, the planets outward, and the months of the year, truncating each word to its first three... \\
\texttt{0130} & 1\,-1 & the letters of a very well known passage or title, one... & c, a, l, l, m, e & the letters of the opening of Moby Dick in order, one per term \\
\texttt{0131} & 3\,1 & archery target score rings & we, pn, as, bk, kt, ts & alternate through archery ring colors by score, chess pieces by value, and zodiac signs in order; emit first and last letters \\
\texttt{0132} & 3\,-1 & the chemical elements by atomic number, as names, start... & sodium, mercury, january, magnesium, venus, february & Alternate in order: element from Sodium, 8 planets Mercury-Neptun\linebreak[1]\hspace{0pt}{}e, month January-Decemb\linebreak[1]\hspace{0pt}{}er. \\
\texttt{0133} & -1\,-1 & the words of a very well known song or nursery rhyme, o... & 1 mary, 2 had, 3 a, 4 little, 5 lamb, 6 mary & step number and repeating words of 'mary had a little lamb' \\
\texttt{0134} & 3\,0 & tailor seam allowance widths & 1/8 in, mercury, pacific, 1/4 in, venus, atlantic & Cycle 3 lists: 7 seam allowances (1/8 in..1 in), 8 planets (Mercury..Nept\linebreak[1]\hspace{0pt}{}une), 5 oceans (Pacific..Arct\linebreak[1]\hspace{0pt}{}ic). \\
\texttt{0135} & 4\,2 & decathlon event order & 1, s, a, m, ll, aa & Cycle through decathlon events, summer Olympic host cities since 2000, Greek alphabet letters, and planets in order; each term takes the first character of the item... \\
\texttt{0136} & 3\,1 & snow white dwarf names & bas, red, mer, doc, ora, ven & Alternate in order through the seven dwarfs alphabetically\linebreak[1]\hspace{0pt}{}, rainbow colors in order, and solar system planets outward from the Sun; truncate each word to three lett... \\
\texttt{0137} & -1\,7 & espresso shot extraction times & 120s, 287s, 22s, 117s, 284s, 19s & v starts at 25; for term n from 1: add n to v, multiply by 7, sum digits, square, subtract n, take absolute value, set v to the result, emit v with suffix 's' \\
\texttt{0138} & 1\,7 & beekeeping frame super depths & [======], [======], [=======], [=======], [======], [=======] & For step n (1-indexed), take the n-th planet from the sun (cycling). Count its letters, square the count, add n, sum the digits, take mod 4, add 5, and render that... \\
\texttt{0139} & 0\,7 & goldilocks porridge temperatures & 69c: just right, 64c: just right, 64c: just right, 89c: too hot, 89c: too hot, 64c: just right & t starts at 70; at step n (from 1): 1) multiply t by n, 2) sum digits, 3) square, 4) take modulo 100, 5) add 20 to set new t, 6) classify t (\texttt{<}55: too cold, \texttt{>}85: too... \\
\texttt{0140} & 3\,2 & us state admission order & jhd, fhp, mln, abg, mbc, jcm & for term n from 1 to 50, take the nth US state in admission order, the nth chemical element by atomic number, and the nth month of the year (cycling); extract the i... \\
\texttt{0141} & 0\,7 & hokey pokey body parts & head, left foot, right hip, head, head, right hip & count n from 1; square n; add 7; sum digits; multiply by 5; reverse digits; take remainder mod 8 to index (0: right foot, 1: left foot, 2: right hand, 3: left hand,... \\
\texttt{0142} & 1\,7 & the four seasons order & gn*rps1, r*tn*w2, gn*rps3, mn*t**4, mn*t**5, gn*rps6 & At step n (1-indexed): square n, sum its digits, add 3, take mod 4 to pick the season (0=spring, 1=summer, 2=autumn, 3=winter), reverse the season name, replace vow... \\
\texttt{0143} & 4\,2 & the gettysburg address & uof, saw, led, naj, ocs, ada & Cycle in turn through four tables: gettysburg address words in order (no punctuation), us presidents by surname, us states by admission order (no spaces), and month... \\
\texttt{0144} & 3\,2 & jack and jill rhyme & jader, anegnaro, jiwolley, weneerg, upeulb, thogidni & Cycle 25 words of Jack and Jill stanza 1, 7 solfege notes, and 7 rainbow colors. Join rhyme word (truncated to note length) with reversed color. \\
\texttt{0145} & 2\,7 & twelve labors of hercules & \#\#\#\#\#\#\#\#\#, \#, \#, \#\#\#\#\#\#\#\#\#, \#\#\#\#\#\#\#\#\#, \#\#\#\#\#\#\#\#\#\#\#\#\#\#\linebreak[1]\hspace{0pt}{}\#\#\#\# & Step n: append lowercase 3-letter month ((n-1) mod 12 + 1) to lowercase element symbol n; reverse; pick 1-based even chars; sum letter values (a=1); square; sum dig... \\
\texttt{0146} & 4\,2 & opening of a christmas carol & namygrelc1, yelram2, egoorcs3, tnelis4, krelc5, tsap6 & Cycle 4 tables (clergyman,cle\linebreak[1]\hspace{0pt}{}rk,undertaker,\linebreak[1]\hspace{0pt}{}mourner; marley,past,pr\linebreak[1]\hspace{0pt}{}esent,future; scrooge,marley\linebreak[1]\hspace{0pt}{},fred,cratchit\linebreak[1]\hspace{0pt}{}; silent,jingle,\linebreak[1]\hspace{0pt}{}deck,joy) round-robin; reverse word, app... \\
\texttt{0147} & 2\,-1 & the letters of a very well known passage or title, one... & ..-., ---, ..-, .-., ..., -.-. & for each letter of the gettysburg address in order, its international morse code \\
\texttt{0148} & 1\,-1 & the whole numbers, counting up in steps of the same siz... & i, ii, iii, iv, v, vi & the Roman numerals in ascending order from 1 \\
\texttt{0149} & 0\,0 & united states pledge of allegiance & i, pledge, allegiance, to, the, flag & the words of the us pledge of allegiance in order, lowercase without punctuation \\
\texttt{0150} & 2\,0 & musical tempo marking names & grave, do, largo, re, adagio, mi & alternate between the 8 main tempo markings in order of speed and solfege syllables in scale order \\
\texttt{0151} & 0\,0 & the poem the raven & do1, re1, mi1, fa1, sol1, la1 & solfege syllables do to ti in order with an incrementing octave number \\
\end{longtable}\endgroup

Finally, \Cref{tab:items-machines} refers to the previous table through id, and for the three reference machines it shows V, U or C representing if the item has reached valid, unquestionable or certified status, and other columns with several complexity metrics,  the estimated $r$ and $t$, the split point for unquestionability ($L_U$), etc. We see that the shortest description found become cryptic sometimes, often expressed in pseudocode. This is on purpose, as this is the conditional version, and explains why ``{\tt British peerage and ROYGBIV colors}" works. It is not sufficiently unambiguous to generate the whole sequence from scratch but a sufficient hint for the continuation. This makes the difference between unconditional and conditional programs more clear.

We checked all the items externally as a postprocess, once all the results were processed, to avoid cherry-picking or introducing a manual step in the construction of the battery. This post-mortem check serves as validation, as we found very few errors. Namely, just a few items admit a simpler reading of their prefix. %
They also show up during evaluation, since the
examinees then agree on the same wrong continuation, and the ranking is not altered. Under Qwen3-14B there
are three of the 51 certified items (\texttt{0133}, \texttt{0110} and \texttt{0127}), under Qwen3-8B
three of 32, and under Phi-4-reasoning four of 38. In \texttt{0133} the examinees
continue the rhyme after ``Mary had a little lamb'', and in \texttt{0110} the
Fibonacci numbers are taken modulo 10 but the five terms shown are all below
10. All of them passed at a prefix of 5 or 13 terms. In addition, one item
certified only under Phi-4-reasoning, 0142, repeats an arithmetic slip of the
generator in 15 of its terms. We keep all of them in the battery. Removing
them raises every correlation between $\widehat{Kt}^{|}$ and accuracy, to
$-0.84$ under Qwen3-14B. Such items could be caught before evaluation by
taking the reference machine's own unaided continuation as a rival, and by
recomputing formal rules with a program.
It is important to highlight that these items lie between 40 and 146 bits of $Kt$ for Qwen3-14B, below the upper end of the scale, where the supra window is located. %
Also, the previous check could be automated without relying on the examinees, so increasing the robustness of the battery would only require some extra compute.

\begingroup\scriptsize\setlength{\tabcolsep}{2.2pt}
\setlength{\LTcapwidth}{\textwidth}
\makeatletter
\def\LT@makecaption#1#2#3{%
  \LT@mcol\LT@cols c{\hbox to\z@{\hss\parbox[t]\LTcapwidth{\normalsize
    \sbox\@tempboxa{#1{#2: }#3}%
    \ifdim\wd\@tempboxa>\hsize
      #1{#2: }#3%
    \else
      \hbox to\hsize{\hfil\box\@tempboxa\hfil}%
    \fi
    \endgraf\vskip\baselineskip}%
  \hss}}}
\makeatother
\begin{longtable}{@{}l ccc rrr rrr rr P{0.605\textwidth}@{}}
\caption{Each of the 107 valid items of the battery. The stage is the furthest each reached for each reference machine: V valid, U unquestionable, C certified. $K$ is $\widehat{K}^{\V}_{\M}(x)$, the cost of the least expensive description of the whole sequence; $\Lu$ the prefix at which no rival survives; $\Kco$ the cost of the description that certified, zero where the machine continued the sequence unaided. $\hat r$ and $\hat t$ are the measured steps under Qwen3-14B. The last column is the description that certified under Qwen3-14B, shortened.}\label{tab:items-machines}\\
\toprule
& \multicolumn{3}{c}{stage} & \multicolumn{3}{c}{$K$} & \multicolumn{3}{c}{$\Lu$} & & & \\
\cmidrule(lr){2-4}\cmidrule(lr){5-7}\cmidrule(lr){8-10}
id & 14B & 8B & Phi & 14B & 8B & Phi & 14B & 8B & Phi & $\hat r$ & $\hat t$ & $\Kco$, and the description \\
\midrule \endfirsthead
\multicolumn{13}{@{}l}{\scriptsize\itshape Table \thetable, continued}\\ \toprule
id & 14B & 8B & Phi & 14B & 8B & Phi & 14B & 8B & Phi & $\hat r$ & $\hat t$ & $\Kco$, and the shortest {\em conditional} description \\
\midrule \endhead
\midrule \multicolumn{13}{r@{}}{\scriptsize continued} \\ \endfoot
\bottomrule \endlastfoot
\texttt{0002} & C & C & C & 108 & 104 & 104 & 13 & 5 & 5 & 0.0 & 2.0 & 0\; (none needed) \\
\texttt{0003} & C & U & U & 272 & 271 & 234 & 60 & 60 & 60 & 1.0 & 5.0 & 187\; L = [1, 3]; loop: x = min(L); L.remove(x); L.append(2*x \% 118); output lowercase period... \\
\texttt{0004} & C & U & U & 234 & 252 & 261 & 60 & 60 & 60 & 2.0 & 5.0 & 225\; t=0; m=['straw','sticks','brick']; d=['monday','tuesday','wednesday','thursday','friday',... \\
\texttt{0005} & C & C & C & 114 & 122 & 109 & 13 & 13 & 42 & 1.0 & 1.0 & 0\; (none needed) \\
\texttt{0006} & C & C & C & 70 & 74 & 71 & 13 & 13 & 13 & 1.0 & 1.0 & 0\; (none needed) \\
\texttt{0007} & V & V & V & 116 & 116 & 107 & --- & --- & --- & 0.0 & 2.2 & --- \\
\texttt{0009} & C & U & U & 143 & 135 & 120 & 5 & 5 & 5 & 2.0 & 2.2 & 141\; nth term: reverse of (thumb, index, middle, ring, pinky)[(n-1)\%5] + ' ' + 3-letter prefix of ... \\
\texttt{0010} & V & U & U & 157 & 167 & 147 & --- & 35 & 35 & 1.0 & 2.2 & --- \\
\texttt{0011} & C & C & C & 28 & 28 & 33 & 5 & 5 & 5 & 1.0 & 0.0 & 0\; (none needed) \\
\texttt{0012} & V & U & U & 141 & 142 & 140 & --- & 27 & 27 & 1.0 & 3.0 & --- \\
\texttt{0014} & V & U & V & 113 & 126 & 129 & --- & 18 & --- & 0.0 & 2.2 & --- \\
\texttt{0016} & C & C & C & 78 & 87 & 84 & 5 & 5 & 5 & 1.0 & 1.0 & 0\; (none needed) \\
\texttt{0020} & U & U & U & 133 & 141 & 128 & 5 & 5 & 5 & 2.0 & 1.0 & --- \\
\texttt{0021} & V & V & V & 245 & 232 & 214 & --- & --- & --- & 0.0 & 3.2 & --- \\
\texttt{0023} & C & C & C & 22 & 17 & 20 & 5 & 5 & 5 & 0.0 & 0.0 & 0\; (none needed) \\
\texttt{0024} & V & U & C & 158 & 156 & 147 & --- & 27 & 27 & 1.0 & 4.0 & --- \\
\texttt{0025} & V & V & U & 161 & 169 & 146 & --- & --- & 23 & 0.0 & 4.4 & --- \\
\texttt{0026} & V & V & V & 166 & 167 & 164 & --- & --- & --- & 0.0 & 1.0 & --- \\
\texttt{0027} & V & V & V & 148 & 179 & 161 & --- & --- & --- & 1.0 & 2.2 & --- \\
\texttt{0029} & V & V & V & 131 & 138 & 135 & --- & --- & --- & 0.0 & 2.4 & --- \\
\texttt{0030} & C & C & U & 154 & 160 & 146 & 27 & 27 & 27 & 0.0 & 4.0 & 115\; n ? (1 + sum of digits of n2)th letter of the alphabet + ' hook' (n = 1, 2, ...) \\
\texttt{0031} & C & C & U & 230 & 225 & 217 & 36 & 36 & 36 & 0.0 & 3.0 & 169\; d1=10; dn+1 = dn + sum\_digits(dn) + 1; termn = dn + "m " + (tf, f, t, s, w, wna)[dn mod 6] \\
\texttt{0032} & C & C & C & 41 & 56 & 50 & 5 & 5 & 5 & 0.0 & 2.0 & 0\; (none needed) \\
\texttt{0034} & C & C & C & 193 & 204 & 182 & 13 & 13 & 13 & 1.0 & 4.4 & 128\; n=1; repeat: p = n * sum(digits of n\textasciicircum{}2); output ROYGBIV[p \% 7] + " " + p; n += 1 \\
\texttt{0041} & V & V & V & 134 & 126 & 127 & --- & --- & --- & 0.0 & 1.8 & --- \\
\texttt{0042} & U & V & U & 126 & 145 & 129 & 23 & --- & 23 & 1.0 & 3.8 & --- \\
\texttt{0043} & C & C & C & 126 & 135 & 131 & 13 & 5 & 35 & 1.0 & 3.0 & 0\; (none needed) \\
\texttt{0045} & U & U & C & 98 & 103 & 94 & 13 & 13 & 5 & 1.0 & 1.0 & --- \\
\texttt{0046} & C & C & C & 99 & 102 & 93 & 13 & 5 & 13 & 1.8 & 1.2 & 0\; (none needed) \\
\texttt{0047} & V & V & V & 176 & 174 & 171 & --- & --- & --- & 1.0 & 4.0 & --- \\
\texttt{0049} & C & C & C & 128 & 115 & 111 & 5 & 5 & 5 & 1.0 & 3.0 & 39\; Month names reversed \\
\texttt{0050} & C & C & C & 87 & 92 & 90 & 5 & 5 & 5 & 2.0 & 0.0 & 54\; British peerage and ROYGBIV colors \\
\texttt{0054} & C & U & U & 203 & 201 & 180 & 13 & 13 & 13 & 1.0 & 5.0 & 189\; term(n) = rotate\_left(1, vowel\_x(reverse(UK\_peerage\_rank(n)))[:n\%3+1]) + n \\
\texttt{0058} & C & U & U & 180 & 169 & 162 & 13 & 13 & 13 & 1.0 & 2.0 & 168\; US presidential election years in order \texttt{|} reverse digits \texttt{|} 0123456789 -\texttt{>} sfqvcmjdbo \\
\texttt{0061} & V & V & V & 151 & 138 & 140 & --- & --- & --- & 0.0 & 1.8 & --- \\
\texttt{0064} & V & V & V & 377 & 406 & 385 & --- & --- & --- & 0.0 & 6.2 & --- \\
\texttt{0066} & U & U & C & 163 & 173 & 163 & 5 & 13 & 13 & 2.0 & 1.8 & --- \\
\texttt{0067} & V & V & V & 277 & 290 & 316 & --- & --- & --- & 0.0 & 3.4 & --- \\
\texttt{0072} & U & U & U & 62 & 68 & 64 & 5 & 5 & 5 & 1.0 & 0.0 & --- \\
\texttt{0073} & V & V & C & 127 & 122 & 110 & --- & --- & 13 & 0.0 & 3.6 & --- \\
\texttt{0074} & C & C & C & 155 & 157 & 145 & 13 & 13 & 13 & 1.0 & 2.8 & 47\; soil horizon layer letters \\
\texttt{0076} & C & U & U & 203 & 210 & 209 & 50 & 50 & 50 & 2.0 & 2.2 & 177\; k -\texttt{>} (k-th rainbow color) + " " + (lowercase month number M[(k-1) mod 8]), where M = ... \\
\texttt{0077} & V & V & C & 147 & 145 & 138 & --- & --- & 13 & 0.0 & 2.0 & --- \\
\texttt{0082} & V & V & V & 284 & 298 & 287 & --- & --- & --- & 1.0 & 6.0 & --- \\
\texttt{0083} & C & U & U & 216 & 208 & 214 & 60 & 60 & 60 & 0.0 & 5.2 & 201\; x, y = 0, 0 for n = 1, 2, ... d = "nesw"[3 * sum\_digits(n\textasciicircum{}2) \% 4] x += (d=='e') - (d=='w') ... \\
\texttt{0084} & C & U & U & 303 & 311 & 288 & 60 & 60 & 60 & 1.0 & 5.2 & 290\; For term n from 1, take the n-th snooker ball in point-value order (1: red, 2: yellow, 3: green... \\
\texttt{0086} & V & U & C & 178 & 184 & 169 & --- & 13 & 13 & 0.0 & 3.4 & --- \\
\texttt{0089} & U & U & U & 163 & 168 & 180 & 38 & 38 & 38 & 0.0 & 3.8 & --- \\
\texttt{0090} & U & U & U & 297 & 289 & 310 & 60 & 60 & 60 & 0.0 & 7.0 & --- \\
\texttt{0091} & U & U & C & 201 & 213 & 211 & 50 & 50 & 50 & 1.0 & 5.2 & --- \\
\texttt{0095} & V & V & V & 104 & 98 & 120 & --- & --- & --- & 0.0 & 1.8 & --- \\
\texttt{0096} & C & C & C & 125 & 127 & 132 & 13 & 13 & 13 & 1.0 & 1.8 & 0\; (none needed) \\
\texttt{0097} & C & C & C & 75 & 88 & 88 & 5 & 13 & 5 & 0.0 & 0.8 & 0\; (none needed) \\
\texttt{0098} & C & U & C & 105 & 114 & 106 & 13 & 13 & 5 & 1.0 & 1.0 & 91\; nth term: (nth Mohs scale benchmark mineral) + " " + (nth letter of the alphabet) \\
\texttt{0099} & C & C & C & 84 & 78 & 62 & 5 & 5 & 5 & 1.0 & 1.0 & 0\; (none needed) \\
\texttt{0100} & C & C & C & 30 & 37 & 30 & 5 & 5 & 5 & 0.0 & 0.0 & 0\; (none needed) \\
\texttt{0101} & C & C & C & 44 & 50 & 46 & 5 & 13 & 5 & 0.0 & 1.0 & 37\; Hamlet soliloquy \\
\texttt{0102} & U & U & U & 192 & 191 & 169 & 13 & 13 & 5 & 0.0 & 4.2 & --- \\
\texttt{0103} & U & U & U & 215 & 228 & 194 & 13 & 13 & 13 & 0.0 & 6.6 & --- \\
\texttt{0104} & U & U & U & 335 & 354 & 324 & 60 & 60 & 60 & 1.0 & 5.8 & --- \\
\texttt{0105} & V & V & V & 215 & 211 & 217 & --- & --- & --- & 0.0 & 3.0 & --- \\
\texttt{0106} & V & V & V & 201 & 198 & 194 & --- & --- & --- & 0.0 & 3.2 & --- \\
\texttt{0107} & V & V & V & 311 & 322 & 302 & --- & --- & --- & 3.0 & 7.6 & --- \\
\texttt{0108} & C & U & U & 320 & 304 & 298 & 59 & 59 & 59 & 2.0 & 2.0 & 290\; 1. Set k = floor((n-1)/2) mod 12. 2. Pick k-th month of the year if n is odd, else k-th zod... \\
\texttt{0109} & C & U & U & 252 & 245 & 235 & 40 & 40 & 40 & 0.0 & 5.0 & 180\; c = [0, 0, 0, 0]; for n = 1, 2, ...: r = rev(5 * digit\_sum(n\textasciicircum{}2 + n)); c[r mod 4] += r; output c[0]-... \\
\texttt{0110} & C & U & U & 117 & 119 & 104 & 5 & 5 & 5 & 1.2 & 2.8 & 118\; f(n) = (F\_n mod 10) + ' ' + (n-th Greek letter minus vowels) \\
\texttt{0111} & C & C & C & 232 & 231 & 230 & 13 & 13 & 13 & 3.0 & 2.0 & 92\; Cycles: 10 number rhymes, 7 solfege, 8 planets \\
\texttt{0112} & C & U & U & 236 & 231 & 226 & 13 & 13 & 13 & 3.0 & 2.2 & 230\; Each term n (starting at 1) is n $\times$ len(c) $\times$ v "plies" c p, where c is the nth rainbow color, p is... \\
\texttt{0113} & C & U & U & 144 & 138 & 143 & 13 & 13 & 13 & 1.0 & 6.0 & 127\; n -\texttt{>} reverse of (n * digit\_sum((length of n-th rainbow color)\textasciicircum{}2 + n)) \\
\texttt{0114} & V & V & V & 196 & 188 & 195 & --- & --- & --- & 1.0 & 3.6 & --- \\
\texttt{0115} & C & C & C & 144 & 160 & 145 & 13 & 13 & 13 & 0.0 & 3.4 & 119\; ((n \texttt{<} 5 ? 100*(5-n) : n-4) * n)\textasciicircum{}2 + " ga" for n = 1, 2, ... \\
\texttt{0116} & V & V & V & 402 & 412 & 406 & --- & --- & --- & 2.0 & 8.4 & --- \\
\texttt{0117} & V & V & V & 256 & 264 & 250 & --- & --- & --- & 1.0 & 3.4 & --- \\
\texttt{0118} & V & V & V & 168 & 170 & 154 & --- & --- & --- & 1.0 & 3.4 & --- \\
\texttt{0119} & V & V & V & 317 & 334 & 276 & --- & --- & --- & 3.0 & 3.0 & --- \\
\texttt{0120} & U & U & U & 288 & 284 & 271 & 13 & 13 & 13 & 1.0 & 4.0 & --- \\
\texttt{0121} & V & C & V & 171 & 166 & 170 & --- & 22 & --- & 0.0 & 3.2 & --- \\
\texttt{0122} & V & V & V & 233 & 249 & 225 & --- & --- & --- & 0.0 & 4.0 & --- \\
\texttt{0123} & C & C & U & 132 & 128 & 140 & 13 & 13 & 47 & 3.0 & 1.0 & 32\; 5 basic tastes \\
\texttt{0124} & C & C & C & 34 & 47 & 51 & 5 & 5 & 5 & 1.0 & 0.0 & 36\; US states in order of admission to the Union, in lowercase \\
\texttt{0125} & V & V & V & 173 & 190 & 179 & --- & --- & --- & 1.0 & 1.0 & --- \\
\texttt{0126} & C & C & C & 42 & 45 & 43 & 5 & 5 & 5 & 0.0 & 1.0 & 55\; US officer pay grades \\
\texttt{0127} & C & U & C & 90 & 94 & 100 & 13 & 13 & 39 & 1.0 & 0.0 & 84\; Alternate endlessly: English numbers 1 to 10 and polygon names for 3 to 10 sides. \\
\texttt{0128} & C & C & C & 45 & 42 & 52 & 5 & 5 & 5 & 1.0 & 0.0 & 0\; (none needed) \\
\texttt{0129} & C & U & U & 186 & 194 & 186 & 13 & 13 & 13 & 3.0 & 1.0 & 147\; Interleave first 3 letters of: line-ending words of stanza 1 of The Raven, planets outward from... \\
\texttt{0130} & U & C & C & 51 & 46 & 47 & 5 & 5 & 5 & 1.0 & 1.0 & --- \\
\texttt{0131} & U & U & U & 148 & 141 & 132 & 13 & 13 & 13 & 3.0 & 1.0 & --- \\
\texttt{0132} & C & U & C & 134 & 126 & 133 & 13 & 13 & 13 & 3.0 & 1.0 & 89\; Cycle through lowercase lists: chemical elements from sodium, planets from mercury, months... \\
\texttt{0133} & C & C & C & 81 & 83 & 80 & 5 & 5 & 5 & 1.0 & 1.0 & 77\; monotone counter paired with cyclic word list \\
\texttt{0134} & C & U & U & 152 & 145 & 146 & 13 & 13 & 13 & 2.0 & 1.0 & 149\; Round-robin emission from 3 lists: [1/8 in, 1/4 in, 3/8 in, 1/2 in, 5/8 in, 3/4 in, 1 in], 8 plan... \\
\texttt{0135} & V & V & V & 276 & 281 & 247 & --- & --- & --- & 4.0 & 2.0 & --- \\
\texttt{0136} & C & U & U & 152 & 163 & 157 & 13 & 13 & 13 & 3.0 & 1.2 & 135\; Cycle in order through 7 dwarfs alphabetically, ROYGBIV rainbow colors, and planets from... \\
\texttt{0137} & V & V & V & 217 & 238 & 243 & --- & --- & --- & 0.0 & 7.0 & --- \\
\texttt{0138} & V & V & U & 211 & 216 & 198 & --- & --- & 43 & 1.0 & 5.4 & --- \\
\texttt{0139} & C & U & U & 263 & 265 & 257 & 50 & 50 & 50 & 0.0 & 5.8 & 243\; t:=70. At step n=1,2,...: t := 20 + (digit\_sum(n*t)\textasciicircum{}2 \% 100). Emit `\$\{t\}c: \$\{t\texttt{<}55 ? 'too cold'... \\
\texttt{0140} & U & U & U & 150 & 151 & 157 & 13 & 13 & 13 & 3.0 & 2.8 & --- \\
\texttt{0141} & C & U & U & 225 & 236 & 210 & 55 & 55 & 55 & 0.0 & 5.2 & 160\; n ? \{1, 2, 3, ...\} ? (right foot, left foot, right hand, left hand, head, right hip, left hip, whole... \\
\texttt{0142} & V & V & C & 238 & 241 & 208 & --- & --- & 13 & 0.0 & 7.6 & --- \\
\texttt{0143} & U & U & U & 205 & 230 & 208 & 13 & 13 & 13 & 4.0 & 3.2 & --- \\
\texttt{0144} & U & U & U & 248 & 254 & 248 & 13 & 13 & 13 & 3.0 & 4.2 & --- \\
\texttt{0145} & U & U & U & 307 & 299 & 292 & 60 & 60 & 60 & 2.0 & 8.0 & --- \\
\texttt{0146} & C & U & U & 364 & 391 & 358 & 13 & 13 & 13 & 0.0 & 3.0 & 367\; f(n) = reverse of W[(n-1)\%4][((n-1)/4)\%4] + n for n=1,2,... with W=[[clergyman,clerk... \\
\texttt{0147} & U & C & C & 64 & 58 & 61 & 5 & 5 & 5 & 2.0 & 0.8 & --- \\
\texttt{0148} & C & C & C & 51 & 49 & 51 & 5 & 5 & 5 & 1.0 & 0.0 & 0\; (none needed) \\
\texttt{0149} & C & C & U & 64 & 68 & 71 & 5 & 13 & 5 & 1.0 & 2.0 & 38\; natural text order \\
\texttt{0150} & U & U & U & 113 & 116 & 122 & 13 & 5 & 13 & 2.0 & 0.0 & --- \\
\texttt{0151} & V & V & V & 104 & 109 & 120 & --- & --- & --- & 0.0 & 2.2 & --- \\
\end{longtable}\endgroup

\section{Proof of \Cref{cor:compressibility-cond}}\label{app:proof}

Despite being trivial, we include the proof here for completeness:

\begin{proof}
Let $\beta L$ be the cost of the program $p^{\lambda}$ ``print $x_{1:L}$,
then $x_{L+1:L+k}$'', where $c_{0}$
is the fixed overhead of the print template and is independent of the intended
term. Let $z'\ne x_{L+1:L+k}$ be an alternative continuation. This gives a rival program $q$ of the same cost $\beta L + c_{0}$. If
$\beta L + c_{0} < \cost(p^{\star}) + \delta$, this program is a rival inside
the window that continues differently, contradicting $\delta$-unquestionability.
Hence $\beta L + c_{0} \geq \cost(p^{\star}) + \delta$, i.e.
$\cost(p^{\star}) \leq \beta L - \delta + c_{0}$.
\end{proof}

\section{Acknowledging Models, Data and Libraries Used}

We used the following models, data and libraries, and we acknowledge their use. 

The
reference machines were Qwen3-14B and Qwen3-8B~\citep{qwen3tech2025} and
Phi-4-reasoning~\citep{abdin2025phi4reasoning}, all loaded in 4-bit NF4
quantisation~\citep{dettmers2023qlora}. Gemini 3.6
Flash~\citep{google2026gemini36flash} was the composer, the verifier, the typer
and the scoring judge, and Gemini 3.5 Flash-Lite~\citep{google2026gemini35flashlite}
the budget judge. The examinees were GPT-5.6 Sol~\citep{openai2026gpt56},
GPT-5 nano~\citep{openai2025gpt5}, Claude Opus 5~\citep{anthropic2026opus5},
Claude Sonnet 5~\citep{anthropic2026sonnet5}, Claude Fable
5.1~\citep{anthropic2026fable51}, Claude Haiku 4.5~\citep{anthropic2025haiku45},
Gemini 3.8 Flash, 3.6 Flash, 3.5 Flash-Lite and 3.1
Flash-Lite~\citep{google2026gemini38flash,google2026gemini36flash,google2026gemini35flashlite,google2026gemini31flashlite},
OLMo-2-7B~\citep{olmo2024olmo2} and Qwen2.5-1.5B~\citep{qwen25tech2024}.
Gemma-3-4B~\citep{gemma2025gemma3}, Llama-3.1-8B~\citep{grattafiori2024llama3}
and OLMo-2-7B were also calibrated, as a check on the reference machines'
constants. 

Abilities were compared with Epoch AI's Capabilities
Index~\citep{epoch2025eci}, whose data are released under CC BY 4.0. The local
models ran with PyTorch~\citep{paszke2019pytorch},
Transformers~\citep{wolf2020transformers}, Accelerate~\citep{accelerate2022} and
bitsandbytes, with weights from the Hugging Face Hub, and the API models were
called through the Google Gen AI, OpenAI and Anthropic Python libraries. The
analyses used NumPy~\citep{harris2020numpy}, SciPy~\citep{virtanen2020scipy} and
Matplotlib~\citep{hunter2007matplotlib}, in Jupyter
notebooks~\citep{kluyver2016jupyter} run on Kaggle Code and Google Colab. 

\section{Extended Related Work}\label{sec:extendedrelated}

\paragraph{Inverse problems, inversion problems and one-way functions.}
Inverse problems have always been quintessential to (artificial) intelligence. \citet{mccarthy2006challenges} said that ``getting reality from appearance is an inverse problem. Geologists, oil companies, and astronomers are faced with inverse problems. Their
solution is intellectually difficult and computationally intensive. Human-level AI
systems will also have to be able to infer reality from appearances related to
them in complex ways''. 
The term `inverse problems' was traditionally understood in the context of education and scientific discovery \citep{groetsch1999inverse}. In machine learning, it is now commonly associated with the restoration or recovering of a part of an image or other modality 
\citep{zirvi2025diffusion}
but its original conception as general inversion problems in science is still preserved with some benchmarks for LLMs \citep{dou2025hybrid,zheng2025inversebench}.
The term `inversion problem' usually has a more restricted perspective in computer science and mathematics, as the specific goal of inverting input-outputs of a given function or machine \citep{levin1973,levin2003tale,levin2013universal}. Some of these problems, the so-called one-way functions, are problems that are very easy to compute in one direction, but extremely difficult and time-consuming to reverse. Its existence, not proven yet in its most precise formulation \citep{levin2003tale},  would imply that NP $\neq$ P, which would be reassuring to the notion of hard-to-solve easy-to-verify in which future AI evaluation may depend. 

\paragraph{Solutional complexity.} 
Inverse problems are different for functions, sequences and worlds. One goal is to find a short description of the world, and another to find a policy for that world that allows us to success or explain part of it. 
In the C-test the \emph{instance generator} and the
\emph{acceptable policy} (the solution) for the `problem' coincide; in interactive settings this should be different. \citet{hernandez2015ctests} shows that aggregating tasks
weighted by such a \emph{difficulty function on policies}, rather than by a
universal distribution over tasks, yields decompositions of which the C-test
is the natural non-interactive special case. This solutional view is
developed at length by \citet{hernandez2017book}, where task difficulty is
characterised by properties (size, running time and space, etc.) of the set of acceptable solutions (given some tolerance). Even if we are using sequences, the
present paper adopts this view, made clear by the use of the conditional version of Levin's $Kt$. If we were to generate
environments, we should look for the shortest description of an acceptable
policy (solution) for the environment, not the shortest description of the
environment, as in the Universal Intelligence metric from \citet{legg2007universal}, or its approximation \citep{legg2013approximation}, where
simple uninteresting environments  dominate the measure, or make the
concept of optimality weak \citep{leike2015bad}.

\paragraph{Compression tests, LLMs as compressors and distilling
Occam's razor.}
Many `intelligence as compression' tests exist \citep{mahoney1999text,Dowe-Hajek,hutter2006human}, asking the examinee to compress the data. The idea of an LLM compressing a distribution comes back to the very definition of language model given by  \citet{shannon1949communicationB,shannon1951prediction}. \citet{jaimovitch2021think} asked whether LLMs identify
simple algorithmic concepts from small witness sets. In other words, they found that Occam's razor can be \emph{distilled} from data into the model's
prior. They measured how GPT architectures, program-induction systems, and
humans perform as a joint function of concept complexity and number of
examples. %
In general, it is well recognised today that LLMs are good compressors 
\citep{deletang2024language}, but the analysis of their power for small sequences for testing and the calibration of difficulty is only fully achieved in C-test ONE.

\paragraph{Item response theory, characteristic curves and Epoch's ECI.} 
Our derivation of abilities from a logistic fit follows the tradition of Item Response Theory (IRT) \citep{embretson2013item,bock2021item}, where the notion of difficulty plays a central stage \citep{rust2014modern}. IRT contrasts difficulty and ability from the start, and sets a common scale for them. 
Given a matrix of results of subjects and items, IRT infers the ability of the subjects and the difficulty of items at the same time.
Many traditional IRT models represent the probability of success as a logistic function of the  difference between ability and difficulty. 
The simplest of these models, the 1-PL or Rasch model, only estimates ability and difficulty, assuming the logistic curve has slope 1. The 2-PL model extends this by allowing a second parameter per item, the discrimination, represented by the slope of the logistic curve. These logistic curves are called item characteristic curves, with ability on the \xaxis and response on the \yaxis for a population of subjects. The dual curve is called the person characteristic curve \citep{vale1975study,
trabin1983person}, in this paper referred to as  (agent) characteristic curve, which maps response on the \yaxis for a population of items as a function of difficulty (in our case $Kt$) on the \xaxis. 
An important thing to clarify is that IRT, unlike this paper, is populational on the subjects, and cannot predict performance for items that have not been characterised during the estimation phase, as the parameters depend on the population of subjects. If the subjects or the items are changed (very different groups of humans), all the parameters may change.
IRT has been adapted to machine learning and NLP \citep{martinez2019item,lalor2016building,kipnis2024metabench,polo2024tinybenchmarks,vania-etal-2021-comparing,lalor2024item}, and more recently with the famous ECI index \citep{ho2025rosetta}, a standard IRT application of a large and well-curated matrix of AI models and benchmarks. This IRT application is performed at the benchmark level, rather than the item level. In our paper, difficulty is not inferred from a population but estimated from the items directly. That is why we can anticipate difficulty even before no human or systems has yet solved the problem.

\paragraph{Procedural generation and evaluating with video games.} 
Procedural generation \citep{cobbe2020leveraging,jin2023mini,mirzadeh2024gsm,voudouris2025animal} 
is common in assessment, and it is becoming more frequent in AI because of contamination. Our use of $r,t$ pairs is a kind of parametrised procedural generation, but the actual difficulty metric is then assigned by the program search. %
Sequences have many properties that make it a right choice for procedural generation, but more complex mini-worlds, where the solution is a policy, are likely to be the next step of this programme of evaluating AI using algorithmic information theory.
As mentioned in \Cref{sec:background}, there is a revival of evaluating AI with video games or other small worlds
\citep{mnih2015human,hu2025gamearena,paglieri2025balrog,ying2507assessing}, but the major links are with those that are procedurally-generated \citep{liebana2022general,voudouris2025animal,jin2023mini,ying2026ai}. 

\paragraph{Measurement theory and AI scales} 
Measurement theory \citep{hand2016measurement} deals, among other things, with the type of {\em scale} and the measurement {\em units}. Regarding the scales, Steven's typology of measurement \citep{stevens1946theory} includes `nominal', `ordinal', `interval' and `ratio' scales. %
For quantitative measurement, interval scales make the comparisons of differences meaningful, and calculating a mean is well justified. Ratio scales set an absolute value, and hence ratios become meaningful.  
For a latent factor, the choice of a scale is arbitrary. In IRT, it is somewhat controversial (e.g., \citep{michell1999measurement,lord1975ability}). Although some researchers advocate for a ratio or interval-scale basis for test scores (e.g., `the odds (or log odds) of student 1 answering a class of items correctly are twice the odds (or log odds) of student 2 doing so' is a meaningful assertion \citep{freund2019rasch}), others point out that a latent trait scale of a fitted IRT model is convenient but not inherently meaningful because any monotonic transformation of the estimated latent trait scores produces an equally valid model \citep{lord1975ability,wallmark2024introducing}. 
In contrast to this, our magnitude is just bits, which is a ratio scale, with an absolute zero and fully additive. Beyond algorithmic information theory \citep{hernandez2017book}, it is unusual to use bits as a unit for intelligence evaluation. A major exception is the concept of perplexity (or cross-entropy) in LLMs, inherited from \citet{shannon1948mathematical}. While perplexity was a good proxy of capabilities in the early days of LLMs, the extensive use of post-training stages and more recently reasoning makes it unfit for general capability evaluation. And this will be more so in agential situation. This paper reconnects evaluation, large language models and information coding in a different, yet similarly fundamental way. 

\end{document}